\documentclass{article} 
\usepackage{iclr2026_conference,times}

\usepackage{amsmath,amsfonts,bm}

\def\eqref#1{equation~\ref{#1}}

\def\1{\bm{1}}

\DeclareMathAlphabet{\mathsfit}{\encodingdefault}{\sfdefault}{m}{sl}
\SetMathAlphabet{\mathsfit}{bold}{\encodingdefault}{\sfdefault}{bx}{n}

\usepackage{hyperref}
\usepackage{url}
\usepackage{inconsolata}

\usepackage{times}
\usepackage{latexsym}
\usepackage{amsfonts}
\usepackage{amsmath}
\usepackage{amssymb}
\usepackage{multicol}
\usepackage{multirow}
\usepackage{xspace}
\usepackage{booktabs}
\usepackage{bbding}
\usepackage{array}
\usepackage{threeparttable}
\usepackage{tcolorbox}
\usepackage{tabularx}
\usepackage{enumitem}
\usepackage{xcolor,colortbl}
\usepackage{color}
\usepackage{setspace}
\usepackage{makecell}
\usepackage{listings}

\usepackage{xltabular}
\usepackage{tcolorbox}
\tcbuselibrary{breakable}
\usepackage{xcolor}

\lstdefinestyle{python}{
    language=Python,
    basicstyle=\ttfamily\small,
    keywordstyle=\color{blue}\bfseries,
    commentstyle=\color{green},
    stringstyle=\color{red},
    numberstyle=\tiny\color{gray},
    showstringspaces=false,
    frame=single,
    breaklines=true,
    backgroundcolor=\color{lightgray!20}
}

\lstdefinestyle{bash}{
    language=Bash,
    basicstyle=\ttfamily\small,
    backgroundcolor=\color{black!5}, 
    frame=single,                    
    breaklines=true,                 
    postbreak=\mbox{\textcolor{red}{$\hookrightarrow$}\space}, 
    showstringspaces=false
}

\definecolor{royalblue}{RGB}{65, 105, 225}
\definecolor{softgreen}{RGB}{85, 170, 85} 
\definecolor{softred}{RGB}{200, 50, 50}   

\hypersetup{
    colorlinks=true,
    linkcolor=softred,    
    citecolor=softgreen,    
    urlcolor=royalblue      
}

\usepackage{booktabs} 

\usepackage[utf8]{inputenc} 
\usepackage[T1]{fontenc}    
\usepackage{hyperref}       
\usepackage{url}            
\usepackage{booktabs}       
\usepackage{amsfonts}       
\usepackage{nicefrac}       
\usepackage{microtype}      
\usepackage{xcolor}         
\usepackage{float}
\usepackage{graphicx}       
\usepackage{wrapfig}        
\usepackage{subcaption}     
\usepackage{makecell}
\usepackage{soul}
\usepackage{cancel}
\usepackage{algorithm}
\usepackage{algorithmic}
\usepackage{amsthm}

\newtheorem{proposition}{Proposition}
\newtheorem{lemma}{Lemma}
\newtheorem{definition}{Definition}

\definecolor{mygray}{gray}{0.95}
\usepackage{hyperref}
\usepackage{url}

\newcommand{\ours}{Entropy Regularizing Activation}
\newcommand{\short}{ERA}

\title{Entropy Regularizing Activation: Boosting Continuous Control, Large Language Models, and Image Classification with Activation as Entropy Constraints}

\iclrfinalcopy
\author{\textbf{Zilin Kang}$^{1,2}$\thanks{Equal Contribution}~
    \textbf{Chonghua Liao}$^{3*}$
    \textbf{Tingqiang Xu}$^{3*}$ 
    \textbf{Huazhe Xu}$^{1,3,4}$  \\
    $^{1}$Shanghai Qi Zhi Institute \\
    $^{2}$Department of Computer Science and Technology, Tsinghua University \\
    $^{3}$Institute for Interdisciplinary Information Sciences, Tsinghua University\\
    $^{4}$Shanghai Artificial Intelligence Laboratory\\
    \texttt{\{kzl22,lch22,xtq23\}@mails.tsinghua.edu.cn}
}

\begin{document}

\maketitle
\vskip -0.2in
\begin{abstract}
We propose \short, a new paradigm that constrains the sampling entropy above given thresholds by applying specially designed activations to the outputs of models. Our approach demonstrates broad effectiveness across different domains: 1) for large language models~(LLMs), boosting the AIME 2025 score for Qwen2.5-Math-7B by 37.4\%; 2) for continuous control reinforcement learning agents, improving performance by more than 30\% over strong baselines such as SAC on the challenging HumanoidBench; 3) for image classification, enhancing ImageNet top-1 accuracy by 0.69\% for ResNet-50. These gains are achieved with a computational overhead of less than 7\%. Our work validates output activation as a powerful tool for entropy control, opening a new direction for designing simpler and more robust algorithms. Code available at: \href{https://nothingbutbut.github.io/era}{\textit{https://nothingbutbut.github.io/era}}
\vskip -0.1in
\end{abstract}

\section{Introduction}

\label{sec:introduction}
Decision-making problems represent a broad class of challenges, from robotic control to Large Language Models alignment
~\citep{sutton1998reinforcement,ouyang2022training,kober2013reinforcement,yuan2025hermes}. In these settings, encouraging exploration and maintaining policy stochasticity, often quantified by entropy, is critical~\citep{ziebart2008maximum,schulman2017proximal}. In reinforcement learning, the maximum entropy paradigm, exemplified by algorithms like Soft Actor-Critic (SAC)~\citep{haarnoja2018soft}, has become a prevailing approach in control tasks. However, these methods, which add an entropy bonus directly to the training objective, inevitably alter the optimization landscape and can interfere with the optimization of the primary objective.

The challenge becomes even more pronounced in LLM alignment. Policy gradient methods~\citep{NIPS1999_464d828b} such as GRPO~\citep{shao2024deepseekmath} frequently suffer from entropy collapse~\citep{cui2025entropy}, leading to reduced diversity and performance degradation. Directly incorporating entropy bonuses has been shown to be unstable or ineffective in this setting~\citep{cui2025entropy}. Moreover, prior works have explored methods that avoid direct modification of the loss function, including clip-higher~\citep{yu2025dapo} and training exclusively on the high-entropy tokens \citep{wang2025beyond}. While these methods provide useful insights, they remain ad hoc, lack a principled mechanism for entropy regulation, and are narrowly tailored to the LLM domain, limiting their applicability to broader settings such as continuous control and computer vision tasks.

These observations highlight a fundamental gap: existing approaches either distort the primary optimization objective, as in RL algorithms with entropy bonus terms, or provide heuristic, domain-specific fixes with no theoretical guarantees, as in LLM alignment. Therefore, there is a pressing need for a new entropy-constraining paradigm that is universally applicable, non-invasive to the primary objective, and theoretically grounded.

In this work, we introduce \textbf{E}ntropy \textbf{R}egularizing \textbf{A}ctivation~(\short ), a novel paradigm for entropy-constrained training. The key insight of \short\ lies in imposing an entropy constraint through a class of well-designed activation functions applied to the model's final output. This approach completely decouples the optimization of the primary objective from the entropy constraint, allowing the loss function to focus solely on its original goal (e.g., maximizing rewards). We show that \short\ not only provides provable entropy guarantees in theory, but in practice, it functions as a non-invasive module that can be seamlessly integrated with existing algorithms. 

The generality and effectiveness of this paradigm are validated across diverse domains, including continuous control, image classification, and large language models. 
For example, on the DeepMind Control Suite~\citep{tassa2018deepmind}, \short\ improves the performance of SAC on high-dimensional tasks like Humanoid and Dog by over 25\%. Its versatility is also demonstrated in image classification, a domain where preventing model overconfidence via regularization is critical. Our approach complements established methods, boosting performance on top of strong data augmentation and label smoothing~\citep{szegedy2016rethinking}. In LLM RL, \short\ enables a GRPO-trained Qwen-2.5-Math-7B~\citep{yang2024qwen2} to achieve a remarkable improvement of 9.0\% and 37.4\% on the AIME-24 and AIME-25 benchmarks, respectively.

Our main contributions are summarized as follows: 
\begin{itemize}[leftmargin=0.5cm]
\item We introduce \textbf{\short, a novel entropy constraint paradigm} based on activation functions, and establish a theoretical framework with provable entropy guarantees.
\item We design effective instantiations of \short\ for both continuous (control) and discrete (image classification) domains. For large language models, we propose a \textbf{specialized, adaptive variant of \short\ that addresses the unique challenges} within this domain.
\item Our experiments of these instantiations \textbf{demonstrate significant performance improvements} over strong baselines across domains, and reveal their properties such as parameter sensitivity.
\end{itemize}

\begin{figure*}[t]
    \centering
    \vskip -0.25in
    \begin{subfigure}[b]{0.38\textwidth}
        \centering
        \includegraphics[width=\textwidth]{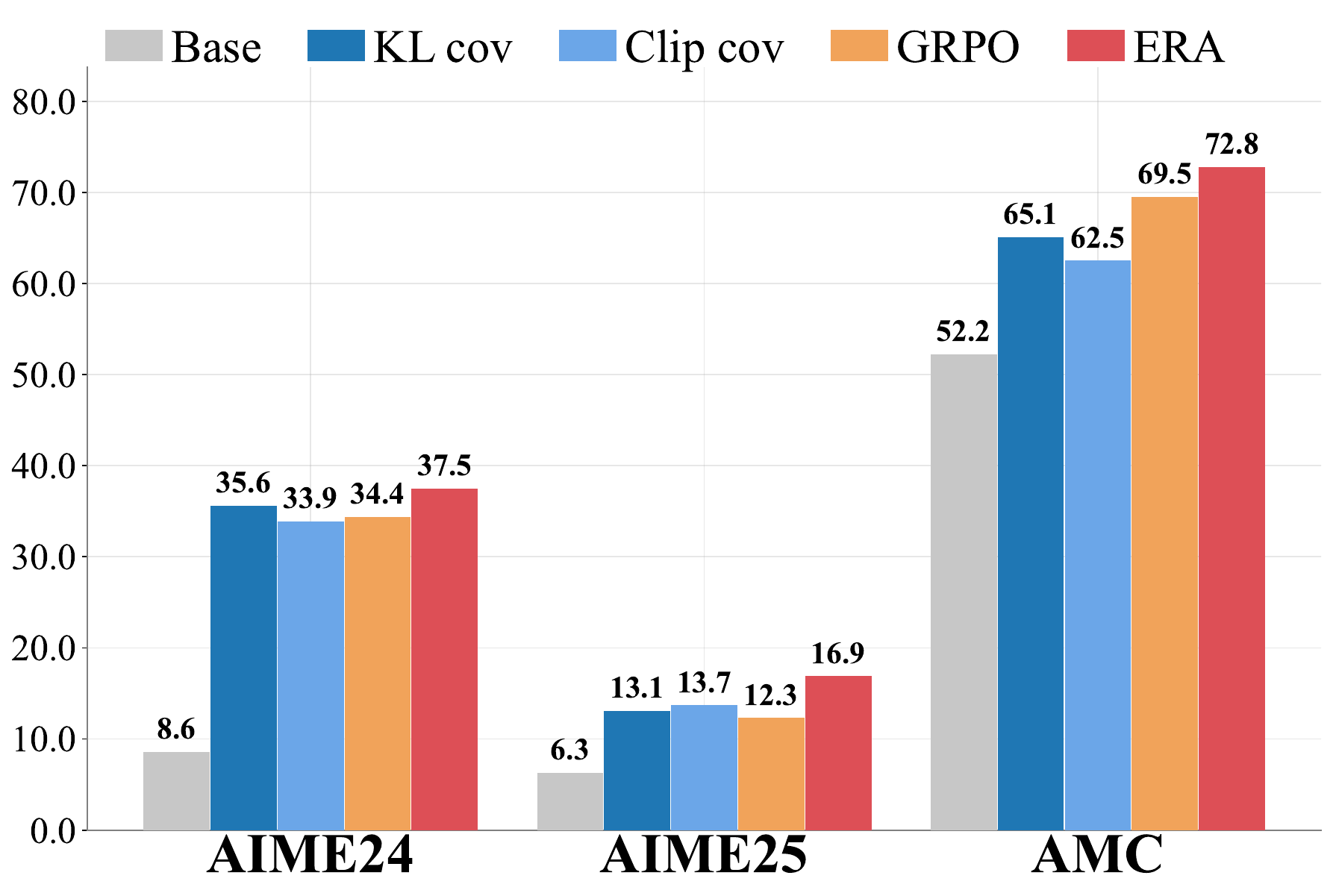}
        \subcaption{}
        \label{fig:teaser_llm}
    \end{subfigure}
    \begin{subfigure}[b]{0.285\textwidth}
        \centering
        \includegraphics[width=\textwidth]{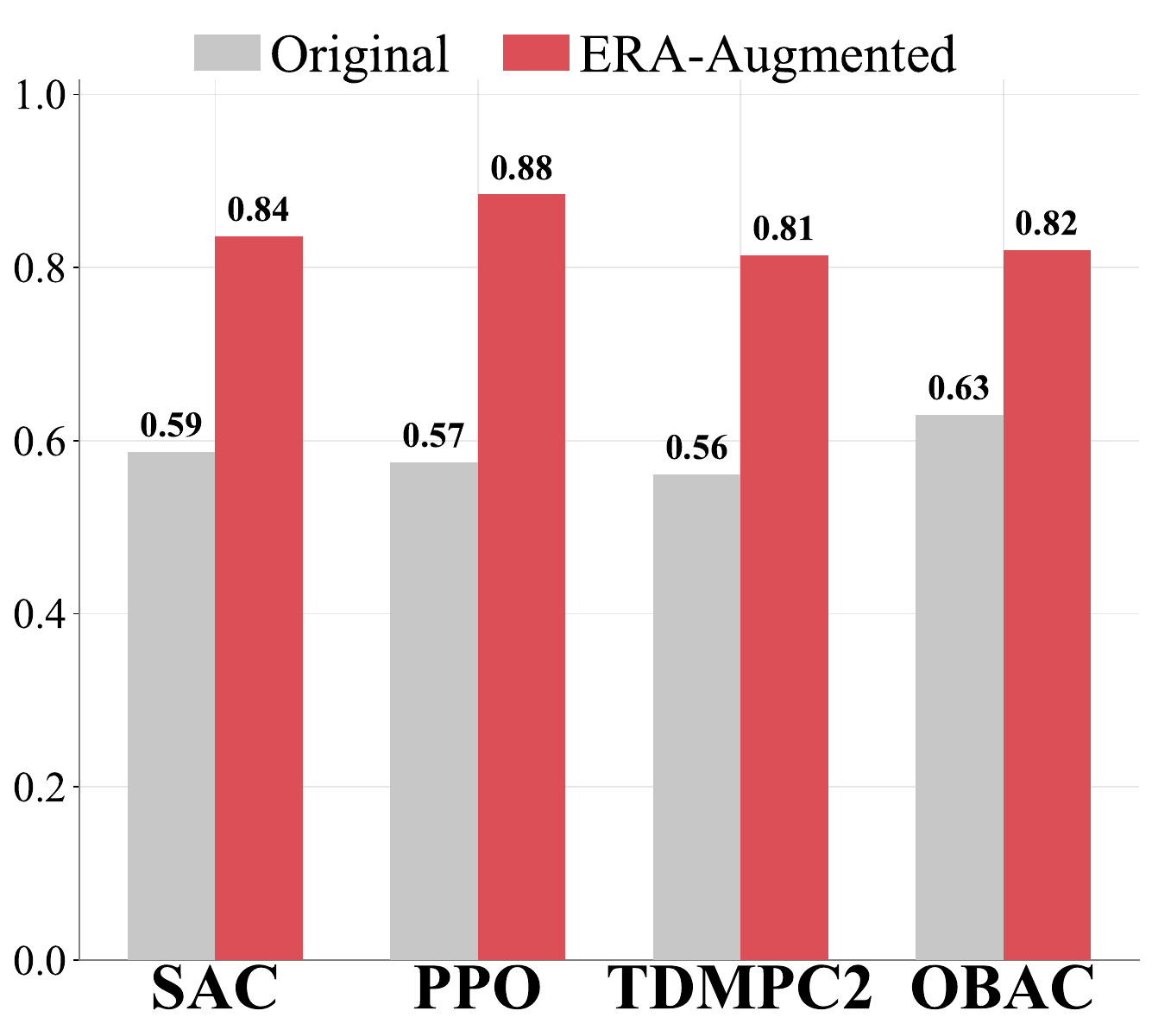}
        \subcaption{}
        \label{fig:teaser_control}
    \end{subfigure}
    \begin{subfigure}[b]{0.285\textwidth}
        \centering
        \includegraphics[width=\textwidth]{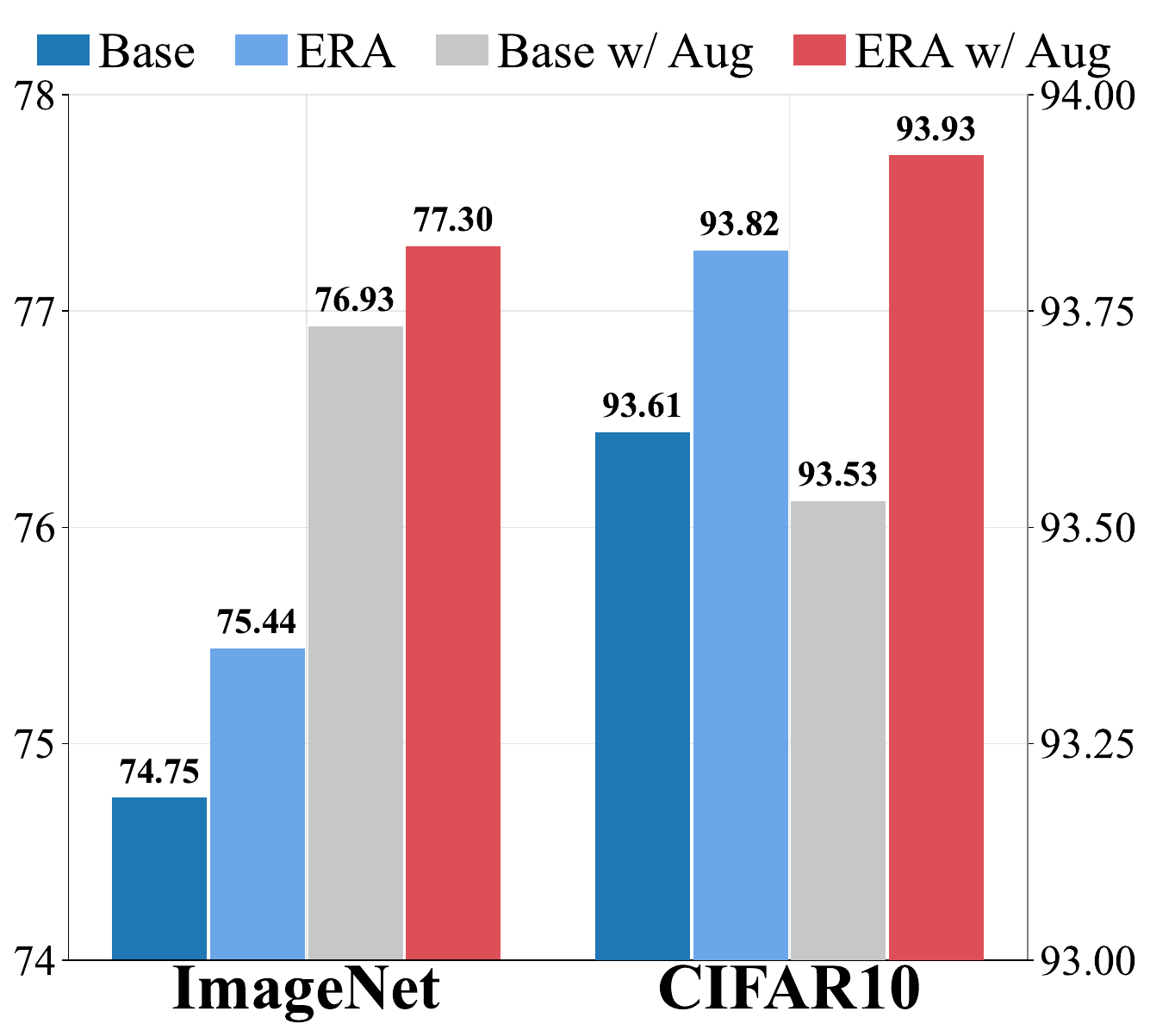}
        \subcaption{}
        \label{fig:teaser_vision}
    \end{subfigure}
    \vskip -0.1in
    \caption{
        \textbf{ERA Boosts Large Language Models, Continuous Control and Image Classification.} 
        (a) \textbf{Large Language Models:} ERA consistently enhances the performance of Qwen-2.5-Math-7B on AIME'24,AIME'25 and AMC datasets.
        (b) \textbf{Continuous Control:} ERA significantly improves multiple popular RL algorithms, including SAC, PPO, TD-MPC2 and OBAC.
        (c) \textbf{Image Classification:} ERA consistently boosts the performance of ResNet-50 on ImageNet and CIFAR-10 datasets.
    }
    \label{fig:teaser}
    \vskip -0.2in
\end{figure*}
\section{Related Work}
\label{sec:related_work}
\paragraph{Policy learning in control.}
Entropy maximization is a crucial aspect of RL, significantly enhancing exploration and robustness~\citep{ziebart2010modeling,haarnoja2017reinforcement}. Prior work has explored various methods to incorporate entropy maximization into RL algorithms~\citep{o2016combining,nachum2017bridging,haarnoja2017reinforcement}. PPO~\citep{ppo} introduced an entropy bonus in its clipped surrogate objective. SAC~\citep{haarnoja2018soft} later employed a maximum-entropy objective with a dynamically adjusted temperature parameter, but this can lead to instability. 
More recent approaches have introduced alternative methodologies for implementing maximum entropy RL~\citep{chao2024maximumentropyreinforcementlearning,choe2024maximumentropyonpolicyactorcritic}, while others have shifted the optimization focus directly to state entropy~\citep{zhong2024maximum}.
All these methods, while effective, modify the original cumulative reward objective by introducing the entropy term, which can lead to suboptimal performance. Our approach addresses this issue by maintaining the original objective, ensuring more reliable performance. 

\paragraph{RL for LLMs.}
Recent breakthroughs in LLM reasoning, such as OpenAI-o1~\citep{jaech2024openai}, DeepSeek-R1~\citep{guo2025deepseek}, and Kimi-k1.5~\citep{team2025kimi}, have redirected attention from chain-of-thought prompting~\citep{wei2022chain} and supervised fine-tuning~\citep{li2024common,yeo2025demystifying} toward RL. Within this paradigm, policy entropy collapse emerges as a fundamental obstacle: the decay of exploratory behavior often leads to performance plateaus. A prevalent approach is reward shaping~\citep{cheng2025reasoning}, which augments the reward or advantage with an entropy bonus to maintain a viable exploration–exploitation trade-off. Complementary strategies, including loss re-weighting~\citep{wang2025beyond,cui2025entropy} and clip-higher regularization~\citep{yu2025dapo}, mitigate the risk of entropy collapse. Unlike these approaches, our method is a general and concise paradigm, universally applicable across domains and endowed with rigorous theoretical guarantees.

\section{Preliminaries}
\label{sec:preliminaries}


\textbf{Policy optimization.}
Policy gradient (PG) methods optimize $J(\pi_\theta) = \mathbb{E}_{\tau \sim \pi_\theta} \left[ \sum_{t=0}^{T} \gamma^t R(s_t, a_t) \right]$ via gradient ascent. For large language model~(LLM) alignment, Proximal Policy Optimization (PPO)~\citep{schulman2017proximal} is commonly used. The GRPO variant estimates the advantage $A(y)$ for a generated response $y$ from a set of $K$ samples as:
\begin{equation}
\label{eq:grpo_advantage}
A(y) = \frac{r(y) - \text{mean}(r(y^{1:K}))}{\text{std}(r(y^{1:K}))}.
\end{equation}
The policy is then updated using the clipped surrogate objective:
\begin{equation}
\label{eq:ppo_objective}
\mathcal{L}^{\text{CLIP}}(\theta) = \mathbb{E}_t \left[ \min\left( r_t(\theta) A_t, \text{clip}(r_t(\theta), 1-\epsilon, 1+\epsilon) A_t \right) \right],
\end{equation}
where $r_t(\theta) = \frac{\pi_\theta(a_t|s_t)}{\pi_{\theta_{\text{old}}}(a_t|s_t)}$ is the probability ratio.

\textbf{Policy entropy.}
Policy entropy, $\mathcal{H}(\pi(\cdot|s))$, measures the policy's stochasticity.
For discrete action spaces, the token-level entropy is given by Eq.~\ref{eq:discrete_entropy}.
For continuous policies, there are several common ways to ensure actions remain within a bounded space.
A popular method is to use a squashed Gaussian policy, which outputs a bounded action $a = \tanh(u)$ by sampling $u$ from a Gaussian distribution $\pi_\theta(\cdot|s)=\mathcal{N}(\mu_\theta(s), \Sigma_\theta(s))$ parameterized by the policy network. The entropy of this policy is given by Eq.~\ref{eq:squashed_gaussian_entropy}.
Alternatively, another common approach is to directly sample actions from a Truncated Gaussian distribution $\pi_\theta(\cdot|s)=\text{TN}(\mu_\theta(s), \Sigma_\theta(s), -1, 1)$ over the bounded hypercube $[-1, 1]^D$. Assuming the dimensions are independent, its entropy is given by Eq.~\ref{eq:truncated_gaussian_entropy}.
\begin{align}
\label{eq:discrete_entropy}
\mathcal{H}(\pi_\theta) &= - \mathbb{E}_{x \sim \rho_\pi, y \sim \pi_\theta(x)} \left[ \frac{1}{|y|} \sum_{t=1}^{|y|} \log \pi_\theta(y_t | y_{<t}, x) \right], \\
\label{eq:squashed_gaussian_entropy}
\mathcal{H}(\pi_\theta)  &= \mathbb{E}_{s \sim \rho_\pi, u \sim \mathcal{N}(\mu_\theta(s), \Sigma_\theta(s))}\left[-\log\mathcal{N}(u|\mu_\theta(s), \Sigma_\theta(s)) + \sum_{i=1}^D \log(1 - \tanh(u_i)^2)\right], \\
\label{eq:truncated_gaussian_entropy}
\mathcal{H}(\pi_\theta) &= \mathbb{E}_{s \sim \rho_\pi} \left[ \sum_{i=1}^{D} \left( \log(\sigma_{\theta,i}(s) Z_i(s) \sqrt{2\pi e}) - \frac{\beta_i(s) \phi(\beta_i(s)) - \alpha_i(s) \phi(\alpha_i(s))}{2Z_i(s)} \right) \right]
\end{align}
where for the truncated Gaussian entropy in Eq.~\ref{eq:truncated_gaussian_entropy}, $\phi$ and $\Phi$ are the PDF and CDF of the standard normal distribution, respectively. 
We define the standardized bounds $\alpha_i(s) = (-1 - \mu_{\theta,i}(s))/\sigma_{\theta,i}(s)$, $\beta_i(s) = (1 - \mu_{\theta,i}(s))/\sigma_{\theta,i}(s)$, and the normalization constant $Z_i(s) = \Phi(\beta_i(s)) - \Phi(\alpha_i(s))$.

\textbf{Maximum entropy reinforcement learning.}
Building upon policy entropy, the maximum entropy RL framework aims to maximize the standard reward objective subject to a minimum entropy constraint $\mathcal{H}_0$:
\begin{equation}
\label{eq:constrained_objective}
\max_{\theta} J(\pi_\theta) \quad \text{s.t.} \quad \mathbb{E}_{s \sim \rho_\pi}[\mathcal{H}(\pi_\theta(\cdot|s))] \ge \mathcal{H}_0.
\end{equation}
Practical algorithms like Soft Actor-Critic (SAC)~\citep{haarnoja2018soft} solve the Lagrangian dual of this problem. SAC is an off-policy actor-critic algorithm that updates a soft Q-function $Q_\phi$ and a policy $\pi_\theta$. The Q-function is updated by minimizing the soft Bellman residual $J_Q(\phi)$:
\begin{equation}
\label{eq:sac_q_loss}
J_Q(\phi) = \mathbb{E}_{(s_t, a_t, s_{t+1}) \sim \mathcal{D}} \left[ \frac{1}{2} \left( Q_\phi(s_t, a_t) - y \right)^2 \right]
\end{equation}
\begin{equation}
\label{eq:sac_q_target}
y = R(s_t, a_t) + \gamma \mathbb{E}_{a_{t+1} \sim \pi_{\theta}(\cdot|s_{t+1})} \left[ Q_{\phi'}(s_{t+1}, a_{t+1}) - \alpha \log \pi_{\theta}(a_{t+1}|s_{t+1}) \right]
\end{equation}
with the target $y$ computed using a target Q-network $Q_{\phi'}$. The target network parameters $\phi'$ are updated via an exponential moving average (EMA): $\phi' \leftarrow \tau \phi + (1-\tau)\phi'$.
\begin{equation}
\label{eq:sac_policy_loss}
J_\pi(\theta) = \mathbb{E}_{s_t \sim \mathcal{D}, a_t \sim \pi_\theta} \left[ Q_\phi(s_t, a_t) - \alpha \log \pi_\theta(a_t|s_t)  \right].
\end{equation}
The policy is then updated by maximizing the objective in Eq.~\ref{eq:sac_policy_loss}.

\section{The Entropy Regularizing Activation}
\label{sec:method}
\subsection{The Core Idea: Entropy Constraint via Output Activation}
\label{subsec:core_idea}
The core of \ours\ is to enforce maximum entropy reinforcement learning on the policy, not through a loss penalty, but via integrating the constraint into the network's architecture via a special activation function.

Let a parameterized policy $f_\theta(s)$ produce distribution parameters $z = f_\theta(s)$, where $z$ belongs to a parameter space $\mathcal{Z}$. The policy corresponding to these parameters is $\pi_z(\cdot|s)$. We introduce an activation function $g: \mathcal{Z} \to \mathcal{Z}$, which transforms the initial parameters $z$ to a new set $z' = g(z)$. The final policy, which we denote as $\pi_\theta$, is thus given by $\pi_\theta(\cdot|s) = \pi_{g(f_\theta(s))}(\cdot|s)$. The function $g(.)$ is designed to ensure that the policy $\pi_\theta$ satisfies a constraint on its expected entropy, for a given target entropy $\mathcal{H}_0$:
$$ \mathbb{E}_{s \sim \rho_\pi}[\mathcal{H}_{\pi_\theta(\cdot|s)}] \geq \mathcal{H}_0 $$

This formulation enables the policy to satisfy the expected entropy condition while leaving the training objective for $\theta$ free of an explicit entropy term, as shown in Eq.~\ref{eq:constrained_objective}. This approach effectively mitigates gradient conflicts between the task objective and the entropy maximization objective, allowing the optimization to focus on the primary objective.
\subsection{Instantiations for Continuous and Discrete Spaces}
\label{subsec:instantiation_foundational}
To ground the general framework presented in section~\ref{subsec:core_idea}, we now instantiate the entropy regularizing activation~$g(.)$ for two canonical policy classes: policies based on a bounded Gaussian distribution, such as the Tanh-squashed Gaussian~\citep{haarnoja2018soft} or the clipped Gaussian~\citep{fujimoto2018addressingfunctionapproximationerror}, commonly used in continuous control; and the softmax policy prevalent in discrete spaces.

\subsubsection{Continuous Control with Bounded Gaussian Policies}
\label{subsubsec:instantiation_continuous}
In continuous control, policies often sample actions from a Gaussian distribution and then apply a bounding function (e.g., a $\tanh$ squash or clipping) to ensure outputs lie within a valid range. This bounding operation complicates direct entropy maximization, as it introduces a state-dependent bias term. Prior methods typically address this by adding an entropy bonus to the learning objective. Our insight is that the entropy of the final bounded policy, $\mathcal{H}_{\pi}$, can be seen as the entropy of the original unbounded Gaussian, $\mathcal{H}_{\text{Gaussian}}$, minus a non-negative bias term introduced by the bounding operation, i.e., $\mathcal{H}_{\pi} = \mathcal{H}_{\text{Gaussian}} - \mathbb{E}[\text{bias}]$. Consequently, a minimum entropy constraint on the final policy can be satisfied by constraining the underlying Gaussian's entropy to a corresponding, higher value. This is achieved by adjusting the Gaussian's standard deviation, $\sigma$. The entropy of a $D$-dimensional Gaussian with a diagonal covariance matrix is:
\begin{equation}
    \mathcal{H}_{\text{Gaussian}}(s) = \frac{1}{2} \sum_{i=1}^{D} \log(2\pi e \sigma_i(s)^2)
    \label{eq:gaussian_entropy}
\end{equation}
To maintain training stability, the standard deviation must also be kept within a predefined range $[\sigma_{\min}, \sigma_{\max}]$. Our activation function $g(.)$ simultaneously satisfies both constraints. Given network outputs for the mean $\mu$ and a pre-activation standard deviation $\hat{\sigma}$, the function $g(\mu, \hat{\sigma})$ produces the final parameters $(\mu', \sigma')$ where:
\begin{equation}
    \mu' = \mu, \quad \sigma' = \exp{\left[\max\left(\log \sigma_{\max} + (\mathcal{H}_0' -D\log \sqrt{2\pi e} - D \log \sigma_{\max})\frac{e^{\hat{\sigma}_i}}{\sum_{j=1}^{D}e^{\hat{\sigma}_j}} , \log\sigma_{\min}\right)\right]}
    \label{eq:continuous_era}
\end{equation}
Here, $\mathcal{H}_0'$ is the target entropy for the final policy~$\mathcal{H}_0$ plus a compensation parameter $\delta \geq 0$ to account for the bounding bias, which can either be a constant or automatically tuned by learning with the loss in Eq.~\ref{eq:residual_loss}.
\begin{equation}
    \mathcal{L}(\hat{\delta}) = \mathbb{E}_{s \sim \mathcal{D}} \left[\hat{\delta}(\mathcal{H}[\pi(\cdot|s)] - \mathcal{H}_0)\right]
    \label{eq:residual_loss}
\end{equation}
We refer the reader to Appendix~\ref{subsec:impl_continuous} for implementation details and Appendix~\ref{subsec:proofs_continuous} for a proof of the entropy bound.

By satisfying the entropy constraint architecturally, our method obviates the need for an explicit entropy term in the objective function. Hence, target of the critic and the actor loss of SAC in Eq.~\ref{eq:sac_q_target} and Eq.~\ref{eq:sac_policy_loss} can be simplified to the form in Eq.~\ref{eq:simplified_sac_q_target} and Eq~\ref{eq:simplified_sac_policy_loss}
\begin{align}
\label{eq:simplified_sac_q_target}
y &= R(s_t, a_t) + \gamma \mathbb{E}_{a_{t+1} \sim \pi_{\theta}(\cdot|s_{t+1})} \left[ Q_{\phi'}(s_{t+1}, a_{t+1}) \cancel{- \alpha \log \pi_{\theta}(a_{t+1}|s_{t+1})} \right] \\
\label{eq:simplified_sac_policy_loss}
J_\pi(\theta) &= \mathbb{E}_{s_t \sim \mathcal{D}, a_t \sim \pi_\theta} \left[ Q_\phi(s_t, a_t) \cancel{- \alpha \log \pi_\theta(a_t|s_t)}  \right]
\end{align}

\subsubsection{Discrete Classification with Softmax Policies}
\label{subsubsec:instantiation_discrete}
In discrete classification, regularizing the predictive entropy is crucial for preventing the overconfidence that leads to overfitting. \short\ provides architectural regularization by enforcing a minimum entropy level, analogous to how techniques like label smoothing improve generalization by smoothing the output distribution. For a softmax policy, we enforce this constraint by transforming the pre-activation logits $z$ into $z'$ such that the resulting policy's entropy is at least $\mathcal{H}_0$:
\begin{equation}
    z' = h^{-1}\left[\max \left(\frac{\log \tau}{\tau} + \left(C_{\mathcal{H}_0} - n \frac{\log \tau}{\tau}\right)\frac{1}{D-1}\left(1 - \frac{e^{z_i}}{\sum_{j=1}^{D}e^{z_j}}\right), 0\right)\right]
    \label{eq:discrete_era}
\end{equation}
Here, $h^{-1}$ denotes the inverse of $-xe^x$ on $[0, \frac{1}{e}]$, approximated by $\hat{h}^{-1}(x) = -\frac{1}{4}-\sqrt{2(-1-\ln (x))}+\frac{3}{4}\ln x$. We also define $C_{\mathcal{H}_0} = \exp(\mathcal{H}_0 - 1)$, where $\tau \ge e$ is a fixed hyperparameter (e.g., $\tau = 4$). A formal proof is provided in Appendix~\ref{subsec:proofs_discrete}.

In contrast to label smoothing, which applies a fixed and uniform regularization, \short~offers greater flexibility. It allows the model to learn a structured, input-dependent uncertainty distribution, tailoring the regularization to each sample and thus offering greater expressive capacity and potential for improved performance.
\subsection{Instantiations for RL in Large Language Models}
\label{subsec:instantiation_llms}

In reinforcement learning for LLMs, each token is treated as a discrete action, with the policy defined by a canonical softmax distribution. Prior approaches to addressing entropy collapse in LLMs—such as the traditional entropy bonus, clip-higher
, KL-Cov, and Clip-Cov
—do not provide a provable entropy lower bound, and are incompatible with the on-policy setting. In contrast, our method introduces ERA, a simple and non-invasive activation function that offers a theoretical guarantee of a minimum entropy level, effectively resolving entropy collapse in on-policy reinforcement learning.

In contrast to standard RL settings, the action space is extremely large. In the previous ERA instantiation, each token has a lower entropy bound. However, due to the intrinsic structure of natural language, most tokens are nearly deterministic; therefore, directly enforcing high entropy across all tokens is impractical: it will lead to unintended tokens and can corrupt the entire response. Furthermore, modifying the internal structure of the model also introduces instability in different training environments, leading to unpredictable behavior.

To address these challenges, we propose a new instantiation of ERA that is applied \emph{after} the sampling process. Specifically, responses are first generated using the original model output $z$, and the advantages are computed following the GRPO rule. Then, during model updates, the probabilities of the sampled tokens are reinterpreted as $z'$, obtained by applying our entropy-regularized activation. This design leaves the sampling policy unchanged while still ensuring effective entropy regularization.

Formally, when updating model parameters, we apply an activation layer to the logits $z$ to obtain a transformed set $z'$, defined as:  
\begin{equation}
    z' = \begin{cases}
        kz & H_{\text{resp}} < \omega_{\text{low}},\; A_{t}>0, \\
        z & \omega_{\text{low}} \leq H_{\text{resp}} \leq \omega_{\text{high}},\; A_{t}<0 \text{ or } A_{t}>0, \\
        \tfrac{1}{k}z & H_{\text{resp}} > \omega_{\text{high}},\; A_{t}>0,
    \end{cases}\label{eq:llm}
\end{equation}
where $k>1$, and $\omega_{\text{low}}, \omega_{\text{high}}$ are algorithm-specific constants. Here, $A_{t}$ denotes the advantage of the token, and $H_{\text{resp}}$ is the average entropy of the top $20\%$ of tokens with the highest entropy in the response. To balance the gradient between modified tokens and unmodified tokens (details are shown in Appendix~\ref{subsec:proofs_llm}), we add another scaling factor on the advantages of modified tokens:

\begin{equation}
    A_t' = \begin{cases}
        \frac 1k A_t & H_{\text{resp}} < \omega_{\text{low}},\; A_{t}>0, \\
        A_t & \omega_{\text{low}} \leq H_{\text{resp}} \leq \omega_{\text{high}},\; A_{t}<0 \text{ or } A_{t}>0, \\
        kA_t & H_{\text{resp}} > \omega_{\text{high}},\; A_{t}>0,
    \end{cases}
\end{equation}

The on-policy GRPO objective becomes:

\begin{equation}
    J(\theta) = \mathbb{E}_t[\mathbb{E}_{a_t\sim \pi_\theta(\cdot |s_t)}\log \pi_\theta'(a_t|s_t) A'_t]
\end{equation}

where $\pi_\theta$ is the original policy from $z$ (representing that the inference still follows the original policy), and $\pi'_\theta$ is the ERA-adjusted policy from $z'$ (representing that the model update relies on the new policy). Intuitively, this activation layer adjusts all positively advantaged responses: when entropy is too low, it sharpens the probability distribution; when entropy is too high, it flattens it. Unlike our instantiation for control tasks, increasing policy entropy here requires \emph{sharpening} the distribution. The rationale is that sampling has already occurred, and by treating the samples as if they were drawn from a sharpened policy, the model perceives itself as overexploiting, thus encouraging additional exploration. The choice of the top $20\%$ tokens is based on the fact that, in natural language, these tokens are considered forking tokens, whose entropy is the target of regularization, and the remaining tokens are allowed to have almost zero entropy~\citep{wang2025beyond}.

We show that, under reasonable assumptions, this ERA instantiation ensures that the policy entropy remains above a fixed constant $\mathcal{H}_0$. We refer the reader to Appendix~\ref{subsec:proofs_llm} for a formal proof.

\begin{figure*}[t]
    \vskip -0.2in
    \begin{center}
    \centerline{\includegraphics[width=0.98\textwidth]{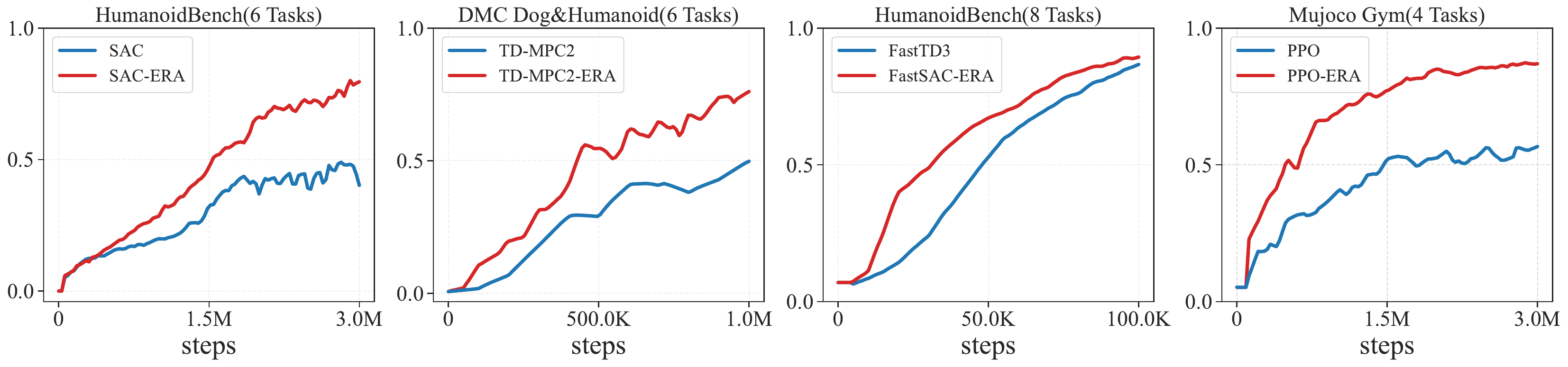}}
    \end{center}
    \vskip -0.25in
    \caption{\textbf{Main Results of \short\ in Continuous Control.} Aggregate normalized performance on HumanoidBench (6 tasks, with SAC), DMC (Humanoid \& Dog) (6 tasks, with TD-MPC2), HumanoidBench (8 tasks, with FastSAC) and Mujoco Gym (4 tasks, with PPO). \short\ consistently accelerates learning and achieves superior asymptotic performance.}
    \vskip -0.2in
    \label{fig:preview_results}
\end{figure*}
\section{Results and Analysis}
\label{sec:results_foundational}
\subsection{Experiments on Continuous Control}
\label{subsec:exp_control}
We conduct extensive experiments to validate the effectiveness of \short\ in continuous control tasks. 
We demonstrate the broad applicability and performance gains by integrating \short\ into five distinct algorithms---SAC, OBAC~\citep{luo2024offline}, TD-MPC2, PPO, and FastSAC~\citep{seo2025fasttd3}. The evaluation is performed on a wide range of challenging benchmarks, including the DeepMind Control Suite (Humanoid \& Dog), HumanoidBench~\citep{sferrazza2024humanoidbench}, and MuJoCo Gym~\citep{todorov2012mujoco}.
Implementation details, environment specifics, and hyperparameter settings are available in Appendix~\ref{subsec:impl_continuous}. Comprehensive results for all tasks can be found in the Appendix~\ref{sec:additional_results}.

\textbf{Main results.}
We present our main results in continuous control in Figure~\ref{fig:preview_results}. Integrating \short\ consistently yields significant improvements in both sample efficiency and final performance across diverse algorithms and benchmarks. 
\textbf{\short\ consistently improves performance across various entropy targets.} We evaluate the performance of SAC and SAC-\short\ under varying entropy targets. 
The results in Figure~\ref{fig:ablation_entropy}, tested on four DMC tasks (\textit{dog-run, dog-trot, humanoid-run, humanoid-walk}) 
with 5 seeds on each environment, show that SAC-\short\ consistently outperforms original SAC across the entire tested spectrum of entropy values. By bypassing the entropy constraint within the learning objective, \short\ allows the policy to focus more directly on reward maximization. While simply removing the entropy term from SAC can also avoid this constraint, its performance is inferior to the \short-enhanced version due to insufficient exploration. This consistent outperformance suggests that \short\ can achieve strong results without precise tuning of the entropy hyperparameter, offering a significant practical advantage.

\subsection{Experiments on Image Classification}
\label{subsec:exp_vision}
\begin{table}[h!]
\centering
\caption{Top-1 and Top-5 accuracy (\%) on ImageNet and CIFAR-10. We compare \short\ against the original ResNet-50 baseline. $\Delta$ denotes the absolute improvement of \short. All models are trained for 200 epochs.}
\label{tab:main_results_vision}
\resizebox{\linewidth}{!}{%
\begin{tabular}{llcccc|cccc}
\toprule
\multirow{2}{*}{\textbf{Dataset}} & \multirow{2}{*}{\textbf{Method}} & \multicolumn{4}{c|}{\textbf{Without Data Augmentation}} & \multicolumn{4}{c}{\textbf{With Data Augmentation}} \\
\cmidrule(lr){3-6} \cmidrule(lr){7-10}
 & & Top-1 Acc. & $\Delta$ & Top-5 Acc. & $\Delta$ & Top-1 Acc. & $\Delta$ & Top-5 Acc. & $\Delta$ \\
\midrule
\multirow{2}{*}{ImageNet} & Original & 74.75 $\pm$ 0.38 & - & 92.04 $\pm$ 0.23 & - & 76.93 $\pm$ 0.36 & - & 93.37 $\pm$ 0.21 & - \\
& \texttt{\short} & \textbf{75.44 $\pm$ 0.37} & \textcolor[RGB]{34,120,5}{+0.69} & \textbf{92.15 $\pm$ 0.23} & \textcolor[RGB]{34,120,5}{+0.11} & \textbf{77.30 $\pm$ 0.36} & \textcolor[RGB]{34,120,5}{+0.37} & \textbf{93.39 $\pm$ 0.21} & \textcolor[RGB]{34,120,5}{+0.02} \\
\midrule
\multirow{2}{*}{CIFAR-10} & Original & 93.61 $\pm$ 0.14 & - & 99.69 $\pm$ 0.08 & - & 93.53 $\pm$ 0.03 & - & 99.84 $\pm$ 0.02 & - \\
& \texttt{\short} & \textbf{93.82 $\pm$ 0.08} & \textcolor[RGB]{34,120,5}{+0.21} & \textbf{99.82 $\pm$ 0.03} & \textcolor[RGB]{34,120,5}{+0.13} & \textbf{93.93 $\pm$ 0.12} & \textcolor[RGB]{34,120,5}{+0.4} & \textbf{99.86 $\pm$ 0.01} & \textcolor[RGB]{34,120,5}{+0.02} \\
\bottomrule
\end{tabular}%
}
\end{table}

We evaluate our method on the ImageNet~\citep{russakovsky2015imagenetlargescalevisual} and CIFAR-10 datasets~\citep{krizhevsky2009learning}. Our implementation utilizes the ResNet-50 architecture from the PyTorch Image Models (\texttt{timm}) library~\citep{rw2019timm}. To ensure a fair comparison, both our method and the baseline were trained for 200 epochs, with all other hyperparameters held constant. Notably, we retain key default settings from \texttt{timm} for all experiments, including a label smoothing factor of 0.1. This demonstrate \short's complementarity with existing regularizations.

\textbf{Main results.}
Table~\ref{tab:main_results_vision} summarizes the primary classification results, comparing \short\ against the standard ResNet-50 baseline. For these results, we use a minimal entropy of 1.2 for ImageNet and 0.6 for CIFAR-10. The comparison is conducted under two settings: with and without the standard data augmentation provided by the timm library. The results show that \short\ consistently outperforms the baseline across both datasets and settings.

\textbf{Ablation study on minimal entropy.}
We study our method's robustness to the minimal entropy hyperparameter on ImageNet and CIFAR-10, using checkpoints from the 100th and 200th epochs, respectively, for efficiency.  As shown in Figure~\ref{fig:vision_entropy}, our method exhibits low sensitivity to this parameter. Rather than fine-tuning for peak performance, our intent is to show that competitive accuracy is maintained across a reasonable range of values. This demonstrates strong performance is achievable without extensive tuning.

\begin{figure}[t!]
\begin{center}
    \vskip -0.2in
    \begin{subfigure}[b]{0.3\textwidth}
        \centering
        \includegraphics[width=\textwidth]{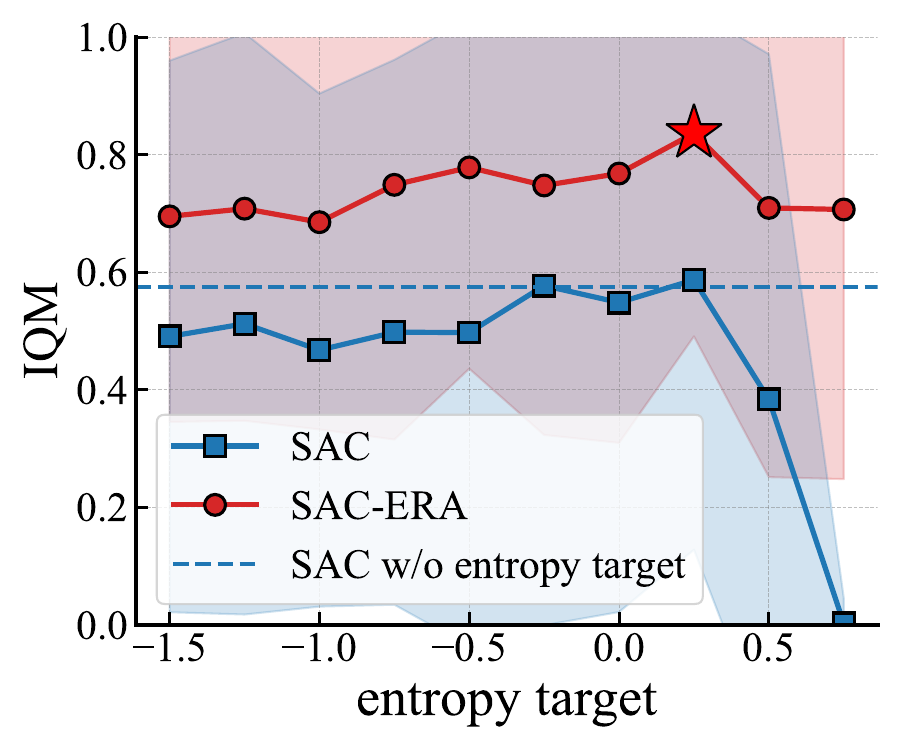}
        \vskip -0.1in
        \subcaption{}
        \vskip -0.1in
        \label{fig:ablation_entropy}
    \end{subfigure}
    \hfill
    \begin{subfigure}[b]{0.66\textwidth}
        \centering
        \includegraphics[width=\textwidth]{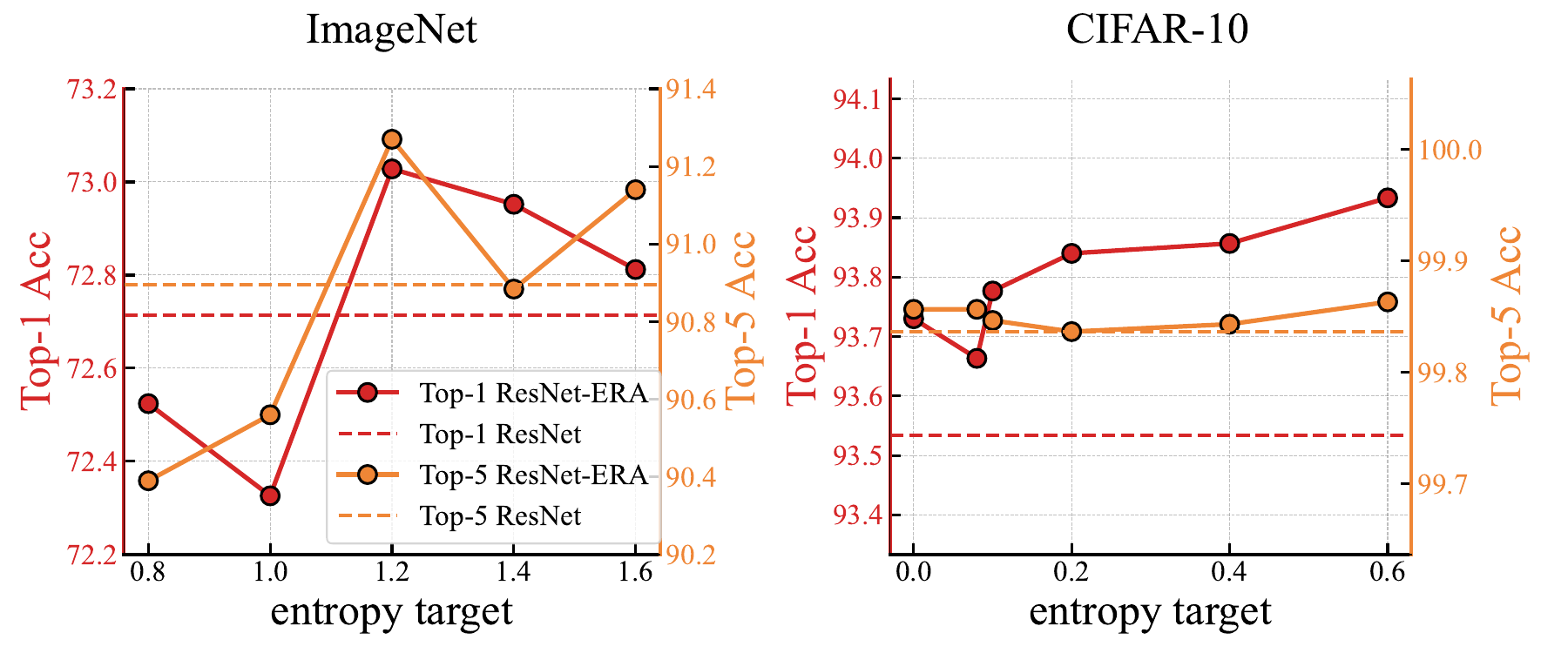}
        \vskip -0.1in
        \subcaption{}
        \vskip -0.1in
        \label{fig:vision_entropy}
    \end{subfigure}
\caption{\textbf{Sensitivity of \short\ to the Minimum Entropy.} 
(a) \textbf{1M Steps Performance on DMC Tasks.} Comparison between SAC-\short\ and the baseline SAC on Humanoid and Dogs environments under various minimum entropy constraints. Our method achieves superior performance across all settings. 
(b) \textbf{Accuracy on ImageNet and CIFAR-10.} ResNet-\short\ maintains stable Top-1 and Top-5 accuracy across a range of minimum entropy values, indicating its robustness to the choice of this hyperparameter.}
\end{center}
\vskip -0.3in
\end{figure}
\subsection{Results and Analysis on Large Language Models}
\label{sec:results_llms}
We first present the results of ERA in \hyperlink{sec5.3.1}{§5.3.1 Main Results} and \hyperlink{sec5.3.2}{§5.3.2 Extension to More Models and Algorithms}. We then use \hyperlink{sec5.3.3}{§5.3.3 Analysis on Entropy and Reasoning Capacity Boundary} and \hyperlink{sec5.3.4}{§5.3.4 Out-of-Distribution Generalization} to illustrate the role of encouraging exploration. Additional ablation studies on method design are provided in the Appendix~\ref{subsec:additional_llm}.
\hypertarget{sec5.3.1}{\subsubsection{Main Results}}
\label{sec5.3.1}
We evaluate ERA on Qwen2.5-Math-7B, trained with the DAPO-Math-17K~\citep{yu2025dapo} dataset using codebase adopted from verl~\citep{sheng2025hybridflow}. To improve training stability and ensure well-controlled entropy decay, we adopt a two-stage training strategy. In the first stage, we set $\omega_{\text{low}}=0.45$, $\omega_{\text{high}}=3.0$, and $k=2$, and train for $600$ steps. In the second stage, we continue training for $500$ steps with a relaxed entropy bound, setting $\omega_{\text{low}}=0.2$, $\omega_{\text{high}}=+\infty$, and keeping $k=2$.

We then evaluate the resulting model on six standard mathematical reasoning tasks: AIME'24, AIME'25, AMC'23~\citep{li2024numinamath}, MATH500~\citep{hendrycks2021measuring}, Minerva~\citep{lewkowycz2022solving}, and OlympiadBench~\citep{dataset_olympiad}. Table~\ref{tab:main_results}~presents comparisons against base models, classical RL methods, and recent entropy-control approaches. AIME’24, AIME’25, and AMC’23 are conducted with a decoding temperature of $0.7$, and reported as the average accuracy over $16$ sampled responses. MATH500, Minerva, and OlympiadBench are conducted with greedy sampling. The evaluation process is sampled on the original policy $z$ (before ERA). Full implementation details and hyperparameter settings are provided in Appendix~\ref{subsec:impl_llm}. The results show that ERA consistently achieves the best results on most of the benchmarks. Notably, it outperforms strong entropy-based baselines such as KL-Cov and Clip-Cov by significant margins.

\begin{table}[h!]
\centering
\caption{Main results ($\%$) on five competition-level reasoning benchmarks based on Qwen2.5-Math-7B. For AIME and AMC, the results are avg.@16. The best results on each benchmark are highlighted in \textbf{bold}.}
\vskip 0.1in
\label{tab:main_results}
\resizebox{0.95\textwidth}{!}{
\begin{tabular}{lccccccc}
\toprule
\textbf{Model}  & \textbf{AIME24 $\uparrow$} & \textbf{AIME25 $\uparrow$} & \textbf{AMC $\uparrow$} & \textbf{MATH500 $\uparrow$} & \textbf{Minerva $\uparrow$} & \textbf{Olympiad $\uparrow$} &  \textbf{Avg. $\uparrow$} \\
\midrule
\multicolumn{8}{l}{\textit{\textbf{Base Models}}} \\
Qwen2.5-Math~\cite{qwen2.5_math}    & 8.6 & 6.3 &  52.2 & 50.8 & 12.1 & 17.2 & 24.5 \\
Qwen2.5-Math-Instruct~\cite{qwen2.5_math}  & 13.3 & 10.0 & 57.1 & 81.0 & 32.7 & 38.8 & 38.8\\\midrule
\multicolumn{8}{l}{\textit{\textbf{Classical Methods}}} \\
SimpleRL-Zero~\cite{zeng2025simplerl}   & 26.7 & 9.3 & 60.0 & 74.6& 27.6 & 35.8 & 39.0 \\
OpenReasoner-Zero~\cite{orz}  & 15.4 & 13.4 & 56.5 & 81.0 & 32.7 & 43.2 & 40.4\\
PRIME-Zero~\cite{prime}   & 18.9 & 11.7 & 57.7 & 79.0 & 36.4 & 40.6 & 40.7 \\
Oat-Zero~\cite{liu2025understanding}    & 28.8 & 10.8 & 65.2 & 79.6 & 34.2 & 39.9 & 43.1 \\
\midrule
\multicolumn{8}{l}{\textit{\textbf{Entropy Control Methods}}} \\
GRPO w/ 20\% Forking Tokens~\citep{wang2025beyond} & 29.0 & \textbf{17.7} & 63.6 & 81.8 & 39.7 & 44.6 & 46.1 \\
KL-Cov~\citep{cui2025entropy} & 35.6 &13.1&65.1&81.0&40.4&44.1&46.6\\
Clip-Cov~\citep{cui2025entropy} & 33.9 & 13.7 &62.5&78.4&35.6&40.3&44.1\\
\midrule
GRPO~\citep{shao2024deepseekmath} & 34.4 &12.3&69.5&80.6&36.8&40.6&45.7\\
\rowcolor{mygray}
ERA  & \textbf{37.5} & {16.9} & \textbf{72.8} & \textbf{84.6} & \textbf{42.6} & \textbf{46.5} & \textbf{50.2} \\
\rowcolor{mygray}
\textbf{$\bigtriangleup$ ($\uparrow$)} & \textcolor[RGB]{34,120,5}{+9.0$\%$}&\textcolor[RGB]{34,120,5}{+37.4$\%$}&\textcolor[RGB]{34,120,5}{+4.7$\%$}&\textcolor[RGB]{34,120,5}{+5.0$\%$}&\textcolor[RGB]{34,120,5}{+15.8$\%$}&\textcolor[RGB]{34,120,5}{+14.5$\%$}&\textcolor[RGB]{34,120,5}{+9.8$\%$} \\
\bottomrule
\end{tabular}}
\vskip -0.1in
\end{table}



\hypertarget{sec5.3.2}{\subsubsection{Extension to More Models and Algorithms}}
\label{sec5.3.2}
To demonstrate ERA’s effectiveness across different model sizes and algorithms, we extend it to the weaker Qwen2.5-Math-1.5B model and also apply ERA to other algorithms such as GSPO~\citep{zheng2025group} on Qwen2.5-Math-7B, showing that ERA is a generic approach not tied to any specific model or algorithm. As reported in Table~\ref{table:moremodels}, ERA yields significant gains on both the smaller model and GSPO. For instance, on Qwen2.5-Math-1.5B it achieves an average improvement of $14.1\%$. 
\begin{table*}[ht!]
\centering
\caption{Accuracy ($\%$) results of different LLMs and different algorithms across six benchmarks. The best results in each box are highlighted in \textbf{bold}.}
\label{table:moremodels}
\resizebox{0.9\linewidth}{!}{
\begin{tabular}{lccccccc}
\toprule
\textbf{Method} & \textbf{AIME24 $\uparrow$} & \textbf{AIME25 $\uparrow$} & \textbf{AMC $\uparrow$} & \textbf{MATH500 $\uparrow$} & \textbf{Minerva $\uparrow$} & \textbf{Olympiad $\uparrow$} &  \textbf{Avg. $\uparrow$}\\
\midrule
\rowcolor{gray!8}\multicolumn{8}{c}{Qwen2.5-Math-1.5B~\cite{qwen2.5_math}}\\
\midrule
CoT  & 4.3 & 2.3 & 26.4 & 59.0 & 24.3 & 27.6 & 24.0 \\
GRPO   & 11.1 & 6.0 & 40.2  & 66.4 & 25.0 & 30.1 & 29.8 \\
ERA    & \textbf{12.1} & \textbf{6.8} & \textbf{49.5} & \textbf{70.6} & \textbf{30.5}&\textbf{34.7} & \textbf{34.0}\\
\textbf{$\bigtriangleup$ $(\uparrow)$} & \textcolor[RGB]{34,120,5}{+9.0$\%$}&\textcolor[RGB]{34,120,5}{+13.3$\%$}&\textcolor[RGB]{34,120,5}{+23.1$\%$}&\textcolor[RGB]{34,120,5}{+6.3$\%$}&\textcolor[RGB]{34,120,5}{+22.0$\%$}&\textcolor[RGB]{34,120,5}{+15.3$\%$}&\textcolor[RGB]{34,120,5}{+14.1$\%$}\\
\midrule
\rowcolor{gray!8}\multicolumn{8}{c}{Qwen2.5-Math-7B~\cite{qwen2.5_math}}\\
\midrule
CoT  & 8.6 & 6.3 &  52.2 & 50.8 & 12.1 & 17.2 & 24.5 \\
GSPO  & 29.8 &13.7&61.2&\textbf{85.1}&37.1&35.1&43.7\\
GSPO + ERA    & \textbf{33.3} & \textbf{15.2}&\textbf{63.8}&84.3& \textbf{40.8}&\textbf{42.7}&\textbf{46.7}\\
\textbf{$\bigtriangleup$ $(\uparrow)$} & \textcolor[RGB]{34,120,5}{+11.7$\%$}&\textcolor[RGB]{34,120,5}{+10.9$\%$}&\textcolor[RGB]{34,120,5}{+4.2$\%$}&\textcolor[RGB]{120,5,34}{-0.9$\%$}&\textcolor[RGB]{34,120,5}{+10.0$\%$}&\textcolor[RGB]{34,120,5}{+21.7$\%$}&\textcolor[RGB]{34,120,5}{+6.9$\%$} \\
\bottomrule
\end{tabular}}
 \vskip -0.1in
\end{table*}

\hypertarget{sec5.3.3}{\subsubsection{Analysis on Entropy and Reasoning Capacity Boundary}}
\label{sec5.3.3}
To better understand the effect of our approach on exploration and reasoning, we examine both the entropy dynamics of the learned policies and their downstream reasoning performance. Figure~\ref{fig:llm_entropy} compares the entropy trajectories of our method (first stage) with the GRPO baseline. While GRPO suffers from entropy collapse, our method maintains a stable entropy level throughout training. This stability indicates the existence of a \emph{non-trivial entropy lower bound}, as we desired by the definition of ERA, which prevents premature policy concentration and preserves the model’s ability to explore diverse reasoning paths.

The presence of this entropy floor aligns with improved reasoning performance. As shown in Figure~\ref{fig:llm_entropy}, ERA achieves consistently higher pass@$k$ scores on AIME'24 and AIME'25 compared to GRPO. This demonstrates that avoiding entropy collapse is not merely a statistical artifact but translates directly into stronger reasoning capacity. In particular, maintaining sufficient entropy ensures the model retains multiple candidate reasoning trajectories, thereby improving the likelihood of successful solutions under pass@$k$ evaluation.

\begin{figure}[h!]
\begin{center}
\centerline{\includegraphics[width=\columnwidth]{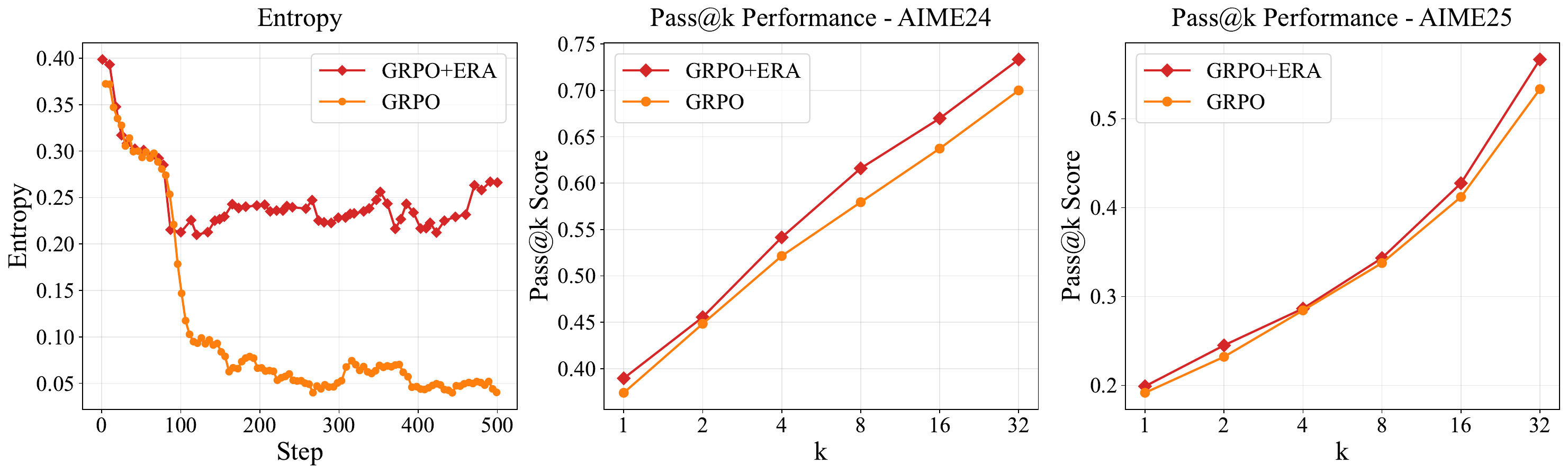}}
\caption{\textbf{Entropy comparison and pass@$k$ results for GRPO with ERA (ours) versus GRPO alone.} 
The entropy curves demonstrate that ERA mitigates entropy collapse and establishes a clear lower bound. 
The pass@$k$ results further indicate that ERA enhances exploration and strengthens the model’s reasoning ability.}
\label{fig:llm_entropy}
\vskip -0.25in
\end{center}
\end{figure}



\hypertarget{sec5.3.3}{\subsubsection{Out-of-Distribution Generalization}}
\label{sec5.3.4}

\begin{wrapfigure}{r}{0.45\textwidth}
  \centering
  \vskip -0.45in
  \includegraphics[width=0.95\linewidth]{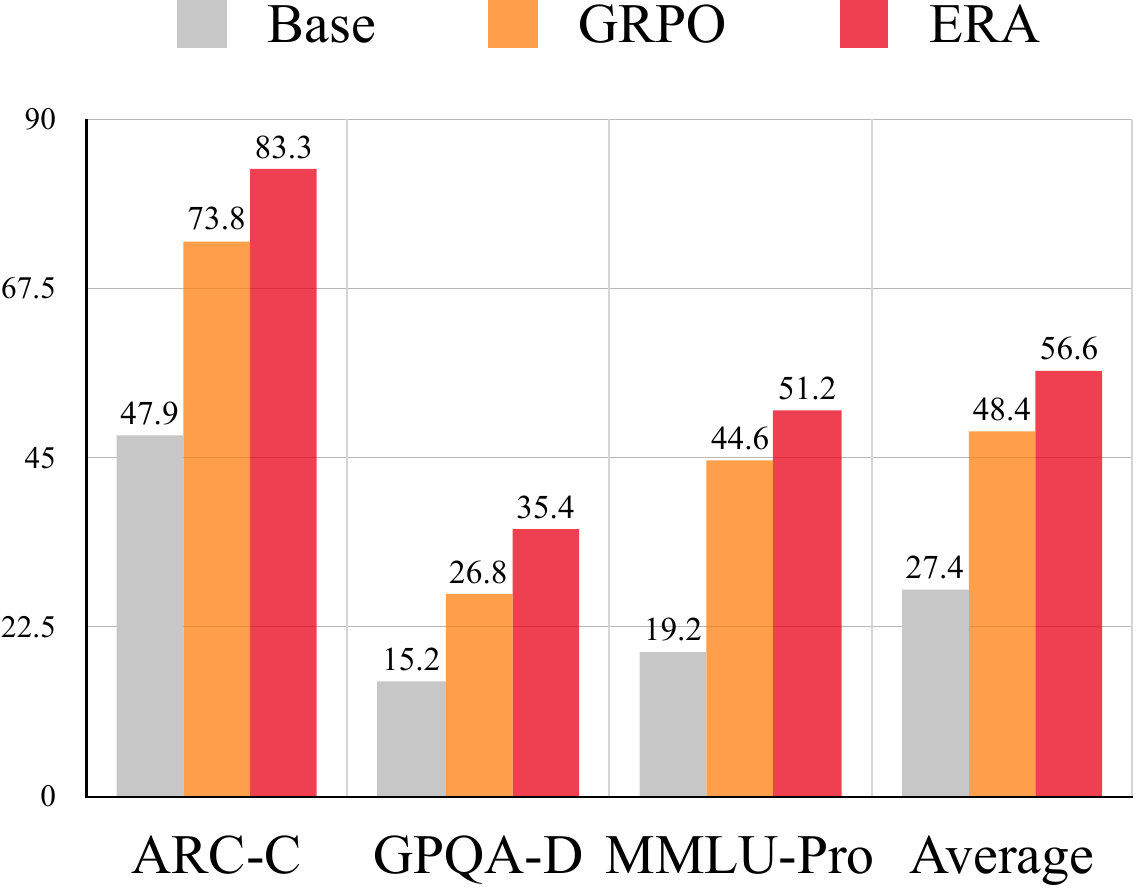}
  \caption{Results on three OOD benchmarks (Qwen2.5-Math-7B).}
  \label{ood_figure}
  \vskip -0.25in
\end{wrapfigure}
Models trained in a specific domain often struggle when applied to other domains~\citep{yuan2023revisiting,0001HH0ZWY0HGJ024}. Since ERA uses entropy constraints to encourage exploration, we hope it can learn \textit{more general skills}. Therefore we want to see if ERA will also do better on out-of-distribution (OOD) data than standard GRPO. To test this, we evaluate ERA on three hard OOD benchmarks: ARC-C~\citep{clark2018think}, GPQA-Diamond~\citep{gpqa}, and MMLU-Pro~\citep{mmlu_pro}. As shown in Figure~\ref{ood_figure}, ERA outperforms GRPO by 16.9$\%$ on average. This confirms our hypothesis that ERA can also enable models to learn more generalizable abilities.

\section{Conclusions}
\label{sec:conclusions}
In this work, we introduced \short, a novel entropy-constrained paradigm built upon the unique principle of treating output activations as a direct medium for entropy regularization. 
Our theoretical analysis is substantiated by strong empirical results across diverse and challenging domains.
In these settings, \short\ consistently surpasses prominent baselines without incurring significant computational overhead. Ultimately, this work offers a new perspective on entropy regularization for both supervised and unsupervised decision-making, opening a promising research avenue for developing more robust and efficient learning agents.
\section*{Reproducibility statement}
\label{sec:reproducibility}
We are strongly committed to the reproducibility of our work. To this end, we provide detailed derivations and proofs for all theoretical claims in the appendix. The appendix also contains comprehensive experimental details, including hyperparameters, environment setups, and additional results, which are crucial for replicating our findings. Furthermore, the core source code for our proposed method, \short, instantiated across all domains, is included in the appendix. As our implementations are built upon publicly available codebases and frameworks, we believe the provided key source code is sufficient for a straightforward reproduction of our results. 
Additionally, a full open-source codebase is available at:\href{https://nothingbutbut.github.io/era}{\textit{https://nothingbutbut.github.io/era}}
\section*{Acknowledgment}
\label{sec:Acknowledgment}
This work was supported by the Tsinghua University Initiative Scientific Research Program No. 20257020004. We would also like to express our gratitude to Kaizhe Hu, Ruizhe Shi, and Huanyu Li from the Tsinghua Embodied AI Lab for their invaluable discussions and insightful feedback, which have significantly contributed to this work.

\bibliography{iclr2026_conference}

\begin{thebibliography}{62}
\providecommand{\natexlab}[1]{#1}
\providecommand{\url}[1]{\texttt{#1}}
\expandafter\ifx\csname urlstyle\endcsname\relax
  \providecommand{\doi}[1]{doi: #1}\else
  \providecommand{\doi}{doi: \begingroup \urlstyle{rm}\Url}\fi

\bibitem[Ba et~al.(2016)Ba, Kiros, and Hinton]{ba2016layernormalization}
Jimmy~Lei Ba, Jamie~Ryan Kiros, and Geoffrey~E. Hinton.
\newblock Layer normalization, 2016.
\newblock URL \url{https://arxiv.org/abs/1607.06450}.

\bibitem[Celik et~al.(2025)Celik, Li, Blessing, Li, Palenicek, Peters, Chalvatzaki, and Neumann]{celik2025dimediffusionbasedmaximumentropyreinforcement}
Onur Celik, Zechu Li, Denis Blessing, Ge~Li, Daniel Palenicek, Jan Peters, Georgia Chalvatzaki, and Gerhard Neumann.
\newblock Dime:diffusion-based maximum entropy reinforcement learning, 2025.
\newblock URL \url{https://arxiv.org/abs/2502.02316}.

\bibitem[Chao et~al.(2024)Chao, Feng, Sun, Lee, See, and Lee]{chao2024maximumentropyreinforcementlearning}
Chen-Hao Chao, Chien Feng, Wei-Fang Sun, Cheng-Kuang Lee, Simon See, and Chun-Yi Lee.
\newblock Maximum entropy reinforcement learning via energy-based normalizing flow, 2024.
\newblock URL \url{https://arxiv.org/abs/2405.13629}.

\bibitem[Cheng et~al.(2025)Cheng, Huang, Zhu, Dai, Zhao, Zhang, and Wei]{cheng2025reasoning}
Daixuan Cheng, Shaohan Huang, Xuekai Zhu, Bo~Dai, Wayne~Xin Zhao, Zhenliang Zhang, and Furu Wei.
\newblock Reasoning with exploration: An entropy perspective.
\newblock \emph{arXiv preprint arXiv:2506.14758}, 2025.

\bibitem[Choe \& Kim(2024)Choe and Kim]{choe2024maximumentropyonpolicyactorcritic}
Jean Seong~Bjorn Choe and Jong-Kook Kim.
\newblock Maximum entropy on-policy actor-critic via entropy advantage estimation, 2024.
\newblock URL \url{https://arxiv.org/abs/2407.18143}.

\bibitem[Clark et~al.(2018)Clark, Cowhey, Etzioni, Khot, Sabharwal, Schoenick, and Tafjord]{clark2018think}
Peter Clark, Isaac Cowhey, Oren Etzioni, Tushar Khot, Ashish Sabharwal, Carissa Schoenick, and Oyvind Tafjord.
\newblock Think you have solved question answering? try arc, the ai2 reasoning challenge.
\newblock \emph{arXiv preprint arXiv:1803.05457}, 2018.

\bibitem[Cui et~al.(2025{\natexlab{a}})Cui, Yuan, Wang, Wang, Li, He, Fan, Yu, Xu, Chen, et~al.]{prime}
Ganqu Cui, Lifan Yuan, Zefan Wang, Hanbin Wang, Wendi Li, Bingxiang He, Yuchen Fan, Tianyu Yu, Qixin Xu, Weize Chen, et~al.
\newblock Process reinforcement through implicit rewards.
\newblock \emph{arXiv preprint arXiv:2502.01456}, 2025{\natexlab{a}}.

\bibitem[Cui et~al.(2025{\natexlab{b}})Cui, Zhang, Chen, Yuan, Wang, Zuo, Li, Fan, Chen, Chen, et~al.]{cui2025entropy}
Ganqu Cui, Yuchen Zhang, Jiacheng Chen, Lifan Yuan, Zhi Wang, Yuxin Zuo, Haozhan Li, Yuchen Fan, Huayu Chen, Weize Chen, et~al.
\newblock The entropy mechanism of reinforcement learning for reasoning language models.
\newblock \emph{arXiv preprint arXiv:2505.22617}, 2025{\natexlab{b}}.

\bibitem[Fujimoto et~al.(2018)Fujimoto, van Hoof, and Meger]{fujimoto2018addressingfunctionapproximationerror}
Scott Fujimoto, Herke van Hoof, and David Meger.
\newblock Addressing function approximation error in actor-critic methods, 2018.
\newblock URL \url{https://arxiv.org/abs/1802.09477}.

\bibitem[Guo et~al.(2025)Guo, Yang, Zhang, Song, Zhang, Xu, Zhu, Ma, Wang, Bi, et~al.]{guo2025deepseek}
Daya Guo, Dejian Yang, Haowei Zhang, Junxiao Song, Ruoyu Zhang, Runxin Xu, Qihao Zhu, Shirong Ma, Peiyi Wang, Xiao Bi, et~al.
\newblock Deepseek-r1: Incentivizing reasoning capability in llms via reinforcement learning.
\newblock \emph{arXiv preprint arXiv:2501.12948}, 2025.

\bibitem[Haarnoja et~al.(2017)Haarnoja, Tang, Abbeel, and Levine]{haarnoja2017reinforcement}
Tuomas Haarnoja, Haoran Tang, Pieter Abbeel, and Sergey Levine.
\newblock Reinforcement learning with deep energy-based policies.
\newblock In \emph{International conference on machine learning}, pp.\  1352--1361. PMLR, 2017.

\bibitem[Haarnoja et~al.(2018)Haarnoja, Zhou, Abbeel, and Levine]{haarnoja2018soft}
Tuomas Haarnoja, Aurick Zhou, Pieter Abbeel, and Sergey Levine.
\newblock Soft actor-critic: Off-policy maximum entropy deep reinforcement learning with a stochastic actor.
\newblock In \emph{International conference on machine learning}, pp.\  1861--1870. Pmlr, 2018.

\bibitem[He et~al.(2024)He, Luo, Bai, Hu, Thai, Shen, Hu, Han, Huang, Zhang, et~al.]{dataset_olympiad}
Chaoqun He, Renjie Luo, Yuzhuo Bai, Shengding Hu, Zhen Thai, Junhao Shen, Jinyi Hu, Xu~Han, Yujie Huang, Yuxiang Zhang, et~al.
\newblock Olympiadbench: A challenging benchmark for promoting agi with olympiad-level bilingual multimodal scientific problems.
\newblock In \emph{Proceedings of the 62nd Annual Meeting of the Association for Computational Linguistics (Volume 1: Long Papers)}, pp.\  3828--3850, 2024.

\bibitem[Hendrycks et~al.(2021)Hendrycks, Burns, Kadavath, Arora, Basart, Tang, Song, and Steinhardt]{hendrycks2021measuring}
Dan Hendrycks, Collin Burns, Saurav Kadavath, Akul Arora, Steven Basart, Eric Tang, Dawn Song, and Jacob Steinhardt.
\newblock Measuring mathematical problem solving with the {MATH} dataset.
\newblock In \emph{Thirty-fifth Conference on Neural Information Processing Systems Datasets and Benchmarks Track (Round 2)}, 2021.

\bibitem[Hu et~al.(2025)Hu, Zhang, Han, Jiang, Zhang, and Shum]{orz}
Jingcheng Hu, Yinmin Zhang, Qi~Han, Daxin Jiang, Xiangyu Zhang, and Heung-Yeung Shum.
\newblock Open-reasoner-zero: An open source approach to scaling up reinforcement learning on the base model, 2025.
\newblock URL \url{https://arxiv.org/abs/2503.24290}.

\bibitem[Ioffe \& Szegedy(2015)Ioffe and Szegedy]{ioffe2015batchnormalizationacceleratingdeep}
Sergey Ioffe and Christian Szegedy.
\newblock Batch normalization: Accelerating deep network training by reducing internal covariate shift, 2015.
\newblock URL \url{https://arxiv.org/abs/1502.03167}.

\bibitem[Jaech et~al.(2024)Jaech, Kalai, Lerer, Richardson, El-Kishky, Low, Helyar, Madry, Beutel, Carney, et~al.]{jaech2024openai}
Aaron Jaech, Adam Kalai, Adam Lerer, Adam Richardson, Ahmed El-Kishky, Aiden Low, Alec Helyar, Aleksander Madry, Alex Beutel, Alex Carney, et~al.
\newblock Openai o1 system card.
\newblock \emph{arXiv preprint arXiv:2412.16720}, 2024.

\bibitem[Kang et~al.(2025)Kang, Hu, Luo, Yuan, Zheng, and Xu]{kang2025forgetandgrowstrategydeepreinforcement}
Zilin Kang, Chenyuan Hu, Yu~Luo, Zhecheng Yuan, Ruijie Zheng, and Huazhe Xu.
\newblock A forget-and-grow strategy for deep reinforcement learning scaling in continuous control, 2025.
\newblock URL \url{https://arxiv.org/abs/2507.02712}.

\bibitem[Kober et~al.(2013)Kober, Bagnell, and Peters]{kober2013reinforcement}
Jens Kober, J~Andrew Bagnell, and Jan Peters.
\newblock Reinforcement learning in robotics: A survey.
\newblock \emph{The International Journal of Robotics Research}, 32\penalty0 (11):\penalty0 1238--1274, 2013.

\bibitem[Kostrikov(2021)]{jaxrl}
Ilya Kostrikov.
\newblock {JAXRL: Implementations of Reinforcement Learning algorithms in JAX}, 10 2021.
\newblock URL \url{https://github.com/ikostrikov/jaxrl}.

\bibitem[Krizhevsky et~al.(2009)]{krizhevsky2009learning}
Alex Krizhevsky et~al.
\newblock Learning multiple layers of features from tiny images.
\newblock 2009.

\bibitem[Lewkowycz et~al.(2022)Lewkowycz, Andreassen, Dohan, Dyer, Michalewski, Ramasesh, Slone, Anil, Schlag, Gutman-Solo, et~al.]{lewkowycz2022solving}
Aitor Lewkowycz, Anders Andreassen, David Dohan, Ethan Dyer, Henryk Michalewski, Vinay Ramasesh, Ambrose Slone, Cem Anil, Imanol Schlag, Theo Gutman-Solo, et~al.
\newblock Solving quantitative reasoning problems with language models.
\newblock \emph{Advances in neural information processing systems}, 35:\penalty0 3843--3857, 2022.

\bibitem[Li et~al.(2024{\natexlab{a}})Li, Wang, Hu, Wei, Zheng, Hu, Zhang, and Peng]{li2024common}
Chen Li, Weiqi Wang, Jingcheng Hu, Yixuan Wei, Nanning Zheng, Han Hu, Zheng Zhang, and Houwen Peng.
\newblock Common 7b language models already possess strong math capabilities.
\newblock \emph{arXiv preprint arXiv:2403.04706}, 2024{\natexlab{a}}.

\bibitem[Li et~al.(2024{\natexlab{b}})Li, Beeching, Tunstall, Lipkin, Soletskyi, Huang, Rasul, Yu, Jiang, Shen, et~al.]{li2024numinamath}
Jia Li, Edward Beeching, Lewis Tunstall, Ben Lipkin, Roman Soletskyi, Shengyi Huang, Kashif Rasul, Longhui Yu, Albert~Q. Jiang, Ziju Shen, et~al.
\newblock Numinamath: The largest public dataset in ai4maths with 860k pairs of competition math problems and solutions.
\newblock \url{https://huggingface.co/datasets/Numinamath}, 2024{\natexlab{b}}.
\newblock Hugging Face repository, 13:9.

\bibitem[Li(2022)]{li2022drlcode}
Zhi Li.
\newblock {DRL-code-pytorch}.
\newblock \url{https://github.com/Lizhi-sjtu/DRL-code-pytorch}, 2022.
\newblock Accessed: 2025-09-10.

\bibitem[Liu et~al.(2025)Liu, Chen, Li, Qi, Pang, Du, Lee, and Lin]{liu2025understanding}
Zichen Liu, Changyu Chen, Wenjun Li, Penghui Qi, Tianyu Pang, Chao Du, Wee~Sun Lee, and Min Lin.
\newblock Understanding r1-zero-like training: A critical perspective.
\newblock \emph{arXiv preprint arXiv:2503.20783}, 2025.

\bibitem[Luo et~al.(2024)Luo, Ji, Sun, Zhang, Xu, and Zhan]{luo2024offline}
Yu~Luo, Tianying Ji, Fuchun Sun, Jianwei Zhang, Huazhe Xu, and Xianyuan Zhan.
\newblock Offline-boosted actor-critic: Adaptively blending optimal historical behaviors in deep off-policy rl.
\newblock \emph{arXiv preprint arXiv:2405.18520}, 2024.

\bibitem[Ma et~al.(2025)Ma, Chen, Wang, Li, and Dai]{ma2025efficientonlinereinforcementlearning}
Haitong Ma, Tianyi Chen, Kai Wang, Na~Li, and Bo~Dai.
\newblock Efficient online reinforcement learning for diffusion policy, 2025.
\newblock URL \url{https://arxiv.org/abs/2502.00361}.

\bibitem[Nachum et~al.(2017)Nachum, Norouzi, Xu, and Schuurmans]{nachum2017bridging}
Ofir Nachum, Mohammad Norouzi, Kelvin Xu, and Dale Schuurmans.
\newblock Bridging the gap between value and policy based reinforcement learning.
\newblock \emph{Advances in neural information processing systems}, 30, 2017.

\bibitem[Nauman et~al.(2024)Nauman, Bortkiewicz, Miłoś, Trzciński, Ostaszewski, and Cygan]{nauman2024overestimationoverfittingplasticityactorcritic}
Michal Nauman, Michał Bortkiewicz, Piotr Miłoś, Tomasz Trzciński, Mateusz Ostaszewski, and Marek Cygan.
\newblock Overestimation, overfitting, and plasticity in actor-critic: the bitter lesson of reinforcement learning, 2024.
\newblock URL \url{https://arxiv.org/abs/2403.00514}.

\bibitem[O'Donoghue et~al.(2016)O'Donoghue, Munos, Kavukcuoglu, and Mnih]{o2016combining}
Brendan O'Donoghue, Remi Munos, Koray Kavukcuoglu, and Volodymyr Mnih.
\newblock Combining policy gradient and q-learning.
\newblock \emph{arXiv preprint arXiv:1611.01626}, 2016.

\bibitem[Ouyang et~al.(2022)Ouyang, Wu, Jiang, Almeida, Wainwright, Mishkin, Zhang, Agarwal, Slama, Ray, et~al.]{ouyang2022training}
Long Ouyang, Jeffrey Wu, Xu~Jiang, Diogo Almeida, Carroll Wainwright, Pamela Mishkin, Chong Zhang, Sandhini Agarwal, Katarina Slama, Alex Ray, et~al.
\newblock Training language models to follow instructions with human feedback.
\newblock \emph{Advances in neural information processing systems}, 35:\penalty0 27730--27744, 2022.

\bibitem[Rein et~al.(2024)Rein, Hou, Stickland, Petty, Pang, Dirani, Michael, and Bowman]{gpqa}
David Rein, Betty~Li Hou, Asa~Cooper Stickland, Jackson Petty, Richard~Yuanzhe Pang, Julien Dirani, Julian Michael, and Samuel~R. Bowman.
\newblock {GPQA}: A graduate-level google-proof q\&a benchmark.
\newblock In \emph{First Conference on Language Modeling}, 2024.
\newblock URL \url{https://openreview.net/forum?id=Ti67584b98}.

\bibitem[Russakovsky et~al.(2015)Russakovsky, Deng, Su, Krause, Satheesh, Ma, Huang, Karpathy, Khosla, Bernstein, Berg, and Fei-Fei]{russakovsky2015imagenetlargescalevisual}
Olga Russakovsky, Jia Deng, Hao Su, Jonathan Krause, Sanjeev Satheesh, Sean Ma, Zhiheng Huang, Andrej Karpathy, Aditya Khosla, Michael Bernstein, Alexander~C. Berg, and Li~Fei-Fei.
\newblock Imagenet large scale visual recognition challenge, 2015.
\newblock URL \url{https://arxiv.org/abs/1409.0575}.

\bibitem[Schulman et~al.(2017{\natexlab{a}})Schulman, Wolski, Dhariwal, Radford, and Klimov]{ppo}
John Schulman, Filip Wolski, Prafulla Dhariwal, Alec Radford, and Oleg Klimov.
\newblock Proximal policy optimization algorithms.
\newblock \emph{arXiv preprint arXiv:1707.06347}, 2017{\natexlab{a}}.

\bibitem[Schulman et~al.(2017{\natexlab{b}})Schulman, Wolski, Dhariwal, Radford, and Klimov]{schulman2017proximal}
John Schulman, Filip Wolski, Prafulla Dhariwal, Alec Radford, and Oleg Klimov.
\newblock Proximal policy optimization algorithms.
\newblock \emph{arXiv preprint arXiv:1707.06347}, 2017{\natexlab{b}}.

\bibitem[Seo et~al.(2025)Seo, Sferrazza, Geng, Nauman, Yin, and Abbeel]{seo2025fasttd3}
Younggyo Seo, Carmelo Sferrazza, Haoran Geng, Michal Nauman, Zhao-Heng Yin, and Pieter Abbeel.
\newblock Fasttd3: Simple, fast, and capable reinforcement learning for humanoid control.
\newblock \emph{arXiv preprint arXiv:2505.22642}, 2025.

\bibitem[Sferrazza et~al.(2024)Sferrazza, Huang, Lin, Lee, and Abbeel]{sferrazza2024humanoidbench}
Carmelo Sferrazza, Dun-Ming Huang, Xingyu Lin, Youngwoon Lee, and Pieter Abbeel.
\newblock Humanoidbench: Simulated humanoid benchmark for whole-body locomotion and manipulation.
\newblock \emph{arXiv preprint arXiv:2403.10506}, 2024.

\bibitem[Shao et~al.(2024)Shao, Wang, Zhu, Xu, Song, Bi, Zhang, Zhang, Li, Wu, et~al.]{shao2024deepseekmath}
Zhihong Shao, Peiyi Wang, Qihao Zhu, Runxin Xu, Junxiao Song, Xiao Bi, Haowei Zhang, Mingchuan Zhang, YK~Li, Yang Wu, et~al.
\newblock Deepseekmath: Pushing the limits of mathematical reasoning in open language models.
\newblock \emph{arXiv preprint arXiv:2402.03300}, 2024.

\bibitem[Sheng et~al.(2025)Sheng, Zhang, Ye, Wu, Zhang, Zhang, Peng, Lin, and Wu]{sheng2025hybridflow}
Guangming Sheng, Chi Zhang, Zilingfeng Ye, Xibin Wu, Wang Zhang, Ru~Zhang, Yanghua Peng, Haibin Lin, and Chuan Wu.
\newblock Hybridflow: A flexible and efficient rlhf framework.
\newblock In \emph{Proceedings of the Twentieth European Conference on Computer Systems}, pp.\  1279--1297, 2025.

\bibitem[Sutton et~al.(1998)Sutton, Barto, et~al.]{sutton1998reinforcement}
Richard~S Sutton, Andrew~G Barto, et~al.
\newblock \emph{Reinforcement learning: An introduction}, volume~1.
\newblock MIT press Cambridge, 1998.

\bibitem[Sutton et~al.(1999)Sutton, McAllester, Singh, and Mansour]{NIPS1999_464d828b}
Richard~S Sutton, David McAllester, Satinder Singh, and Yishay Mansour.
\newblock Policy gradient methods for reinforcement learning with function approximation.
\newblock In S.~Solla, T.~Leen, and K.~M\"{u}ller (eds.), \emph{Advances in Neural Information Processing Systems}, volume~12. MIT Press, 1999.
\newblock URL \url{https://proceedings.neurips.cc/paper_files/paper/1999/file/464d828b85b0bed98e80ade0a5c43b0f-Paper.pdf}.

\bibitem[Szegedy et~al.(2016)Szegedy, Vanhoucke, Ioffe, Shlens, and Wojna]{szegedy2016rethinking}
Christian Szegedy, Vincent Vanhoucke, Sergey Ioffe, Jon Shlens, and Zbigniew Wojna.
\newblock Rethinking the inception architecture for computer vision.
\newblock In \emph{Proceedings of the IEEE conference on computer vision and pattern recognition}, pp.\  2818--2826, 2016.

\bibitem[Tassa et~al.(2018)Tassa, Doron, Muldal, Erez, Li, Casas, Budden, Abdolmaleki, Merel, Lefrancq, et~al.]{tassa2018deepmind}
Yuval Tassa, Yotam Doron, Alistair Muldal, Tom Erez, Yazhe Li, Diego de~Las Casas, David Budden, Abbas Abdolmaleki, Josh Merel, Andrew Lefrancq, et~al.
\newblock Deepmind control suite.
\newblock \emph{arXiv preprint arXiv:1801.00690}, 2018.

\bibitem[Team et~al.(2025)Team, Du, Gao, Xing, Jiang, Chen, Li, Xiao, Du, Liao, et~al.]{team2025kimi}
Kimi Team, Angang Du, Bofei Gao, Bowei Xing, Changjiu Jiang, Cheng Chen, Cheng Li, Chenjun Xiao, Chenzhuang Du, Chonghua Liao, et~al.
\newblock Kimi k1. 5: Scaling reinforcement learning with llms.
\newblock \emph{arXiv preprint arXiv:2501.12599}, 2025.

\bibitem[Todorov et~al.(2012)Todorov, Erez, and Tassa]{todorov2012mujoco}
Emanuel Todorov, Tom Erez, and Yuval Tassa.
\newblock Mujoco: A physics engine for model-based control.
\newblock In \emph{2012 IEEE/RSJ international conference on intelligent robots and systems}, pp.\  5026--5033. IEEE, 2012.

\bibitem[Wang et~al.(2024{\natexlab{a}})Wang, Hu, Hou, Chen, Zheng, Wang, Yang, Ye, Huang, Geng, Jiao, Zhang, and Xie]{0001HH0ZWY0HGJ024}
Jindong Wang, Xixu Hu, Wenxin Hou, Hao Chen, Runkai Zheng, Yidong Wang, Linyi Yang, Wei Ye, Haojun Huang, Xiubo Geng, Binxing Jiao, Yue Zhang, and Xing Xie.
\newblock On the robustness of chatgpt: An adversarial and out-of-distribution perspective.
\newblock \emph{IEEE Data Eng. Bull.}, 47\penalty0 (1):\penalty0 48--62, 2024{\natexlab{a}}.

\bibitem[Wang et~al.(2025)Wang, Yu, Gao, Zheng, Liu, Lu, Dang, Chen, Yang, Zhang, et~al.]{wang2025beyond}
Shenzhi Wang, Le~Yu, Chang Gao, Chujie Zheng, Shixuan Liu, Rui Lu, Kai Dang, Xionghui Chen, Jianxin Yang, Zhenru Zhang, et~al.
\newblock Beyond the 80/20 rule: High-entropy minority tokens drive effective reinforcement learning for llm reasoning.
\newblock \emph{arXiv preprint arXiv:2506.01939}, 2025.

\bibitem[Wang et~al.(2024{\natexlab{b}})Wang, Ma, Zhang, Ni, Chandra, Guo, Ren, Arulraj, He, Jiang, et~al.]{mmlu_pro}
Yubo Wang, Xueguang Ma, Ge~Zhang, Yuansheng Ni, Abhranil Chandra, Shiguang Guo, Weiming Ren, Aaran Arulraj, Xuan He, Ziyan Jiang, et~al.
\newblock Mmlu-pro: A more robust and challenging multi-task language understanding benchmark.
\newblock \emph{arXiv preprint arXiv:2406.01574}, 2024{\natexlab{b}}.

\bibitem[Wei et~al.(2022)Wei, Wang, Schuurmans, Bosma, Xia, Chi, Le, Zhou, et~al.]{wei2022chain}
Jason Wei, Xuezhi Wang, Dale Schuurmans, Maarten Bosma, Fei Xia, Ed~Chi, Quoc~V Le, Denny Zhou, et~al.
\newblock Chain-of-thought prompting elicits reasoning in large language models.
\newblock \emph{Advances in neural information processing systems}, 35:\penalty0 24824--24837, 2022.

\bibitem[Wightman(2019)]{rw2019timm}
Ross Wightman.
\newblock Pytorch image models.
\newblock \url{https://github.com/rwightman/pytorch-image-models}, 2019.

\bibitem[Yang et~al.(2024{\natexlab{a}})Yang, Zhang, Hui, Gao, Yu, Li, Liu, Tu, Zhou, Lin, Lu, Xue, Lin, Liu, Ren, and Zhang]{qwen2.5_math}
An~Yang, Beichen Zhang, Binyuan Hui, Bofei Gao, Bowen Yu, Chengpeng Li, Dayiheng Liu, Jianhong Tu, Jingren Zhou, Junyang Lin, Keming Lu, Mingfeng Xue, Runji Lin, Tianyu Liu, Xingzhang Ren, and Zhenru Zhang.
\newblock Qwen2.5-math technical report: Toward mathematical expert model via self-improvement, 2024{\natexlab{a}}.
\newblock URL \url{https://arxiv.org/abs/2409.12122}.

\bibitem[Yang et~al.(2024{\natexlab{b}})Yang, Zhang, Hui, Gao, Yu, Li, Liu, Tu, Zhou, Lin, et~al.]{yang2024qwen2}
An~Yang, Beichen Zhang, Binyuan Hui, Bofei Gao, Bowen Yu, Chengpeng Li, Dayiheng Liu, Jianhong Tu, Jingren Zhou, Junyang Lin, et~al.
\newblock Qwen2. 5-math technical report: Toward mathematical expert model via self-improvement.
\newblock \emph{arXiv preprint arXiv:2409.12122}, 2024{\natexlab{b}}.

\bibitem[Yeo et~al.(2025)Yeo, Tong, Niu, Neubig, and Yue]{yeo2025demystifying}
Edward Yeo, Yuxuan Tong, Xinyao Niu, Graham Neubig, and Xiang Yue.
\newblock Demystifying long chain-of-thought reasoning in {LLM}s.
\newblock In \emph{ICLR 2025 Workshop on Navigating and Addressing Data Problems for Foundation Models}, 2025.
\newblock URL \url{https://openreview.net/forum?id=AgtQlhMQ0V}.

\bibitem[Yu et~al.(2025)Yu, Zhang, Zhu, Yuan, Zuo, Yue, Dai, Fan, Liu, Liu, et~al.]{yu2025dapo}
Qiying Yu, Zheng Zhang, Ruofei Zhu, Yufeng Yuan, Xiaochen Zuo, Yu~Yue, Weinan Dai, Tiantian Fan, Gaohong Liu, Lingjun Liu, et~al.
\newblock Dapo: An open-source llm reinforcement learning system at scale.
\newblock \emph{arXiv preprint arXiv:2503.14476}, 2025.

\bibitem[Yuan et~al.(2023)Yuan, Chen, Cui, Gao, Zou, Cheng, Ji, Liu, and Sun]{yuan2023revisiting}
Lifan Yuan, Yangyi Chen, Ganqu Cui, Hongcheng Gao, FangYuan Zou, Xingyi Cheng, Heng Ji, Zhiyuan Liu, and Maosong Sun.
\newblock Revisiting out-of-distribution robustness in {NLP}: Benchmarks, analysis, and {LLM}s evaluations.
\newblock In \emph{Thirty-seventh Conference on Neural Information Processing Systems Datasets and Benchmarks Track}, 2023.

\bibitem[Yuan et~al.(2025)Yuan, Wei, Gu, Hua, Liang, Chen, and Xu]{yuan2025hermes}
Zhecheng Yuan, Tianming Wei, Langzhe Gu, Pu~Hua, Tianhai Liang, Yuanpei Chen, and Huazhe Xu.
\newblock Hermes: Human-to-robot embodied learning from multi-source motion data for mobile dexterous manipulation.
\newblock \emph{arXiv preprint arXiv:2508.20085}, 2025.

\bibitem[Zeng et~al.(2025)Zeng, Huang, Liu, Liu, He, Ma, and He]{zeng2025simplerl}
Weihao Zeng, Yuzhen Huang, Qian Liu, Wei Liu, Keqing He, Zejun Ma, and Junxian He.
\newblock Simplerl-zoo: Investigating and taming zero reinforcement learning for open base models in the wild.
\newblock \emph{arXiv preprint arXiv:2503.18892}, 2025.

\bibitem[Zheng et~al.(2025)Zheng, Liu, Li, Chen, Yu, Gao, Dang, Liu, Men, Yang, et~al.]{zheng2025group}
Chujie Zheng, Shixuan Liu, Mingze Li, Xiong-Hui Chen, Bowen Yu, Chang Gao, Kai Dang, Yuqiong Liu, Rui Men, An~Yang, et~al.
\newblock Group sequence policy optimization.
\newblock \emph{arXiv preprint arXiv:2507.18071}, 2025.

\bibitem[Zhong et~al.(2024)Zhong, Yang, Zhang, Jiang, XU, and Zhao]{zhong2024maximum}
Dianyu Zhong, Yiqin Yang, Ziyou Zhang, Yuhua Jiang, Bo~XU, and Qianchuan Zhao.
\newblock Maximum next-state entropy for efficient reinforcement learning, 2024.
\newblock URL \url{https://openreview.net/forum?id=0G6rRLYcxm}.

\bibitem[Ziebart(2010)]{ziebart2010modeling}
Brian~D Ziebart.
\newblock \emph{Modeling purposeful adaptive behavior with the principle of maximum causal entropy}.
\newblock Carnegie Mellon University, 2010.

\bibitem[Ziebart et~al.(2008)Ziebart, Maas, Bagnell, Dey, et~al.]{ziebart2008maximum}
Brian~D Ziebart, Andrew~L Maas, J~Andrew Bagnell, Anind~K Dey, et~al.
\newblock Maximum entropy inverse reinforcement learning.
\newblock In \emph{Aaai}, volume~8, pp.\  1433--1438. Chicago, IL, USA, 2008.

\end{thebibliography}
\bibliographystyle{iclr2026_conference}

\newpage
\appendix
\section{Implementation Details}
\label{sec:impl_details}
\subsection{Implementation Details of Continuous Control Tasks}
\label{subsec:impl_continuous}
\subsubsection{Code Implementation of \short\ in Continuous Control}
\begin{figure}[h!]
\begin{minipage}{0.48\textwidth}
\begin{lstlisting}[style=python, caption={Original Implementation}, label={code:original}]
# Original implementation from the jaxrl codebase, suggested by Ilya
# log_std_min, log_std_max: bounds for log standard deviation
# action_dim: dimension of the action space
# pre_stds: direct output from the actor network
log_stds = log_std_min + (log_std_max - log_std_min) * 0.5 * (1 + nn.tanh(pre_stds))
\end{lstlisting}
\end{minipage}\hfill
\begin{minipage}{0.48\textwidth}
\begin{lstlisting}[style=python, caption={\short\ Implementation}, label={code:modified}]
# h_0: target entropy, can be a fixed value or a learnable parameter
# action_dim: dimension of the action space
k = - self.action_dim * (log_std_max + h_0 + jnp.log(jnp.sqrt(2 * jnp.pi * jnp.e)))
log_stds = k * nn.softmax(pre_stds, axis = -1) + log_std_max
log_stds = jax.clip(log_stds, self.log_std_min, self.log_std_max)
\end{lstlisting}
\end{minipage}
\caption{Comparison of the activation function at the actor's output.}
\label{fig:code_comparison}
\end{figure}
We provide the following JAX implementation snippet of \short\ for the reader's reference, where h\_0 is the target entropy ($\mathcal{H}_0'$ in Eq.~\ref{eq:continuous_era}), which can be a constant (e.g., -action\_dim/2) or a learnable parameter. The terms log\_std\_min and log\_std\_max represent the lower and upper bounds of the log standard deviation, respectively; action\_dim is the dimension of the action space; and pre\_stds refers to the raw output of the actor network.

\subsubsection{Environments}
\label{subsubsec:continuous_envs}
\begin{figure}[h!]
\begin{center}
\centerline{\includegraphics[width=0.9\textwidth]{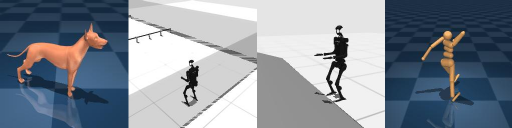}}
\caption{\textbf{Visualization of some continuous control environments used in our experiments.} From left to right: dog-run (DMC), h1-hurdle-v0 (HumanoidBench), h1hand-slide-v0 (HumanoidBench), humanoid-walk (DMC)}
\label{fig:continuous_visualization}
\end{center}
\vskip -0.35in
\end{figure}
Our evaluation of \short\ spans a diverse set of continuous control tasks from three established benchmarks: Mujoco Gym~\citep{todorov2012mujoco}, DeepMind Control Suite (DMC)~\citep{tassa2018deepmind}, and HumanoidBench~\citep{sferrazza2024humanoidbench}. For the Mujoco Gym and DMC environments, we utilized their standard, unmodified configurations. For HumanoidBench, we introduced specific modifications for certain agents.

For experiments involving SAC and OBAC on HumanoidBench, we implemented an action repeat of 2 and disabled episode termination. These adjustments were necessary because the standard tasks proved exceedingly challenging for a baseline SAC agent, as demonstrated in Figure~\ref{fig:humanoidbench_sac}. Conversely, for the FastSAC agent, which is capable of solving the original tasks, we used the standard HumanoidBench environments without these modifications.

For our comparison against TD-MPC2 on DMC environments, we used the performance data reported in the original manuscript. We therefore adhered to their experimental setup, which includes an action repeat of 2.

For main results and training curves, we report results over 10 random seeds for SAC, OBAC, and FastSAC, 5 seeds for PPO, and 3 seeds for TD-MPC2, matching the number provided in its original publication.

\begin{figure}[h!]
\begin{center}
\centerline{\includegraphics[width=\columnwidth]{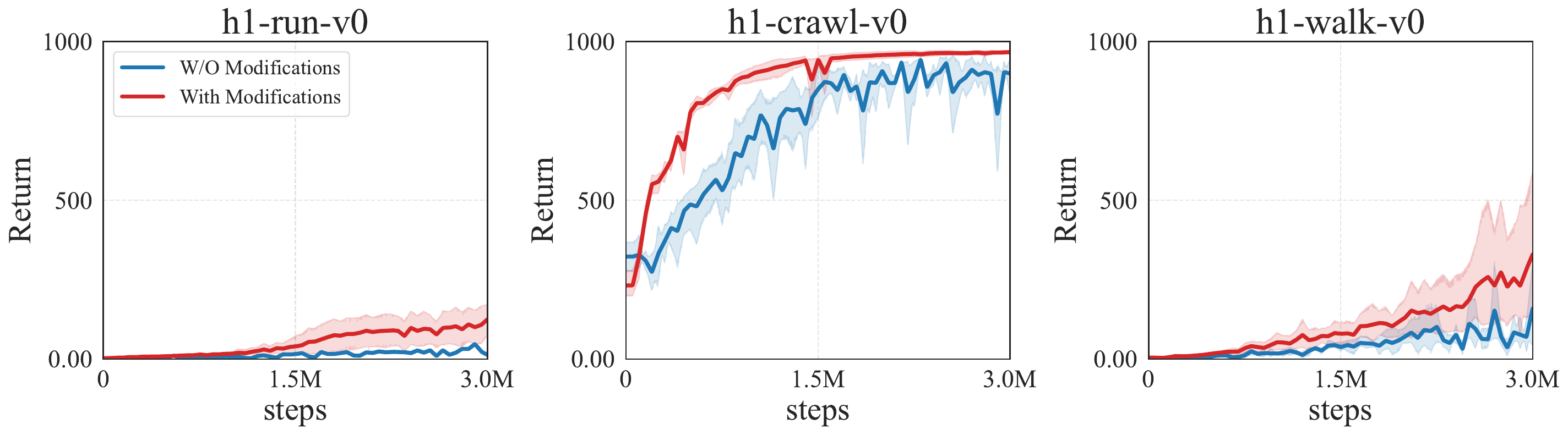}}
\caption{\textbf{Ablation of Environment Modifications for HumanoidBench.} Performance comparison of a standard SAC agent on three challenging HumanoidBench tasks with and without our modified settings (action repeat of 2 and disabled termination). The significant performance gap justifies using these modified settings for our main SAC-based experiments.}
\label{fig:humanoidbench_sac}
\end{center}
\vskip -0.2in
\end{figure}

The action, observation spaces and maximal episode length of the respective environments are shown in Table~\ref{tab:continuous_envs1} and Table~\ref{tab:continuous_envs2}.

\begin{table}[ht!]
\centering
\vspace{-0.2cm}
\caption{List of tasks from DeepMind Control and MetaWorld on which the agents were ablated. The table also contains the dimensions of action, observation space and maximal episode length.}
\label{tab:continuous_envs1}
 \begin{tabular}{l |c | c | c} 
\toprule
 \textbf{Task} & \textbf{Observation dimension} & \textbf{Action dimension} & \textbf{Max episode length} \\
\midrule
\multicolumn{4}{c}{\textsc{DeepMind Control}} \\
\midrule
Dog-Trot & $223$ & $38$ & $1000$ \\
Dog-Walk & $223$ & $38$ & $1000$ \\
Dog-Run & $223$ & $38$ & $1000$ \\
Humanoid-Run & $67$ & $24$ & $1000$ \\
Humanoid-Walk & $67$ & $24$ & $1000$ \\
Humanoid-Stand & $67$ & $24$ & $1000$ \\
\midrule
\multicolumn{4}{c}{\textsc{Mujoco Gym}} \\
\midrule
HalfCheetah-v4 & $17$ & $6$ & $1000$ \\
Ant-v4 & $27$ & $8$ & $1000$ \\
Hopper-v4 & $11$ & $3$ & $1000$ \\
Walker2d-v4 & $17$ & $6$ & $1000$ \\
\bottomrule
 \end{tabular}
\end{table}  
\begin{table}[ht!]
\centering
\vspace{-0.2cm}
\caption{List of tasks from HumanoidBench on which the agents were ablated. The table also contains the dimensions of action, observation space and maximal episode length.}
\label{tab:continuous_envs2}
 \begin{tabular}{l |c | c | c} 
\toprule
 \textbf{Task} & \textbf{Observation dimension} & \textbf{Action dimension} & \textbf{Max episode length} \\
\midrule
h1-walk-v0 & $51$ & $19$ & $500$ \\
h1-run-v0 & $51$ & $19$ & $500$ \\
h1-stand-v0 & $51$ & $19$ & $500$ \\
h1-hurdle-v0 & $51$ & $19$ & $500$ \\
h1-stair-v0 & $51$ & $19$ & $500$ \\
h1-crawl-v0 & $51$ & $19$ & $500$ \\
h1hand-balance\_simple-v0 & $164$ & $61$ & $1000$ \\
h1hand-hurdle-v0 & $151$ & $61$ & $1000$ \\
h1hand-pole-v0 & $151$ & $61$ & $1000$ \\
h1hand-push-v0 & $163$ & $61$ & $1000$ \\
h1hand-stair-v0 & $151$ & $61$ & $1000$ \\
h1hand-slide-v0 & $151$ & $61$ & $1000$ \\
h1hand-walk-v0 & $151$ & $61$ & $1000$ \\
h1hand-run-v0 & $151$ & $61$ & $1000$ \\
\bottomrule
 \end{tabular}
\end{table}

\subsubsection{Pseudo Code of SAC-ERA}
\label{subsubsec: pseudo_code_sac}
To better illustrate the role of our method within the algorithmic framework, we present the pseudocode for a representative example, the Soft Actor-Critic (SAC) algorithm, adapted with \short\ in Algorithm \ref{alg:sac}.

\subsubsection{Hyperparameters}
\label{subsubsec: hyperparams_continuous}
We present the hyperparameters used in our experiments with SAC and PPO in Table~\ref{tab:combined_hyperparameters}

\begin{table}[ht]
    \caption{Comparison of hyperparameters for SAC and PPO.}
    \label{tab:combined_hyperparameters}
    \centering
    \begin{tabular}{lll}
        \toprule
        \textbf{Hyperparameter} & \textbf{SAC} & \textbf{PPO} \\
        \midrule
        
        \multicolumn{3}{l}{\textit{Optimizer Settings}} \\
        \quad Actor optimizer & \multicolumn{2}{c}{Adam} \\
        \quad Actor learning rate & \multicolumn{2}{c}{$3 \times 10^{-4}$} \\
        \quad Critic optimizer & AdamW & Adam \\
        \quad Critic learning rate & \multicolumn{2}{c}{$3 \times 10^{-4}$} \\
        \quad Temperature learning rate & $3 \times 10^{-4}$ & --- \\
        \quad Adam epsilon & --- & $1 \times 10^{-5}$ \\
        \quad Gradient clipping & --- & 0.5 \\
        \midrule

        \multicolumn{3}{l}{\textit{Network Architecture}} \\
        \quad Actor/Critic network & \multicolumn{2}{c}{3-layer MLP} \\
        \quad Hidden layer dimensions & (512, 512) & (64, 64) \\
        \quad Activation function & ReLU & Tanh \\
        \quad LayerNorm & True & False \\
        \midrule

        \multicolumn{3}{l}{\textit{Algorithm Hyperparameters}} \\
        \quad Discount factor ($\gamma$) & \multicolumn{2}{c}{0.99} \\
        \quad Replay buffer size & $1 \times 10^6$ & --- \\
        \quad Polyak averaging coefficient ($\tau$) & 0.005 & --- \\
        \quad Initial temperature ($\alpha$) & 1.0 & --- \\
        \quad Target entropy ($\mathcal{H}_0$) & $-\dim(\mathcal{A})/2$ & --- \\
        \quad Gradient steps per env. step & 2 & --- \\
        \quad Random exploration steps & 5,000 & --- \\
        \quad GAE parameter ($\lambda$) & --- & 0.95 \\
        \quad PPO clip ratio & --- & 0.2 \\
        \quad Entropy coefficient & --- & 0.01 \\
        \quad Batch size & 256 & 2048 \\
        \quad Mini-batch size & --- & 64 \\
        
        \bottomrule
    \end{tabular}
\end{table}

Our implementations of SAC and OBAC are heavily inspired by the official \texttt{jaxrl} repository~\citep{jaxrl}. For the network design, we follow the insights from \citet{nauman2024overestimationoverfittingplasticityactorcritic} and incorporate LayerNorm~\citep{ba2016layernormalization} into the neural networks.

Our OBAC implementation is built upon the codebase provided by \citet{kang2025forgetandgrowstrategydeepreinforcement}. It shares the same fundamental hyperparameters as our SAC implementation, with the behavior cloning weight set to $1 \times 10^{-3}$.

For the PPO and PPO-\short\ experiments, our implementation is based on the publicly available codebase of \citet{li2022drlcode}. We use target entropy of $-0.3\mathcal{A}$ for main experiments on PPO-\short.

For the TD-MPC2 baseline, we utilize the official implementation provided by the original authors. The results for comparison are also directly sourced from those reported in the official repository. We use target entropy of $-\mathcal{A}$ for main experiments on TD-MPC2-\short.

Similarly, our implementations of FastTD3 and FastSAC are based on the official codebases provided by their respective authors. We note that our construction of FastSAC-ERA differs from the method described in the original paper; these differences are detailed in Section~\ref{subsubsec:fastsac}.

\subsubsection{FastSAC-ERA}
\label{subsubsec:fastsac}
The FastTD3~\citep{seo2025fasttd3} framework demonstrated the potential of applying off-policy RL methods to massively parallel RL scenarios, achieving excellent performance on HumanoidBench.

Authors of FastTD3 also provided a FastSAC implementation, which replaced the mixed noise mechanism in FastTD3 with the standard entropy maximization objective from Soft Actor-Critic (SAC). However, they noted that this approach yielded unstable results, and hypothesized that maximizing action entropy in high-dimensional action spaces might be inherently challenging.

To address this issue, we investigated a solution based on minimal modification to the original FastTD3. Our approach, named FastSAC-\short, is derived from FastTD3 by retaining its noise mechanism while removing the Delayed Policy Updates and incorporating an entropy constraint via \short\ implementation. This method achieved performance superior to that of FastTD3.

In practice, our implementation was built directly upon the official FastTD3 codebase. The only modifications were the removal of Delayed Policy Updates and the addition of the \short\ implementation at the actor's output. All other hyperparameters and implementation details were kept identical to the original FastTD3 configuration.

\begin{algorithm}[t]
\caption{Soft Actor-Critic (SAC) with \short}
\label{alg:sac}
\begin{algorithmic}[1]
   \STATE \textbf{Initialize:} actor parameters $\theta$, critic parameters $\phi_1, \phi_2$.
   \STATE \textbf{Initialize:} target network parameters $\phi'_1 \leftarrow \phi_1$, $\phi'_2 \leftarrow \phi_2$.
   \STATE \textbf{Initialize:} replay buffer $\mathcal{D}$.
   \STATE \textbf{Hyperparameters:} learning rates $\lambda_\pi, \lambda_Q$, target entropy $\mathcal{H}_0$, Polyak coefficient $\tau$.
   \FOR{each training step}
       \STATE Sample action from the policy: $a_t \sim \pi_\theta(\cdot | s_t)$.
       \STATE Execute action $a_t$, observe reward $r_t$ and next state $s_{t+1}$.
       \STATE Store transition $(s_t, a_t, r_t, s_{t+1})$ in replay buffer $\mathcal{D}$.
       \STATE Sample a random minibatch of transitions $B = \{(s, a, r, s')\}$ from $\mathcal{D}$.
       
       \STATE \textit{// Update the Q-functions (critics)}
       \STATE Sample next actions: $a' \sim \pi_\theta(\cdot | s')$.
       \STATE Compute the target Q-value by taking the minimum of the two target critics:
       $$ Q'_{\text{target}}(s', a') \leftarrow \min_{i=1,2} Q_{\phi'_i}(s', a') $$
       \STATE Compute the soft Q-target $y$ (matches Eq.~\ref{eq:simplified_sac_q_target}):
       $$ y \leftarrow r + \gamma Q'_{\text{target}}(s', a') $$
       \STATE Update both critics by one step of gradient descent using the loss from Eq.~\ref{eq:sac_q_loss}:
       $$ \nabla_{\phi_i} \frac{1}{|B|} \sum_{(s,a,y) \in B} \frac{1}{2} \left( Q_{\phi_i}(s, a) - y \right)^2 \quad \text{for } i=1,2 $$
       
       \STATE \textit{// Update the policy (actor)}
       \STATE Sample new actions for the policy update (using reparameterization trick): $\tilde{a} \sim \pi_\theta(\cdot|s)$.
       \STATE Compute Q-values for the new actions using the minimum of the two critics:
       $$ Q_{\text{min}}(s, \tilde{a}) \leftarrow \min_{i=1,2} Q_{\phi_i}(s, \tilde{a}) $$
       \STATE Update the policy by one step of gradient ascent to maximize the objective from Eq.~\ref{eq:simplified_sac_policy_loss}:
       $$ \nabla_{\theta} \frac{1}{|B|} \sum_{s \in B} Q_{\text{min}}(s, \tilde{a}) $$


       \STATE \textit{// Update target networks using Polyak averaging}
       \STATE $\phi'_i \leftarrow \tau \phi_i + (1-\tau)\phi'_i \quad \text{for } i=1,2$
   \ENDFOR
\end{algorithmic}
\end{algorithm}
\subsection{Implementation Details of Image Classification}
\label{subsec:impl_vision}

\subsubsection{Code Implementation of \short\ in Image Classification}
\begin{figure}[h!]
    \centering
\begin{lstlisting}[
    style=python, 
    caption={\short\ Implementation in Image Classification}, 
    label={code:image}
]
class ERA(nn.Module):
    def __init__(self, C_H: float, n_dims: int):
        super().__init__()
        self._tau = 4.
        self.C_H = C_H
        self.n_dims = n_dims

        self.upper_bound = math.log(self._tau) / self._tau
        assert C_H >= self.upper_bound
        self.slope = (self.upper_bound - C_H / n_dims) / (1 - 1 / n_dims)
        self.b = (C_H - self.slope) / n_dims
    
    def forward(self, x: torch.Tensor) -> torch.Tensor:
        """
        x: logits before softmax, shape (..., n_dims)
        return: adjusted logits before softmax, shape (..., n_dims)
        """
        h = self.slope * x.softmax(dim=-1) + self.b
        u = -1 - torch.log(h)
        new_logits = (-1 - torch.sqrt(2 * u) - 3/4 * u).to(x.dtype)

        max_values = torch.max(x, dim=-1, keepdim=True).values.detach()
        x = x - max_values
        min_values = torch.min(new_logits, dim=-1, keepdim=True).values.detach()
        new_logits = new_logits - min_values

        return new_logits
\end{lstlisting}
\end{figure}
We provide the implementation of ERA for image classification tasks in Listing~\ref{code:image}. In the code, C\_H corresponds to $C_{\mathcal{H}_0}$ defined in Eq.~\ref{eq:discrete_era}, and n\_dims denotes the number of classes. We set $\tau=4$ in our implementation without performing any tuning for this parameter.

\subsubsection{Training Setup}
Our training for ImageNet was completed on 4 A100 GPUs, and we report the 95\% confidence interval calculated from the dataset. For CIFAR-10, which requires less computation, we trained three separate runs on 3 machines, each with 4 A40 GPUs, and report the confidence interval computed from these three results to ensure maximum reproducibility. 

\subsubsection{Commands Used for Experiments}
We provide two main commands used for training in image classification.
The two commands delineate the training procedures for our models under two distinct settings: one incorporating data augmentation and the other without it.
The training commands were sourced directly from the reference ImageNet training script within the timm library. We employed this identical set of commands for training on both the ImageNet and CIFAR-10 datasets without any dataset-specific hyperparameter tuning to ensure a consistent experimental setup.

\begin{lstlisting}[
    style=bash,
    caption={Command to launch training with data augmentation.},
    label={lst:train_with_aug}
]
./distributed_train.sh 4 --data-dir ../data --dataset torch/cifar10 --dataset-download -b 64 --model resnet50 --sched cosine --epochs 200 --lr 0.05 --amp --remode pixel --reprob 0.6 --aug-splits 3 --aa rand-m9-mstd0.5-inc1 --resplit --split-bn --jsd --dist-bn reduce
\end{lstlisting}

\begin{lstlisting}[
    style=bash,
    caption={Command to launch training without data augmentation (baseline).},
    label={lst:train_without_aug}
]
./distributed_train.sh 4 --data-dir ../data --dataset torch/cifar10 --dataset-download -b 64 --model resnet50 --sched cosine --epochs 200 --lr 0.05 --amp --dist-bn reduce
\end{lstlisting}
\subsection{Implementation Details of LLM Training}
\label{subsec:impl_llm}

\subsubsection{Code Implementation of ERA in LLM}

We provide the core implementation of ERA in LLM in Listing \ref{code:llm}. In the code, \texttt{era\_lb}, \texttt{era\_ub} and \texttt{era\_k} corresponds to $\omega_{\text{low}}, \omega_{\text{high}},k$ defined in Eq.~\ref{eq:llm}, respectively. In the first training stage, we further apply a \texttt{top-$k$} filter (retaining the 20 largest logits) within the \texttt{logprobs\_from\_logits} function to enhance training stability. Additionally, in practice, we found that the advantage scaling does not affect model performance, so we did not implement it in our code.

\begin{figure}[h!]
    \centering
\begin{lstlisting}[
    style=python, 
    caption={\short\ Implementation in LLM}, 
    label={code:llm}
]
length = response_mask.sum(dim=-1)
k_per_sample = (0.2 * length).long().clamp(min=1)

mean_top_entropy = []
masked_entropy = entropy.masked_fill(~response_mask.bool(), float("-inf"))
for b in range(entropy.size(0)):
    k = k_per_sample[b].item()
    top_entropy_b, _ = torch.topk(masked_entropy[b], k)
    mean_top_entropy.append(top_entropy_b.mean())

mean_top_entropy = torch.stack(mean_top_entropy).unsqueeze(-1)
cond_A = (mean_top_entropy < era_lb) & (advantages > 0)
cond_B = (mean_top_entropy > era_ub) & (advantages > 0)

logits[cond_A] = logits[cond_A] * era_k
logtis[cond_B] = logits[cond_B] / era_k

log_prob = logprobs_from_logits(logits)

\end{lstlisting}
\end{figure}

\subsubsection{Hyperparameters}

For GRPO, GRPO w/ 20\% Forking Tokens, ERA, we use a training batch size of $256$ and a mini batch size of $256$ in the verl configuration, which results in a on-policy setting. For KL-Cov and Clip-Cov, we use a training batch size of $256$ and a mini batch size of $32$. The learning rate is $10^{-6}$ and no learning rate warm-up or scheduling is applied. We also utilize dynamic sampling to enhance training efficiency. Since our setting is on-policy, the clip ratio is irrelevant. The maximum response length is $8192$ with no overlong reward shaping. For Qwen2.5-Math-1.5B, we use MATH problems of levels 3–5 as the training set in this experiment since DAPO-Math-17K is too difficult.

The hyperparameters of ERA are fixed to $\omega_{\text{low}}=0.45$, $\omega_{\text{high}}=3.0$, and $k=2$ across all settings, without any tuning. These values are chosen with reference to the initial entropy of the model, $H_{\text{resp}} \approx 1.5$, such that $\omega_{\text{low}}$ and $\omega_{\text{high}}$ lie below and above this value, respectively. The only exception is in the second training stage of ERA for the Qwen2.5-Math-7B model, where we set $\omega_{\text{low}}=0.2$, $\omega_{\text{high}}=+\infty$, and $k=2$.
\section{Proofs And Derivations}
\label{sec:proofs}
\subsection{Proof of entropy bound in continuous space}
\label{subsec:proofs_continuous}
Recall the continuous form of \short:
\begin{equation}
    \mu' = \mu, \quad \sigma' = \exp{\left[\max\left(\log \sigma_{\max} + (\mathcal{H}_0' -D\log \sqrt{2\pi e} - D \log \sigma_{\max})\frac{e^{\hat{\sigma}_i}}{\sum_{j=1}^{D}e^{\hat{\sigma}_j}} , \log\sigma_{\min}\right)\right]} \notag
\end{equation}
Here, $\mathcal{H}_0'$ is the target entropy for the final policy~$\mathcal{H}_0$ plus a residual entropy term $\delta \geq 0$ to account for the bounding bias.
The residual entropy term $\delta$ is defined as the minimum additional entropy required to ensure that the final entropy, after transformation of the Gaussian distribution, is at least $\mathcal{H}_0$.
And forms for $\delta$ in both squashed Gaussian and truncated Gaussian cases are given by Eq.~\ref{eq:squashed_gaussian_residual} and Eq.~\ref{eq:truncated_gaussian_residual}, respectively.
\begin{align}
    \delta_{\tanh} &= -\mathbb{E}_{s \sim \rho_\pi, u \sim \mathcal{N}(\mu_\theta(s), \Sigma_\theta(s))}\left[\sum_{i=1}^D \log(1 - \tanh(u_i)^2)\right]
    \label{eq:squashed_gaussian_residual}
    \\
    \delta_{\text{TN}} &= -\mathbb{E}_{s \sim \rho_\pi} \left[ \sum_{i=1}^{D} \left(\log Z_i(s) - \frac{\beta_i(s) \phi(\beta_i(s)) - \alpha_i(s) \phi(\alpha_i(s))}{2Z_i(s)}\right) \right]
    \label{eq:truncated_gaussian_residual}
\end{align}
In practice, we can learn a $\hat{\delta} \geq 0$ as a state-independent parameter by minimizing loss in Eq.~\ref{eq:residual_loss}, which provides a straightforward learning mechanism that is agnostic to the specific form of the distribution.
\begin{equation}
    \mathcal{L}(\hat{\delta}) = \mathbb{E}_{s \sim \mathcal{D}} \left[\hat{\delta}(\mathcal{H}[\pi(\cdot|s)] - \mathcal{H}_0)\right] \notag
\end{equation}
We directly adopt the loss function from the automated entropy adjustment in SAC~\citep{haarnoja2018soft}. We note that the optimization objective in our method is identical to that in SAC, and thus it can be formulated as a nearly identical dual problem, formed in Eq.~\ref{eq:dual_problem}. The convergence of $\hat{\delta}$ is therefore guaranteed by strong duality.
\begin{align}
    \max_{\pi} \mathbb{E}_{(s,a)\sim\rho_{\pi}} [r(s, a)] = \min_{\hat{\delta} \ge 0} \max_{\pi} \left( \mathbb{E}_{(s,a)\sim\rho_{\pi}} [r(s, a) - \hat{\delta} \log \pi(a|s)] - \hat{\delta} \mathcal{H}_0 \right)
    \label{eq:dual_problem}
\end{align}
Now we can prove that the entropy of the final policy is guaranteed to be at least $\mathcal{H}_0$.
\begin{proposition}
    \label{prop:continuous_era}
    Given a target entropy $\mathcal{H}_0$ and a residual entropy $\hat{\delta} \geq \delta$, the policy defined by Eq.~\ref{eq:continuous_era} has entropy $\mathcal{H}(\pi) \geq \mathcal{H}_0$,
    and $\sigma'$ is bounded within $[\sigma_{\min}, \sigma_{\max}]$.
\end{proposition}
\begin{proof}
    We only need to show that the entropy of the Gaussian before transformation is at least $\mathcal{H}_0 + \delta$ to prove the entropy bound.
    The entropy of a Gaussian distribution $\mathcal{N}(\mu, \Sigma)$ is given by $\mathcal{H}(\mathcal{N}(\mu, \Sigma)) = \frac{D}{2} \log(2\pi e) + \frac{1}{2} \log |\Sigma|$.
    Since $\Sigma$ is diagonal, we have $\log |\Sigma| = \sum_{i=1}^{D} \log \sigma_i^2 = 2 \sum_{i=1}^{D} \log \sigma_i$.
    Therefore, the entropy can be expressed as:
    \begin{equation}
        \mathcal{H}(\mathcal{N}(\mu, \Sigma)) = \frac{D}{2} \log(2\pi e) + \sum_{i=1}^{D} \log \sigma_i
        \label{eq:gaussian_entropy}
    \end{equation}
    Hence, $\mathcal{H}(\mathcal{N}(\mu, \Sigma)) \geq \mathcal{H}_0 + \delta$ is equivalent to:
    \begin{equation}
        \sum_{i=1}^{D} \log \sigma_i \geq \mathcal{H}_0 + \delta - \frac{D}{2} \log(2\pi e)
        \label{eq:entropy_condition}
    \end{equation}
    For \short-transformed Gaussian distribution, we have:
    \begin{align}
        \sum_{i=1}^{D} \log \sigma_i &\geq \sum_{i=1}^{D}[\log \sigma_{\max} + (\mathcal{H}_0' -D\log \sqrt{2\pi e} - D \log \sigma_{\max})\frac{e^{\hat{\sigma}_i}}{\sum_{j=1}^{D}e^{\hat{\sigma}_j}}] \notag \\
        &=D\log \sigma_{\max} + (\mathcal{H}_0' -D\log \sqrt{2\pi e} - D \log \sigma_{\max}) \notag \\
        &=\mathcal{H}_0' - \frac{D}{2} \log(2\pi e)
    \end{align}
    Since $\mathcal{H}_0' = \mathcal{H}_0 + \hat{\delta} \geq \mathcal{H}_0 + \delta$, the condition in Eq.~\ref{eq:entropy_condition} is satisfied.
    Additionally, the use of the $\max$ function in Eq.~\ref{eq:continuous_era} ensures that $\sigma'$ is bounded within $[\sigma_{\min}, \sigma_{\max}]$.
    This completes the proof.
\end{proof}
\subsection{Proof of entropy bound in discrete space}
\label{subsec:proofs_discrete}
Recall the discrete form of \short:
\begin{equation}
    z' = h^{-1}\left[\max \left(\frac{\log \tau}{\tau} + \left(C_{\mathcal{H}_0} - n \frac{\log \tau}{\tau}\right)\frac{1}{D-1}\left(1 - \frac{e^{z_i}}{\sum_{j=1}^{D}e^{z_j}}\right), 0\right)\right] \notag
\end{equation}

Before we delve into the proof of its entropy bound, we first provide some insights into the design of \short\ in the context of vision tasks. To adapt the entropy constraint function from continuous spaces for discrete domains, our initial idea was to have the network output the entropy of individual components rather than their logits. However, this direct approach is problematic because the function $H(p) = -p\ln p$ is non-monotonic over the interval $[0,1]$. This ambiguity means a given entropy value cannot be uniquely mapped back to its corresponding probability; for instance, an entropy of 0 could correspond to a probability of either 0 or 1.

To resolve this ambiguity, we introduce a scaling factor $\tau > e$ and consider a "$\tau$-divided distribution," where each probability is scaled down by $\tau$. By selecting $\tau > e$, we ensure that the function $-p\ln p$ is strictly monotonically increasing on the interval $[0, 1/\tau]$. This establishes a one-to-one mapping, allowing for the unique recovery of a probability value from its entropy within this restricted range. Therefore, our network is designed to output the entropy of this $\tau$-divided distribution. We then map these entropy values back to the probability space using the inverse function, $h^{-1}$. As $h^{-1}$ lacks a closed-form solution, we utilize a numerical approximation. A final normalization step is required because the resulting probabilities from this inverse mapping do not inherently sum to one.

Crucially, we have proven that the entropy loss during this normalization process is bounded. By leveraging the continuous-space entropy constraint function to ensure the initial output entropy is above a threshold $C_{\mathcal{H}_0}$, we can guarantee that the entropy of the final discrete distribution will also exceed $C_{\mathcal{H}_0}$. This constitutes the core mechanism behind the implementation of \short\ in discrete spaces.

\begin{proposition}
    \label{prop:discrete_era}
    Given a target entropy $\mathcal{H}_0$ and a hyperparameter $\tau \ge e$, the policy defined by Eq.~\ref{eq:discrete_era} has entropy $\mathcal{H}(\pi) \geq \mathcal{H}_0$.
\end{proposition}
\begin{proof}
    We denote $\kappa = \max (\frac{\log \tau}{\tau} + (C_{\mathcal{H}_0} - n \frac{\log \tau}{\tau})\frac{1}{D-1}(1 - \frac{e^{z_i}}{\sum_{j=1}^{D}e^{z_j}}), 0)$.
    Similar to the continuous case, we have $\kappa$ bounded within $[0, \frac{\log \tau}{\tau}]$ and $\sum_{i=1}^{D} \kappa_i \geq C_{\mathcal{H}_0}$.
    We denote the probability of the final softmax policy as $p = \text{softmax}(z') = \frac{e^{z'}}{\sum_{j=1}^{D} e^{z'_j}}$.
    Then we have:
    \begin{align}
        \mathcal{H}(\pi) &= -\sum_{i=1}^{D} p_i \log p_i \notag \\
        &= - \frac{\sum_{i=1}^{D}e^{h^{-1}(\kappa_i)}h^{-1}(\kappa_i)}{\sum_{j=1}^{D} e^{h^{-1}(\kappa_j)}}  + \log(\sum_{j=1}^{D} e^{h^{-1}(\kappa_j)}) \notag \\
        &\geq 1 + \log (-\sum_{i=1}^{D}e^{h^{-1}(\kappa_i)}h^{-1}(\kappa_i))
    \end{align}
    Recall that $h = (-x\ln x) \circ e$, so $h^{-1} = \ln \circ (-x\ln x)^{-1}$.
    Hence we have:
    \begin{align}
        \mathcal{H}(\pi) &\geq 1 + \log (-\sum_{i=1}^{D}e^{h^{-1}(\kappa_i)}h^{-1}(\kappa_i)) \notag \\
        &=1 + \log(\sum_{i=1}^{D} \kappa_i)
        \geq 1 + \log(C_{\mathcal{H}_0})
        = \mathcal{H}_0
    \end{align}
\end{proof}
\subsection{Proof of entropy bound in llms}
\label{subsec:proofs_llm}

Recall the definition of the ERA instantiation for LLMs:

\begin{equation}
    z' = \begin{cases}
        kz & H_{\text{resp}} < \omega_{\text{low}}, A_{t}>0, \\
        z & \omega_{\text{low}} \leq H_{\text{resp}} \leq \omega_{\text{high}}, A_{t}<0 \text{ or } A_{t}>0, \\
        \tfrac{1}{k}z & H_{\text{resp}} > \omega_{\text{high}}, A_{t}>0,
    \end{cases}\notag
\end{equation}

and 

\begin{equation}
    A_t' = \begin{cases}
        \frac 1k A_t & H_{\text{resp}} < \omega_{\text{low}}, A_{t}>0, \\
        A_t & \omega_{\text{low}} \leq H_{\text{resp}} \leq \omega_{\text{high}}, A_{t}<0 \text{ or } A_{t}>0, \\
        kA_t & H_{\text{resp}} > \omega_{\text{high}}, A_{t}>0,
    \end{cases}\notag
\end{equation}
where $z$ are the logits, $A_t$ the advantages, and $H_{\text{resp}}$ is the average entropy of the top $20\%$ of tokens with the highest entropy in the response.

These transformations are applied after sampling. The modified policy-gradient objective is therefore

$$J(\theta)=\mathbb{E}_t[\mathbb{E}_{a_t\sim \pi_\theta(\cdot |s_t)}\log \pi'_\theta(a_t|s_t)A'_t]$$

Intuitively, when the entropy is too low, ERA sharpens the policy; when it is too high, ERA flattens it. By rescaling the advantages of modified tokens, we show below that ERA is equivalent to augmenting the vanilla policy-gradient objective with an adaptive KL regularizer. This KL term guarantees that the entropy of responses remains in the interval $[\omega_{\text{low}}, \omega_{\text{high}}]$, preventing entropy collapse. Under mild assumptions, we derive a positive entropy lower bound. 

Fixing the state $s_t$, denote $\pi_a = \pi_\theta(a|s_t)$, $\pi'_a = \pi'_\theta(a|s_t)$, and $A_a$ the advantage of action $a$. The entropy is $H = -\sum_a \pi_a \log \pi_a$. We first derive the gradient of the entropy.

\begin{lemma}

\begin{align}
  \frac{\partial H}{\partial z_a}
  &= \sum_{a'} -\frac{\partial \log \pi_{a'}}{\partial z_a}(\pi_{a'}\log \pi_{a'} + \pi_{a'})\notag\\
  &= \sum_{a'} -([a = a'] - \pi_{a})(\pi_{a'}\log \pi_{a'}+ \pi_{a'})\notag \\
  &=-\pi_a(\log\pi_a + H). \label{eq:dH_dz}
\end{align}
    
\end{lemma}

We also define the $\pi$-weighted covariance that will be used later:

\begin{definition}
    Define the \(\pi\)-weighted covariance for two vectors \(x=(x_a)\), \(y=(y_a)\) by
      \[
        \mathrm{Cov}_\pi(x,y) = \sum_a \pi_a x_a y_a - \Big(\sum_a \pi_a x_a\Big)\Big(\sum_a \pi_a y_a\Big).
      \]
\end{definition}

Now we show our main result:

\begin{proposition}
    \label{prop:llm_era}
    Let $\pi_\theta$ be the base policy and $\pi'_\theta$ the ERA-adjusted policy from Eq.~\eqref{eq:llm}. Suppose that:
    
    \begin{enumerate}[label=(\roman*)]
      \item (Logit approximation) The change in entropy can be approximated by treating logits \(z\) as the effective policy parameters and using first-order (infinitesimal) sensitivity of entropy w.r.t.\ \(z\).
      
      \item (Positive advantage mass) The aggregated positive advantage restricted to the tokens considered in \(H_{\text{resp}}\),
      \[
        C(s_t) = \sum_{a,A_a>0} \pi_a A_a,
      \]
      satisfies \(C(s_t) \ge \gamma\) for some \(\gamma>0\).

      \item (Bounded response entropy) In some intermediate point of the training process, $H_{\text{resp}}$ has a lower bound $H_{\text{min}}$ and an upper bound $\omega_{\text{high}}$.
      
      \item (Bounded PG-induced entropy decrease) 
      We assume the vanilla policy-gradient term's expected effect on entropy is bounded as
      \[
        \mathbb{E}[\mathrm{Cov}_\pi(\pi_a A_a, \log\pi_a)] \le \alpha H,
      \]
      for some \(\alpha \ge 0\), where \(H\) denotes the entropy of the current policy $\pi$.
      
      \item (Bounded KL-induced entropy decrease) 
      We assume there exists a constant $B_k>0$ (that depends on $k$ and $H_{\text{min}}$) such that
    \[
        \mathrm{Cov}_\pi(\pi'_a - \pi_a, \log \pi_a) \ge B_k H,
    \]
    \end{enumerate}
    
    If \(\gamma B_k - \alpha > \beta\) for $\beta>0$, then there exists a constant \(\mathcal{H}_0>0\) such that the response entropy satisfies
    \[
      \mathbb{E}[H_{\text{resp}}] \ge \mathcal{H}_0
    \]
    under ERA updates.
\end{proposition}

\begin{proof}

When $H_{\text{resp}} < \omega_{\text{low}}$, ERA sharpens positively advantaged actions. Following the derivation, the ERA-adjusted gradient satisfies

\begin{align}
    &\phantom{{}={}}\frac{\partial}{\partial z_a}\mathbb{E}_{a'\sim \pi}\log \pi'_{a'}A'_{a'} \notag \\&=\frac{\partial}{\partial z_a}\mathbb{E}_{a'\sim \pi} \left([A_{a'}>0]\log \pi'_{a'}\frac 1k A_{a'} + [A_{a'}<0] \log \pi_{a'}A_{a'}\right) \notag \\
    &= \mathbb{E}_{a'\sim \pi} \left([A_{a'}>0] \frac{\partial \log \pi'_{a'}}{\partial z'_a} \frac{\partial z'_a}{\partial z_a} \frac 1k A_{a'} + [A_{a'}<0] \frac{\partial \log \pi_{a'}}{\partial z_a}A_{a'}\right) \notag \\
    &=\mathbb{E}_{a'\sim \pi} \left([A_{a'}>0]([a'=a] - \pi'_{a'})A_{a'} + [A_{a'}<0]([a'=a] - \pi_{a'})A_{a'}\right) \notag \\
    &= \pi_aA_a - \pi'_a \sum_{a',A_{a'}>0} \pi_{a'}A_{a'} - \pi_a \sum_{a',A_{a'}<0}\pi_{a'}A_{a'},
\end{align}

Since the expectation of advantage is zero, and we have defined $C(s_t) = \sum_{a',A_{a'}>0} \pi_{a'}A_{a'}$, yielding

\begin{equation}
    \frac{\partial}{\partial z_a} \mathbb{E}_{a'\sim \pi} \log \pi'_{a'} A'_{a'} = \pi_aA_a - C(s_t)(\pi'_a - \pi_a).\label{eq:era_grad}
\end{equation}

For the vanilla policy-gradient loss, this reduces to

\begin{equation}
    \frac{\partial}{\partial z_a} \mathbb{E}_{a'\sim \pi} \log \pi_{a'} A_{a'} = \pi_a A_a\label{eq:pg_grad}
\end{equation}

Meanwhile, by a similar derivation, the gradient of the KL divergence is

\begin{equation}
    \frac{\partial}{\partial z_a} \mathrm{KL}[\pi'\|\pi] = -\frac{\partial}{\partial z_a} \mathbb{E}_{a'\sim \pi'} \log \pi_{a'} = \pi_a - \pi'_a.\label{eq:kl_grad}
\end{equation}

Thus, by combining~\eqref{eq:era_grad}, \eqref{eq:pg_grad} and \eqref{eq:kl_grad}, the ERA-adjusted objective can be written as

\begin{equation}
    J'(\theta) = \mathbb{E}_t[\mathbb{E}_{a_t\sim \pi_\theta(\cdot |s_t)}\underbrace{\log \pi_\theta(a_t|s_t)A_t}_{J_{\text{PG}}}  + \text{sg}(C(s_t))\underbrace{\mathrm{KL}[\pi_\theta'(\cdot |s_t), \pi_\theta(\cdot |s_t)]}_{J_{\text{KL}}}],
\end{equation}

where the $\text{sg}(\cdot)$ denotes the stop gradient operator. For the other case $\omega_{\text{low}}\leq$ (we have assumed that $H_{\text{resp}}\leq \omega_{\text{high}}$, the same structure holds; only the definition of $\pi'_\theta$ changes. Hence, ERA is equivalent to a policy gradient objective augmented with an adaptive KL regularizer that sharpens or flattens the distribution depending on $H_{\text{resp}}$ and also the value of $C(s_t)$.

We will evaluate the instantaneous directional derivative of entropy along these gradient directions (this corresponds to the first-order change in entropy under an infinitesimal step in the indicated direction).

Using \eqref{eq:dH_dz}, the first-order change of entropy caused by \(J_{\text{PG}}\) is
\begin{align}
  \Delta H_{\text{PG}}
  &=\sum_a \frac{\partial H}{\partial z_a} \cdot \pi_a A_a \notag \\
  &= \sum_a -\pi_a(\log\pi_a + H)\cdot \pi_a A_a \notag \\
  &= -\sum_a \pi_a^2 A_a (\log\pi_a+H) \notag\\
  &= -\mathrm{Cov}_\pi (\pi_a A_a, \log \pi_a).\label{eq:deltaH_PG}
\end{align}

By assumption (iv) this term is bounded below by \(-\alpha H\):
\[
  \mathbb{E}[\Delta H_{\text{PG}}] \ge -\alpha H.
\]
Thus the vanilla policy-gradient component can decrease entropy, but by no more than \(\alpha H\) in magnitude.

Similarly, the KL-term directional derivative is
\begin{align}
  \Delta H_{\text{KL}}
  &= \sum_a \frac{\partial H}{\partial z_a} \cdot (\pi_a - \pi_a') \notag \\
  &= \sum_a -\pi_a(\log\pi_a + H) \cdot(\pi_a - \pi'_a) \notag \\
  &= \sum_a \pi_a(\pi_a' - \pi_a)(\log\pi_a + H) \notag\\
  &= \mathrm{Cov}_\pi (\pi_a' - \pi_a, \log \pi_a)\label{eq:deltaH_KL_pre}
\end{align}

By assumption (v) we have \(\mathrm{Cov}_\pi(\pi'_a - \pi_a,\log\pi_a)\ge B_k H\). Using assumption (ii) \(C(s_t)\ge\gamma\) therefore yields
\[
  C(s_t)\Delta H_{\text{KL}} \ge \gamma B_k H.
\]

Combining the two contributions,
\[
  \mathbb{E}[\Delta H] = \mathbb{E}[\Delta H_{\text{PG}} + C(s_t)\Delta H_{\text{KL}}]
  \ge -\alpha H + \gamma B_k H
  = (\gamma B_k - \alpha) H.
\]

By the hypothesis \(\gamma B_k - \alpha> \beta\) we have \(\Delta H > \beta H\) whenever \(H>0\) and \(H\) is in the sharpening regime. Thus, if \(H_{\text{resp}}\) drops below \(\omega_{\text{low}}\), the ERA-induced update produces a positive first-order increase in entropy proportional to \(H_{\text{resp}}\). Consequently the dynamics push \(H_{\text{resp}}\) upward until it leaves the sharpening regime (i.e., until \(H_{\text{resp}}\ge\omega_{\text{low}}\) or the KL-term no longer sharpens).

Formally, the expected change of total entropy is at least
\begin{equation}
    \beta \mathbb{E}_{H_{\text{resp}}<\omega_{\text{low}}}[H_{\text{resp}}] - \alpha \mathbb{E}_{H_{\text{resp}}\geq \omega_{\text{low}}}[H_{\text{resp}}]
\end{equation}

Applying Markov’s inequality gives
\(\Pr(H_{\text{resp}}\ge\omega_{\text{low}})\leq \mu/\omega_{\text{low}}\),
where \(\mu=\mathbb{E}[H_{\text{resp}}]\).
Further, by assumption (iii): $H_{\text{min}}\leq H_{\text{resp}}\leq \omega_{\text{high}}$, we obtain the sufficient condition to make the expected entropy change positive:
\[
\beta > \alpha \cdot
\frac{\mu \omega_{\text{high}}}{(\omega_{\text{low}}-\mu)H_{\min}}.
\]
Then there exists a constant $\mu$, such that the expected change of total entropy is positive. Therefore by taking $\mathcal{H}_0=\mu$, \(H_{\text{resp}}\) is prevented from collapsing to zero and satisfies
\[
  H_{\text{resp}} \ge \mathcal{H}_0.
\]

\end{proof}

We now justify the assumptions made in Proposition~\ref{prop:llm_era}.

\begin{enumerate}[label=(\roman*)]
    \item The first assumption, namely approximating entropy differences by treating logits as policy parameters, is standard and also adopted by~\citep{cui2025entropy}.
    
    \item Recall that $C(s_t)=\sum_{a,A_{a}>0} \pi_{a}A_{a}$ measures the aggregated positive advantage, which reflects the ``importance’’ of a token. Intuitively, $C(s_t)$ indicates whether a token should remain explorative and thus be subject to entropy regularization. We assume that for important tokens, $C(s_t)$ is uniformly bounded below by some constant $\gamma>0$.
    
    \item Empirically, our training curves show that responses with $H_{\text{resp}}>\omega_{\text{high}}$ vanish rapidly, and such cases contribute negligibly to the average entropy. This supports the assumption $H_{\text{resp}}\leq \omega_{\text{high}}$. Moreover, in the early stage of training, the highest entropy tokens (top $20\%$) contain a lot of exploratory tokens, exhibiting a large average entropy, motivating the assumption of a positive lower bound $H_{\text{resp}}\geq H_{\min}$.
    
    \item It is provable that
    \[
    \mathrm{Cov}_\pi (\pi_aA_a, \log \pi_a)\leq H,
    \]
    where $H$ denotes the entropy. In practice this upper bound is rarely tight, and we assume instead a looser bound with a small constant $\alpha \in (0,1)$.

    \item In our regime, the entropy is low enough that the token with the largest probability dominates (with probability $\geq 0.6$). In this setting, the covariance is large enough and is proportional to the entropy $H$.
\end{enumerate}

In practice, the observed entropy lower bound is higher than the theoretical bound derived in Proposition~\ref{prop:llm_era}, owing both to the looseness of the Markov inequality used in the derivation and to the fact that the tokens outside $H_{\text{resp}}$ (bottom $80\%$) also get an entropy boost.
\section{Additional Results}
\label{sec:additional_results}
\subsection{Additional Results on Continuous Control Tasks}
\label{subsec:additional_continuous}
In this subsection, we provide additional experimental results on continuous control tasks to further validate the effectiveness of our proposed method, \short, and to find more insights regarding entropy regularization in reinforcement learning.

\subsubsection{Truncated Gaussian is more stable than Tanh Gaussian}
\begin{figure*}[h]
    \centering
    \begin{subfigure}[b]{0.24\textwidth}
        \centering
        \includegraphics[width=\textwidth]{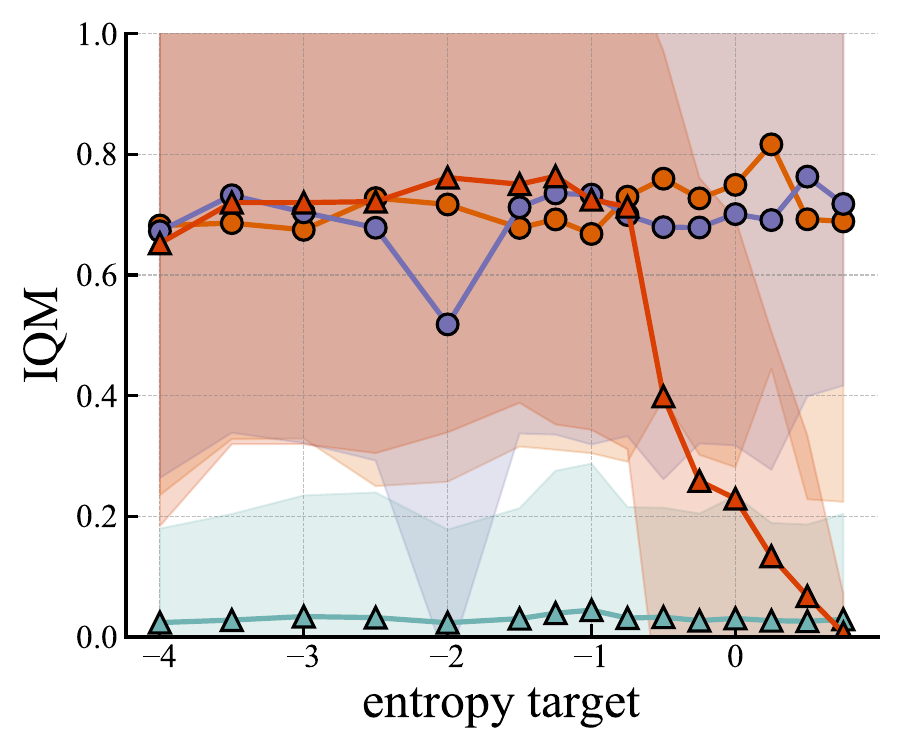}
        \vskip -0.1in
        \subcaption{}
        \label{fig:dist_truncated}
    \end{subfigure}
    \begin{subfigure}[b]{0.48\textwidth}
        \centering
        \includegraphics[width=\textwidth]{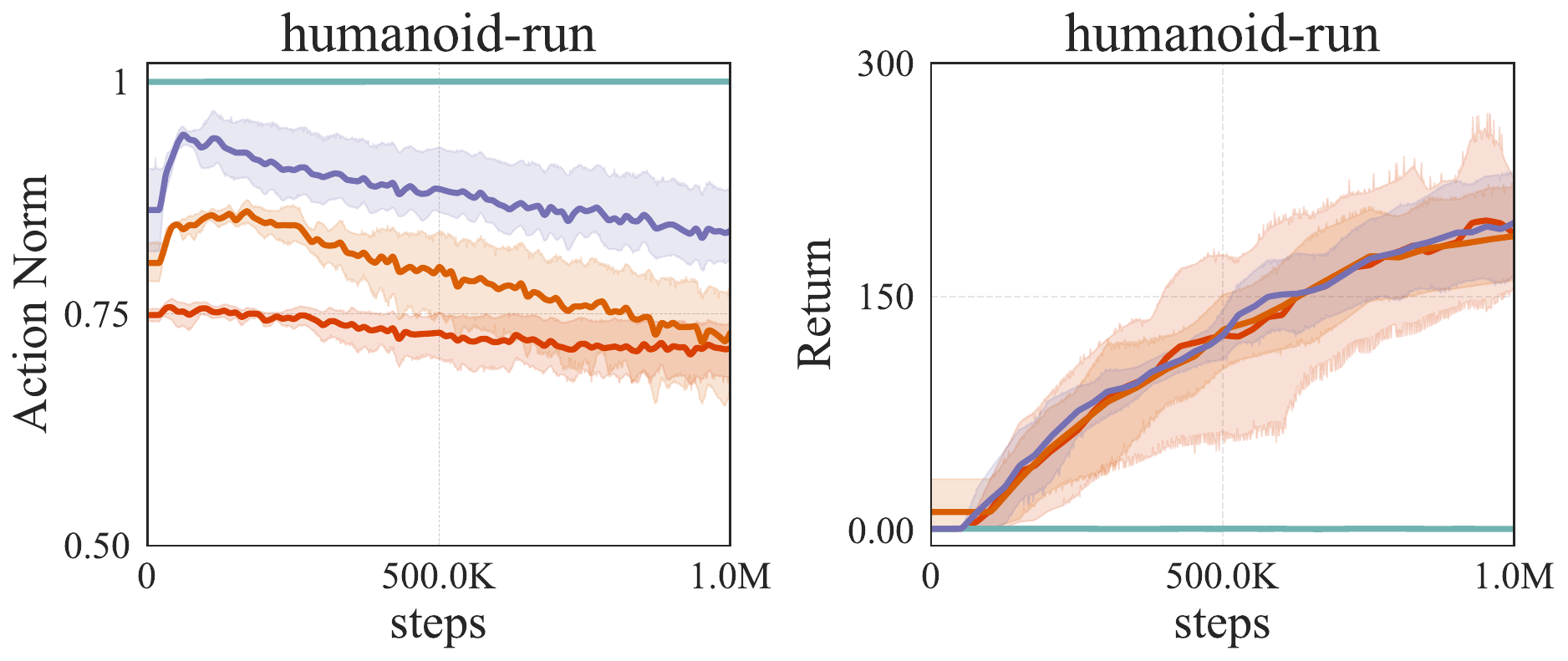}
        \vskip -0.1in
        \subcaption{}
        \label{fig:dist_tanh}
    \end{subfigure}
    \begin{subfigure}[b]{0.24\textwidth}
        \centering
        \includegraphics[width=\textwidth]{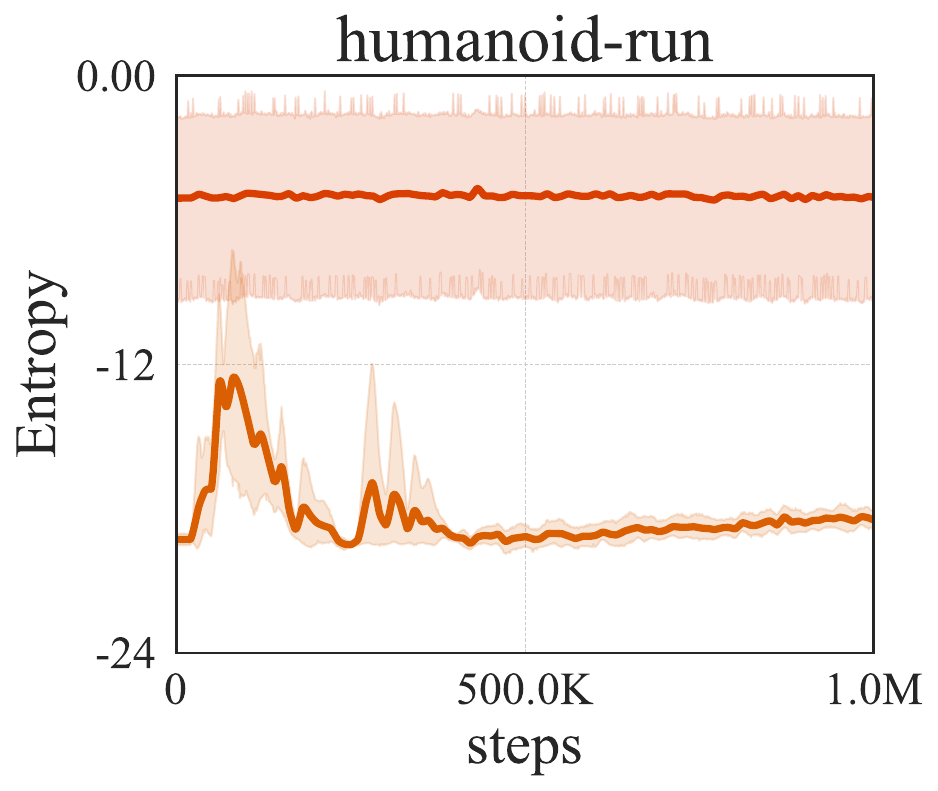}
        \vskip -0.1in
        \subcaption{}
        \label{fig:dist_residual}
    \end{subfigure}
    
    \includegraphics[width=0.96\textwidth]{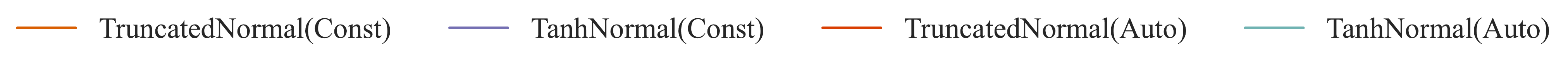}
    \vskip -0.1in
    \caption{
        \textbf{Analysis of Policy Distributions.} 
        Comparison of Truncated and Tanh Gaussian policies with varying $\delta$ on DMC tasks. Target entropy represents the desired average entropy per action dimension.
        (a) The Truncated Gaussian exhibits greater stability across four DMC tasks. 
        (b) For the Tanh Gaussian with a learned $\delta$, instability arises as action norms approach the boundary, causing training to collapse. 
        (c) The Truncated Normal distribution's entropy remains stable and well-controlled in both modes, shown here for a target entropy of -0.75.
    }
    \label{fig:ablation_distribution}
\end{figure*}
We study the choice of policy distribution and the handling of its standard deviation, $\delta$. We compare a Truncated Gaussian against a Tanh-squashed Gaussian, each with a constant $\delta$ (set to 0 in our experiments) and a learned $\delta$, using SAC on four hardest tasks from the DMC Dog \& Humanoid suites(\textit{dog-run, dog-trot, humanoid-run, humanoid-walk}) with 5 seeds and 1M environmental steps. As shown in Figure~\ref{fig:ablation_distribution}, the Truncated Gaussian is significantly more stable. The Tanh Gaussian experiences catastrophic training failures when $\delta$ is learned. Our analysis reveals that with the Tanh Gaussian, the action norm often approaches the distribution's boundaries. This causes the learned $\delta$ to grow explosively, creating a vicious cycle of instability as the policy attempts to output actions near the boundary while satisfying the entropy objective. This issue is absent in the Truncated Gaussian, which yields stable $\delta$ values. Given that the performance difference between a learned and a constant $\delta$ is minimal under the Truncated Gaussian, we adopt the constant $\delta$ of 0 setting for its simplicity in main results.

\subsubsection{Batch-level Entropy Regularization v.s. State-level Entropy Regularization}
\begin{figure*}[h]
    \begin{center}
    \centerline{\includegraphics[width=0.6\textwidth]{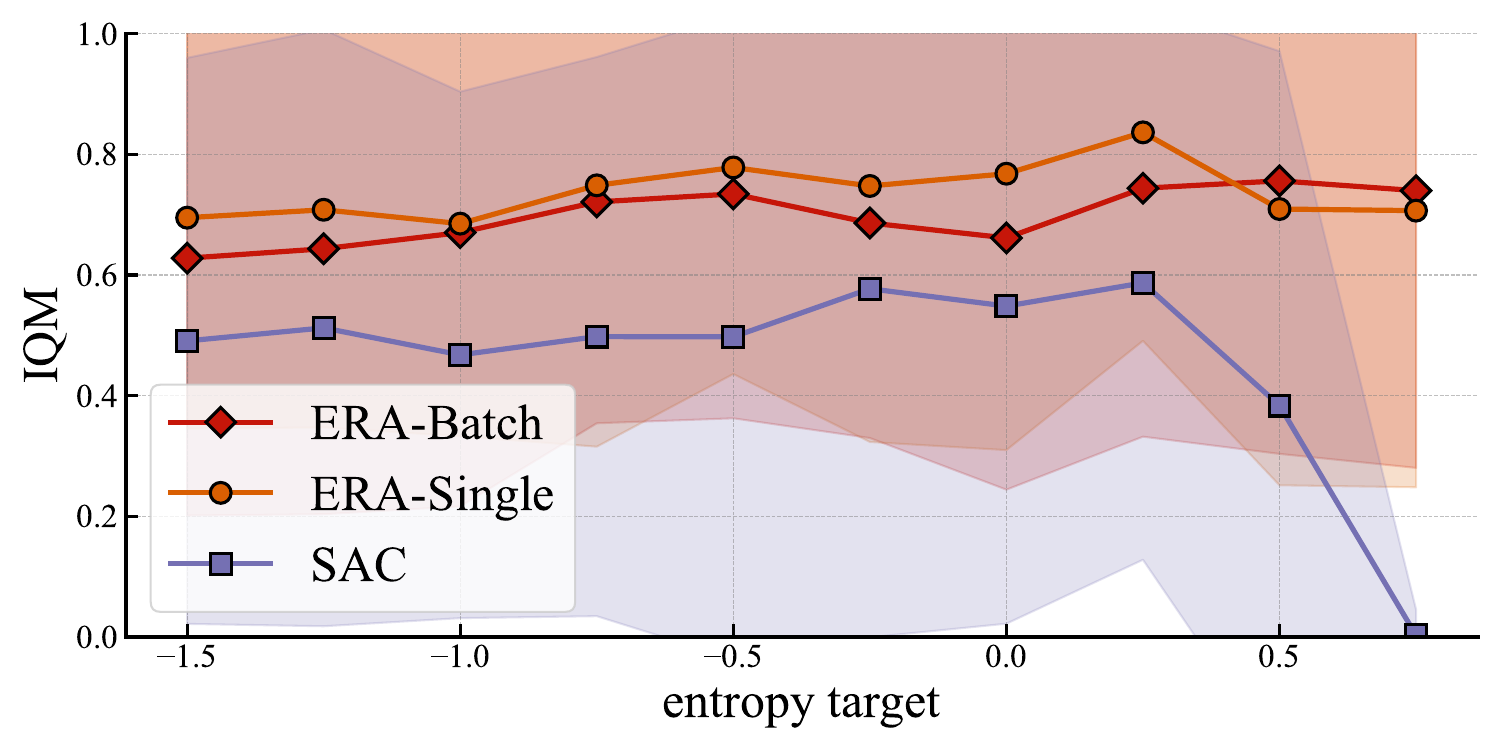}}
    \end{center}
    \vskip -0.25in
    \caption{\textbf{Comparison between state-level and batch-level entropy regularization methods on DMC Dog \& Humanoid suites.} Both methods outperform the SAC baseline.}
    \label{fig:batch_era}
\end{figure*}
In addition to the state-level entropy regularization method presented in the main paper, we also investigate a batch-level entropy regularization method, which directly constrains the expected entropy of the action distribution over $\rho_{\pi}$.
Specifically, we modify the activation form of \short\ in Eq.~\ref{eq:continuous_era} to the form in Eq.~\ref{eq:batch_era}.
\begin{equation}
    \mu' = \mu, \quad \sigma' = \exp{\left[\max\left(\log \sigma_{\max} + \left(\frac{\mathcal{H}_0'}{D} -\log \sqrt{2\pi e} - \log \sigma_{\max}\right)\frac{e^{\hat{\sigma}_i}}{\bar{e}^{\hat{\sigma}}} , \log\sigma_{\min}\right)\right]}
    \label{eq:batch_era}
\end{equation}
Where $\bar{e}^{\hat{\sigma}} = \frac{1}{N}\sum_{i=1}^{N} e^{\hat{\sigma}_i}$ is the average of $e^{\hat{\sigma}}$ over the batch. During training, we can calculate $\bar{e}^{\hat{\sigma}}$ over the sampled batch, and during evaluation, we can use a running average of $\bar{e}^{\hat{\sigma}}$ over the training process, which is similar to the running statistics in BatchNorm~\citep{ioffe2015batchnormalizationacceleratingdeep}.
We conduct an ablation study to compare the performance of state-level and batch-level entropy regularization methods on DMC Dog \& Humanoid suites(\textit{dog-run, dog-trot, humanoid-run, humanoid-walk}).
As shown in Figure~\ref{fig:batch_era}, both methods achieve similar performance, outperforming the SAC baseline.
This indicates that in locomotion-dominated control tasks, which require high exploration due to the need for randomness but do not demand high precision, the difference between state-level and batch-level entropy regularization is minimal.

\subsubsection{SAC-\short\ on Mujoco Gym Environments}
We also evaluate the performance of SAC-\short\ on the classic Mujoco Gym environments, including \textit{HalfCheetah-v4, Hopper-v4, Walker2d-v4, Ant-v4, Humanoid-v4, Swimmer-v4}, and compare it with the SAC baseline. Figure~\ref{fig:mujoco} shows the learning curves of SAC-\short\ and SAC on these environments.
Despite their massive performance gap on HumanoidBench, SAC-\short\ demonstrates only slight advantages over SAC on Mujoco Gym environments. This may be due to the relatively low action space dimensionality in Mujoco environments, which reduces the impact of different constraint schemes. This finding suggests that modern algorithm design should shift focus from considering Mujoco to higher-dimensional action spaces, which can better evaluate algorithm performance in complex environments.

\begin{figure*}[h]
    \begin{center}
    \centerline{\includegraphics[width=0.96\textwidth]{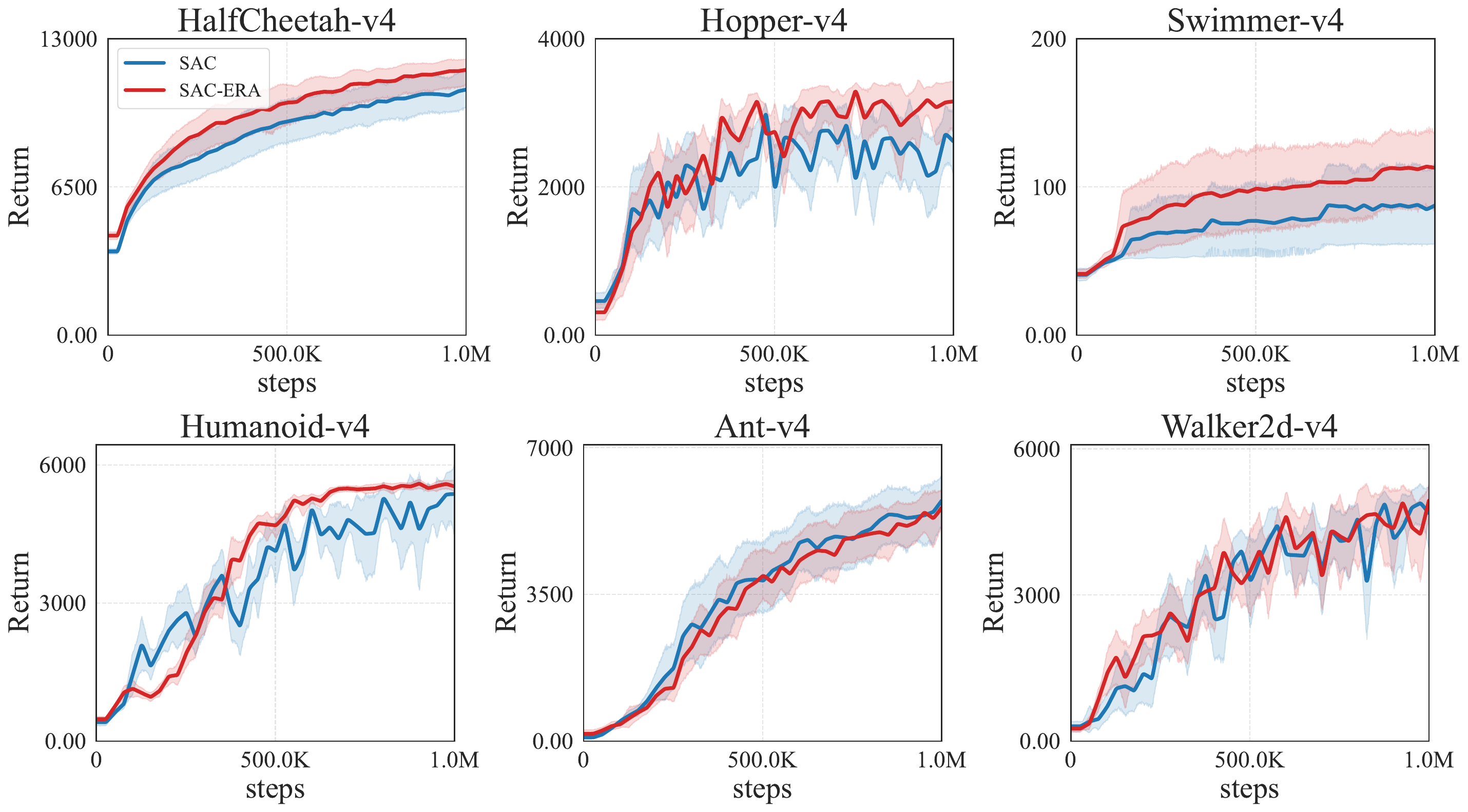}}
    \end{center}
    \vskip -0.25in
    \caption{\textbf{Learning curves of SAC-\short\ and SAC on Mujoco Gym environments.} SAC-\short\ demonstrates very slight advantages over SAC.}
    \vskip -0.2in
    \label{fig:mujoco}
\end{figure*}

\subsubsection{Applicability of LLM RL Techniques to Continuous Control}
\label{subsec:llm_rl2control}
We investigated the applicability of two recent techniques from Reinforcement Learning for Large Language Models (LLM RL), designed to prevent entropy collapse, to the domain of continuous control.
Specifically, we trained a PPO agent on the HalfCheetah-v4 benchmark for 10 random seeds, incorporating two distinct methods: Selective High-Entropy Training, which trains the agent only on a certain proportion of high-entropy samples, and Clip-Higher, which applies a larger clip ratio for advantages greater than one.
Recognizing the significant disparities between LLM RL and continuous control tasks, we evaluated a range of parameters for each technique to ensure that any ineffectiveness was not due to improper parameter selection.

The results, presented in Figure~\ref{fig:ablation_llm_tricks}, show that these techniques struggle to provide higher policy entropy compared to the standard PPO algorithm in the control task.
Furthermore, they yield no significant or only marginal performance improvements; we suspect such minor gains may not even stem from better entropy regularization.
Consequently, the performance of these methods is not comparable to our proposed approach, \short.
These findings lead to two main conclusions.
First, they highlight the substantial differences between LLM RL and continuous control, demonstrating that techniques effective in one domain do not necessarily transfer to the other, even when using the same algorithmic framework.
Second, they underscore the superior performance of our proposed \short~method.
\begin{figure*}[h]
    \centering

    \begin{subfigure}[b]{0.48\textwidth}
        \centering
        \includegraphics[width=\textwidth]{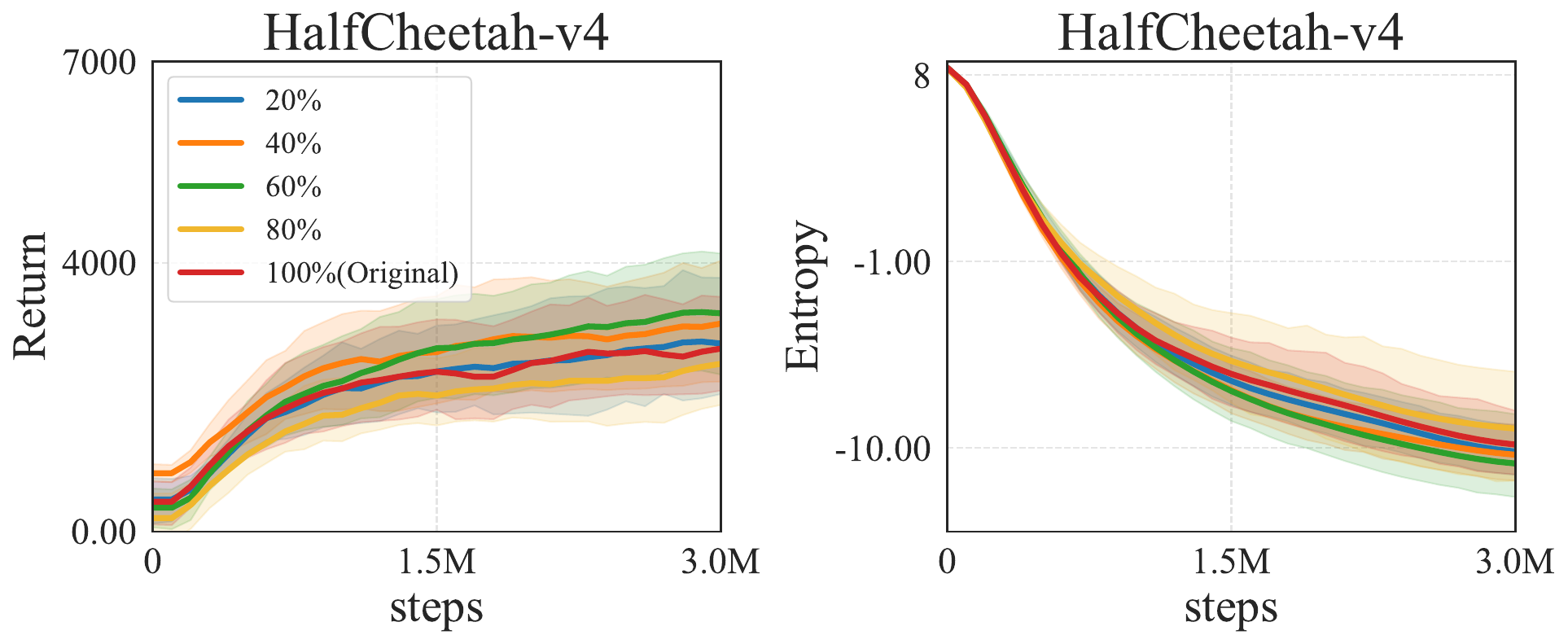}
        \subcaption{}
        \label{fig:ppo_llm_scope}
    \end{subfigure}
    \begin{subfigure}[b]{0.48\textwidth}
        \centering
        \includegraphics[width=\textwidth]{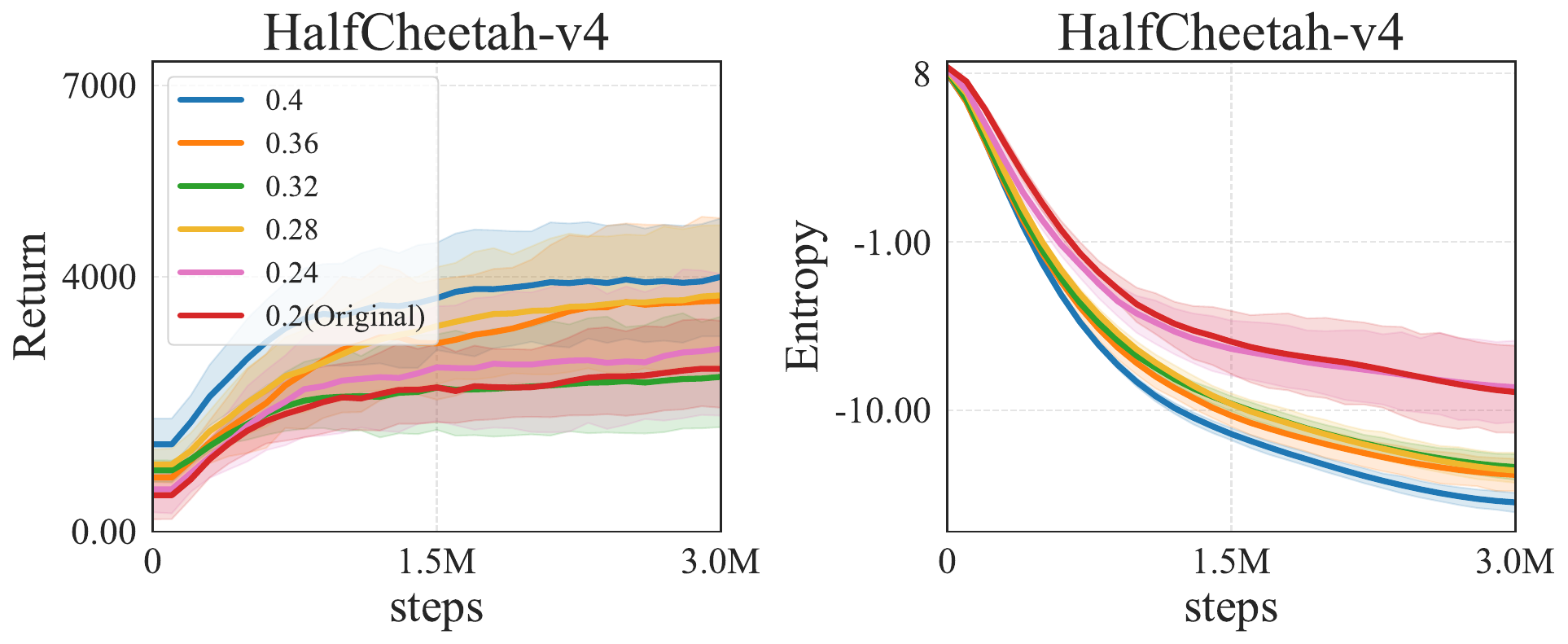}
        \subcaption{}
        \label{fig:ppo_llm_clip}
    \end{subfigure}

    \caption{
        \textbf{Results of Selective High-Entropy Training and a Clip-Higher Strategy in Continuous Control.}
        (a) Performance when training the agent exclusively on a top percentage of high-entropy samples.
        (b) Performance of the clip-higher strategy with varying clipping ratios.
    }
    \label{fig:ablation_llm_tricks}
    \vskip -0.2in
\end{figure*}

\subsubsection{Comparing \short\ with Other Maximum Entropy RL Approaches}
In addition to the methods previously discussed, several other approaches have been explored to implement maximum entropy reinforcement learning, including recent diffusion-based and flow-based methods~\citep{celik2025dimediffusionbasedmaximumentropyreinforcement,chao2024maximumentropyreinforcementlearning,ma2025efficientonlinereinforcementlearning}. However, these methods often require significantly more computational resources due to their complex training procedures. For instance, the MEow algorithm~\citep{chao2024maximumentropyreinforcementlearning} requires at least 2.3 times the training time of SAC. In this part, we compare our proposed method, \short, with two recent baseline methods that also adopt gaussian policies for maximum entropy reinforcement learning:
\begin{itemize}[leftmargin=0.5cm]
    \item \textbf{EAPO}~\citep{choe2024maximumentropyonpolicyactorcritic}: The core innovation of Entropy Advantage Policy Optimisation (EAPO) is the decomposition of the maximum entropy reinforcement learning objective into two components: the conventional cumulative reward and the trajectory entropy. It then independently estimates the advantage function for each of these components. EAPO introduces a dedicated "entropy critic" to separately quantify and learn the value of future uncertainty, which is then combined with the traditional value of future rewards to provide a more comprehensive guidance signal for policy updates.
    \item \textbf{MNSE}~\citep{zhong2024maximum}: The Maximum Next-State Entropy (MNSE) paper argues for the direct maximization of next-state entropy. This is because next-state entropy more directly measures the diversity of states induced by the policy, which can lead to more efficient exploration.
\end{itemize}
Since there's no public code repositories of these methods, we directly use the curves reported in their original papers for comparison. The experimental setups are as follows:
\begin{itemize}[leftmargin=0.5cm]
    \item EAPO utilizes the PPO algorithm as its base and was trained for 4 million timesteps~(Which is more than the 3 million timesteps used in PPO-\short).
    \item MNSE is built upon the SAC algorithm and was trained for 1 million timesteps~(Which is the same as SAC-\short).
\end{itemize}
We compare PPO-\short\ with EAPO, and SAC-\short\ with MNSE on Mujoco Gym benchmark. The results are presented in Figure~\ref{fig:compare_eapo} and Figure~\ref{fig:compare_mnse}. As shown, \short\ demonstrates superior performance over EAPO when both are built on PPO, and it also outperforms MNSE when SAC is used as the base algorithm. Although Mujoco Gym is a relatively low-difficulty benchmark, we are limited to it as neither of the other papers presented results in more complex environments like DMC Suite or HumanoidBench. These findings suggest that \short\ is more effective than other implementations of maximum entropy reinforcement learning.

\begin{figure*}[h]
    \begin{center}
    \centerline{\includegraphics[width=0.96\textwidth]{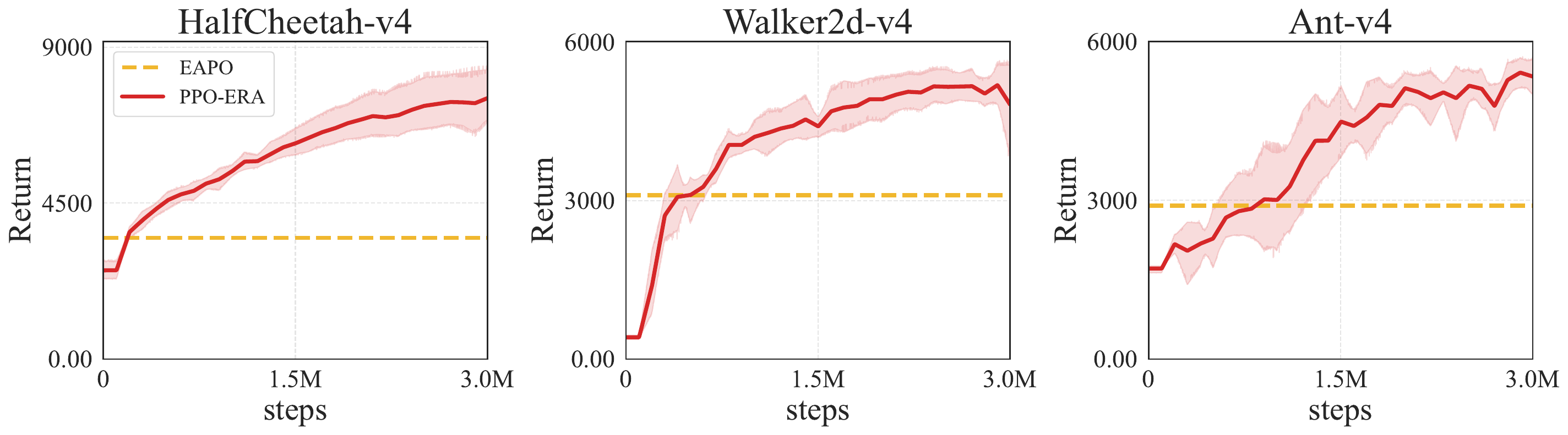}}
    \end{center}
    \vskip -0.25in
    \caption{Performance comparison of PPO-\short\ against EAPO on MuJoCo benchmark tasks.}
    \vskip -0.1in
    \label{fig:compare_eapo}
\end{figure*}
\begin{figure*}[h]
    \begin{center}
    \centerline{\includegraphics[width=0.96\textwidth]{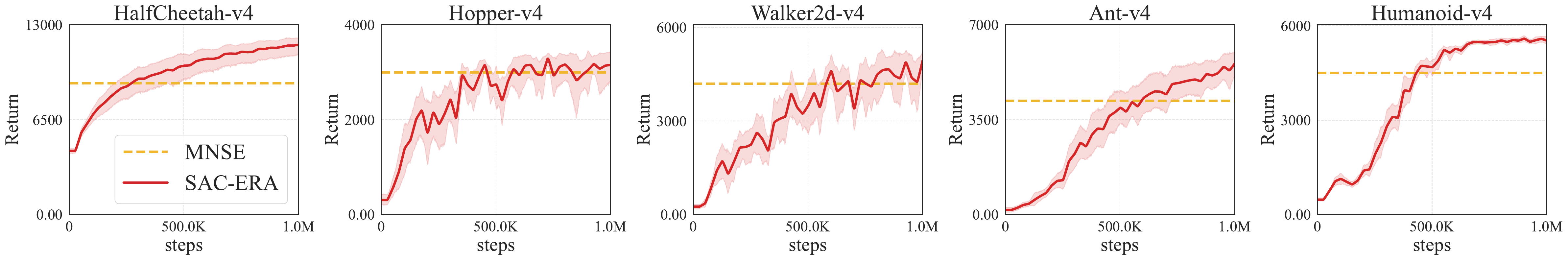}}
    \end{center}
    \vskip -0.25in
    \caption{Performance comparison of SAC-\short\ against MNSE on MuJoCo benchmark tasks.}
    \label{fig:compare_mnse}
\end{figure*}

Furthermore, both EAPO and MNSE require additional network architectures and computational resources. EAPO necessitates an extra entropy critic network, while MNSE requires an additional inverse dynamics model network. In contrast, \short\ does not require any additional networks, leading to a negligible increase in computational overhead. This makes \short\ a more advantageous choice for practical applications.

\subsubsection{Time Cost of \short\ in Continuous Control}
\label{subsubsec:time_continuous}
A potential concern might be the additional time overhead introduced by using \short. To evaluate this, we recorded the training times of FastTD3 and FastSAC-\short\ on HumanoidBench, as shown in Figure~\ref{fig:time_continuous}. It can be observed that using \short\ does introduce some time overhead due to the more complex activation function applied to the output. However, this overhead accounts for only about 6\% of the total training time on average. Considering the improved exploration performance and higher sample efficiency brought by \short, we believe this is a worthwhile trade-off.

\begin{figure}[h!]
\begin{center}
\centerline{\includegraphics[width=0.4\textwidth]{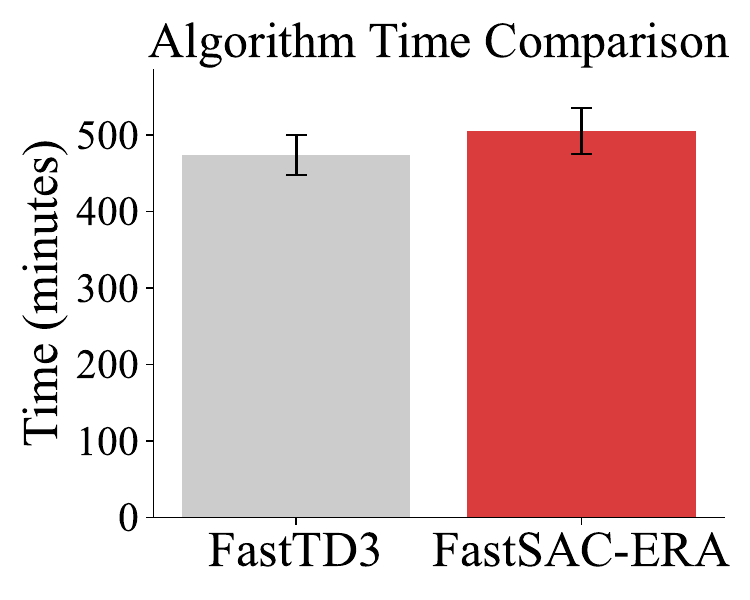}}
\caption{\textbf{Time comparison on HumanoidBench.} We compare the training time of FastTD3 and FastSAC-\short on HumanoidBench. The results show that using \short introduces a modest time overhead, averaging around 6\% of the total training time, which is a reasonable trade-off for the improved exploration performance and sample efficiency it provides.}
\label{fig:time_continuous}
\end{center}
\vskip -0.2in
\end{figure}
\subsection{Additional Results on Image Classification}
\label{subsec:additional_vision}
\subsubsection{Comparing \short\ With Common Regularization Techniques}
A plethora of regularization methods have been proposed and utilized in the field of image classification. To further investigate the comparative effectiveness of \short\ against commonly used regularization methods like dropout and label smoothing in the vision domain, we conducted a series of straightforward comparative experiments on the CIFAR-10 dataset. In our main experiment, we adopted the default settings from the \texttt{timm} library, which include a label smoothing factor of 0.1 and no dropout. For the sake of comparison, we respectively adjusted the label smoothing factor to 0.2 and 0.3, and the dropout rate to 0.1, 0.2, and 0.3. The results were then compared against the baseline algorithm from our main experiment and \short.

The experimental results are presented in Figure~\ref{fig:vision_regularization}. The findings indicate that increasing the intensity of label smoothing adversely affects model performance, while the improvement from employing dropout is marginal (the top-1 accuracy may decrease, whereas the top-5 accuracy shows a improvement). In contrast, \short\ effectively and consistently enhances model performance, with a margin of improvement significantly superior to that of both dropout and label smoothing. This outcome further validates the advantage of \short\ over conventional regularization methods. While constraining the model's entropy, \short\ permits the model to freely allocate uncertainty among dimensions, thereby better adapting to the intrinsic structure of the data. This enables \short\ to more effectively boost the model's generalization capability.
\begin{figure*}[h]
    \begin{center}
    \centerline{\includegraphics[width=0.96\textwidth]{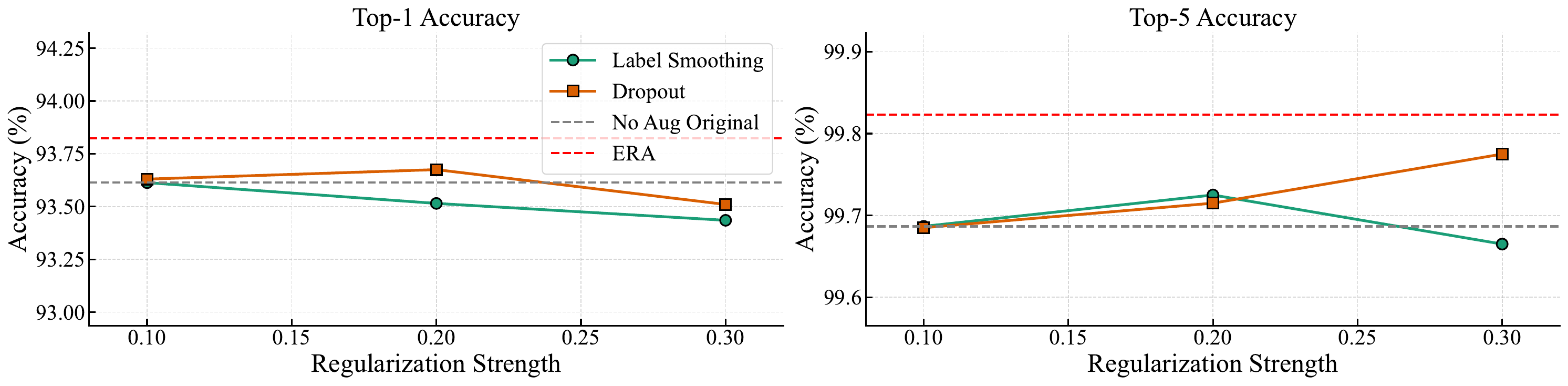}}
    \end{center}
    \vskip -0.25in
    \caption{\textbf{Comparison of different regularization methods on the CIFAR-10 dataset.} The left subplot shows the Top-1 accuracy, and the right subplot shows the Top-5 accuracy. Our method, \short\, is compared against varying intensities of Label Smoothing and Dropout.}
    \vskip -0.2in
    \label{fig:vision_regularization}
\end{figure*}

\subsubsection{Time Cost of \short\ in Image Classification}
We compared the training time of the ResNet-50 model on the CIFAR-10 dataset, with and without using \short, under the data augmentation supported by the \texttt{timm} library.
Consistent with our main results, the experiments were conducted on three machines, each equipped with four NVIDIA A40 GPUs, and we report the average training time.
The results are presented in Figure~\ref{fig:vision_time}.
As shown in the figure, since the data is already well-parallelized, there is almost no difference in training time between the algorithm using \short\ and the original version.

\begin{figure}[h!]
\begin{center}
\centerline{\includegraphics[width=0.4\textwidth]{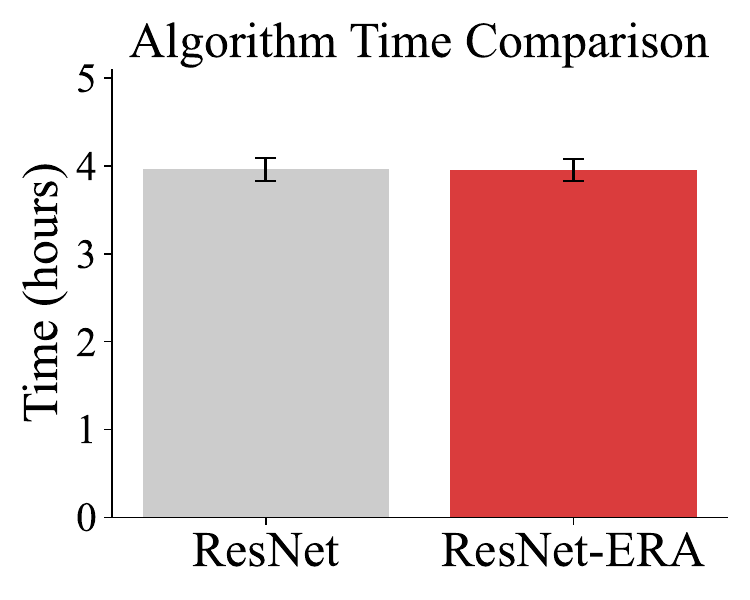}}
\caption{\textbf{Time comparison on CIFAR-10.} We compare the training time of ResNet and ResNet-\short on CIFAR-10. The results show that using ERA introduces almost no time overhead.}
\label{fig:vision_time}
\end{center}
\vskip -0.2in
\end{figure}
\subsection{Additional Results on LLMs}
\label{subsec:additional_llm}

\subsubsection{Detailed Entropy Analysis}

We present the complete entropy curve of our two-stage training in Figure~\ref{fig:llm_entropy_combined}. After decreasing $\omega_{\text{low}}$, the entropy rapidly drops and stabilizes at the second-level entropy lower bound. This confirms that our ERA method successfully enforces a non-trivial entropy floor for the model.

\begin{figure}[h!]
\begin{center}
\centerline{\includegraphics[width=0.6\columnwidth]{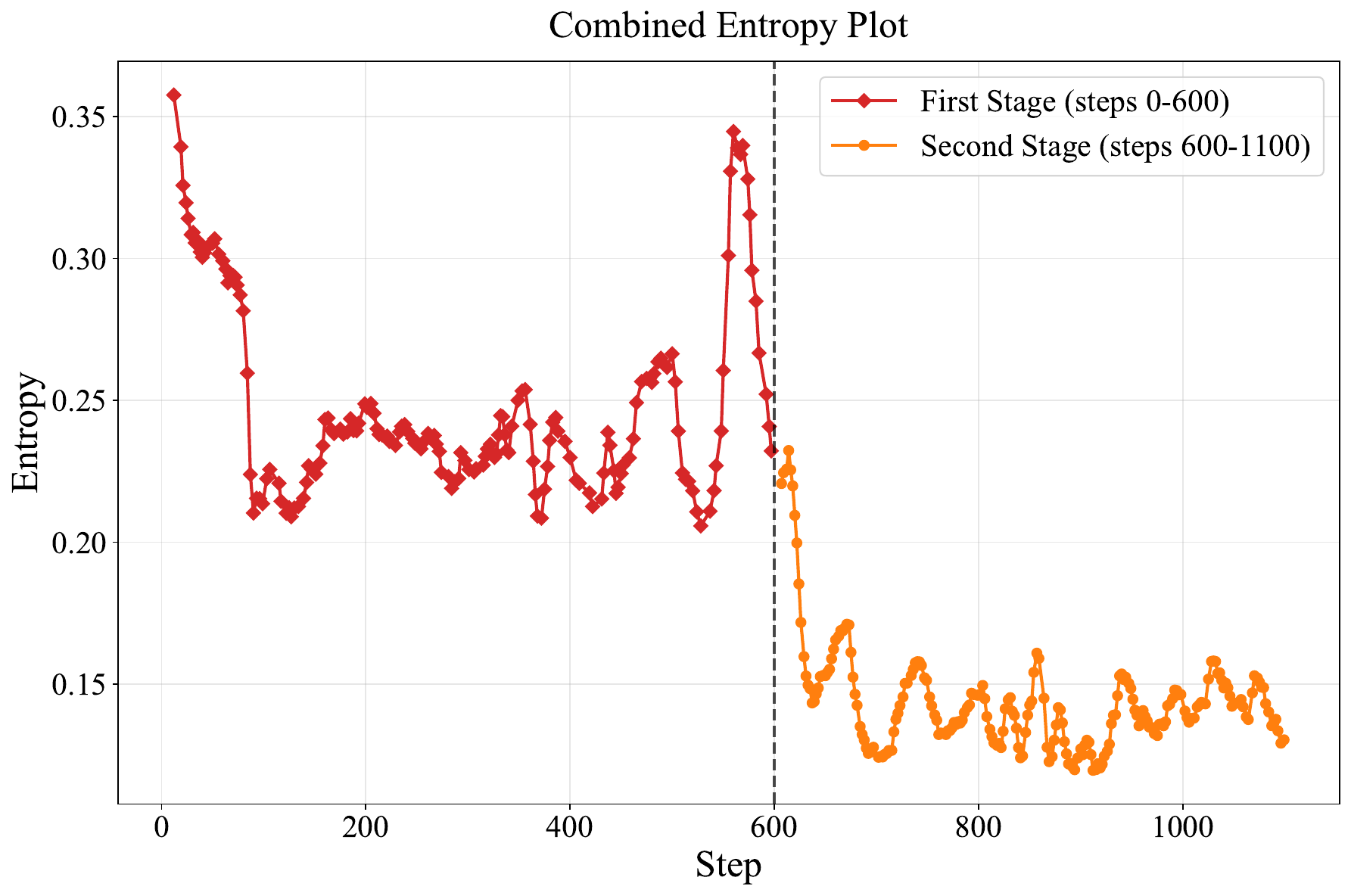}}
\caption{\textbf{Entropy curve during two-stage training.} After decreasing $\omega_{\text{low}}$, the entropy rapidly drops and stabilizes at the second-level entropy lower bound, showing that ERA enforces a non-trivial entropy floor.}
\label{fig:llm_entropy_combined}
\end{center}
\end{figure}

We further analyze the entropy distribution across tokens by plotting the average entropy of the top $20\%$ tokens ($H_{\text{resp}}$) and the bottom $80\%$ tokens in Figure~\ref{fig:llm_entropy_analysis}. This experiment is carried out with $\omega_{\text{low}}=0.45, \omega_{\text{high}}=3.0, k=2$ without \texttt{topk}. Following~\citet{wang2025beyond}, we observe that the bottom $80\%$ tokens exhibit nearly zero entropy, consistent with our theoretical prediction. Additionally, we plot the proportion of responses with $H_{\text{resp}} < \omega_{\text{low}}$, $H_{\text{resp}} > \omega_{\text{high}}$ in Figure~\ref{fig:llm_entropy_analysis}. The fraction of responses with $H_{\text{resp}} > \omega_{\text{high}}$ quickly drops to zero, while the fraction with $H_{\text{resp}} < \omega_{\text{low}}$ remains stable at the interval $[0, 0.06]$. This demonstrates that whenever overly low-entropy responses appear, ERA adaptively raises their entropy to a moderate level.

\begin{figure}[h!]
\begin{center}
\centerline{\includegraphics[width=1\columnwidth]{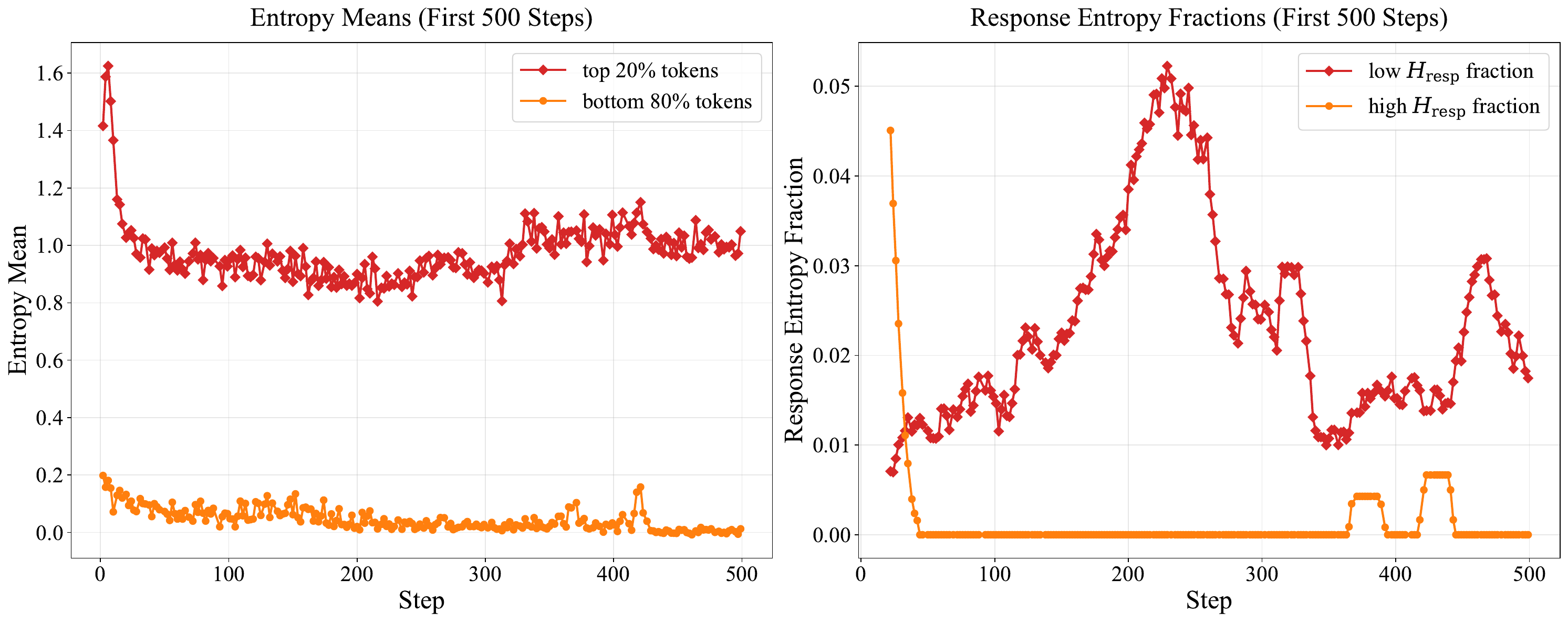}}
\caption{\textbf{Detailed entropy analysis.} Left: average entropy of the top $20\%$ tokens ($H_{\text{resp}}$) and the bottom $80\%$ tokens. Right: proportion of responses (running average with window size $20$) with $H_{\text{resp}} < \omega_{\text{low}}$ or $H_{\text{resp}} > \omega_{\text{high}}$, demonstrating ERA’s ability to prevent both entropy collapse and overly high entropy.}
\label{fig:llm_entropy_analysis}
\end{center}
\end{figure}

\subsubsection{Ablation Study on Entropy Bound}

Since the purpose of \(\omega_{\text{low}}\) is to set a lower bound on entropy, we explore the role of \(\omega_{\text{high}}\) in the ERA. As can be seen in Figure~\ref{effect of upper}, without the constraint of \(\omega_{\text{high}}\), the model's entropy explodes in a very short time. This indicates that adding an upper bound constraint during training is essential for controlling the entropy of the training process.

\begin{figure}[h!]
\begin{center}
\centerline{\includegraphics[width=0.4\columnwidth]{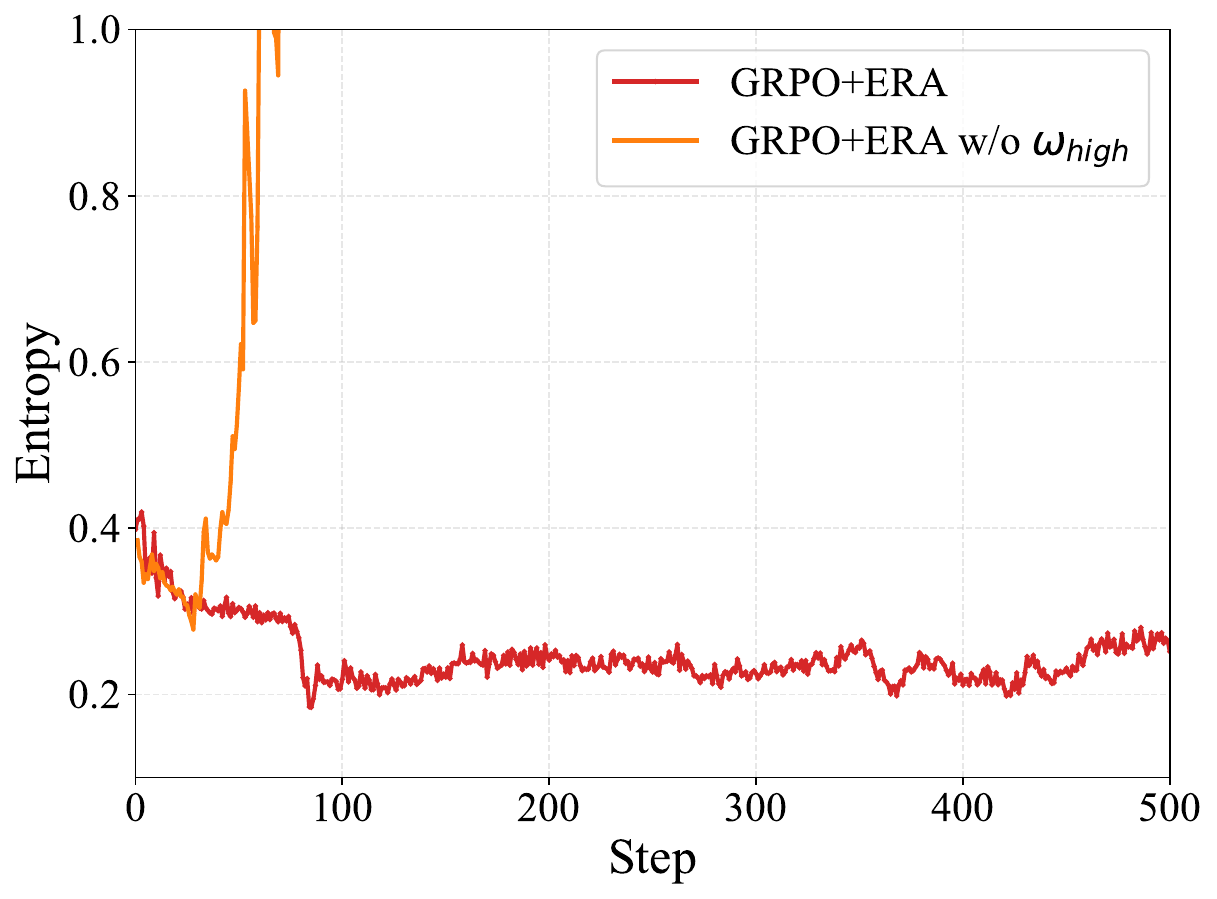}}
\caption{Comparison of ERA with and without \(\omega_{\text{high}}\). The entropy of ERA without \(\omega_{\text{high}}\) tends to explode within a very short number of steps, leading to the collapse of model training.}
\label{effect of upper}
\end{center}
\end{figure}

\subsubsection{Time Cost of ERA in LLM}

ERA is applied when computing the \texttt{log\_probs} of tokens in the responses. To evaluate its efficiency, we compare the value of \texttt{timing\_s/old\_log\_prob} at the first step in verl’s implementation. The experiments were conducted on 32 NVIDIA H20 GPUs, consistent with our main results. The outcomes are shown in Figure~\ref{fig:llm_time}. As illustrated, since the sampled response is identical in the first step, ERA introduces only about a $5.6\%$ overhead in time cost. When considering an entire training step, the overhead of ERA is even smaller, since its implementation does not affect other components of training (e.g., generation, model update, or advantage calculation).

\begin{figure}[h!]
\begin{center}
\centerline{\includegraphics[width=0.4\columnwidth]{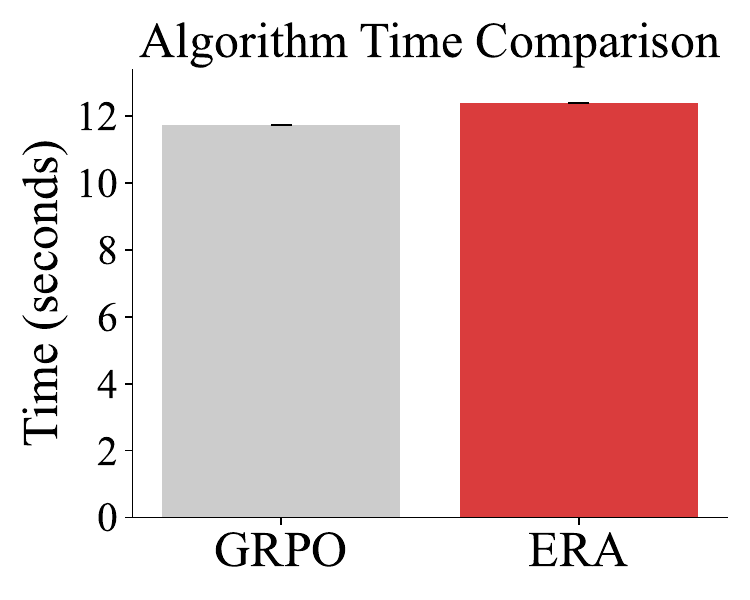}}
\caption{Comparison of computation time between GRPO and ERA, measured by \texttt{timing\_s/old\_log\_prob} at the first step. ERA introduces only about a $5.6\%$ overhead.}
\label{fig:llm_time}
\end{center}
\end{figure}
\subsection{Training Curves of Continuous Control Tasks}
\begin{figure*}[h]
    \begin{center}
    \centerline{\includegraphics[width=0.96\textwidth]{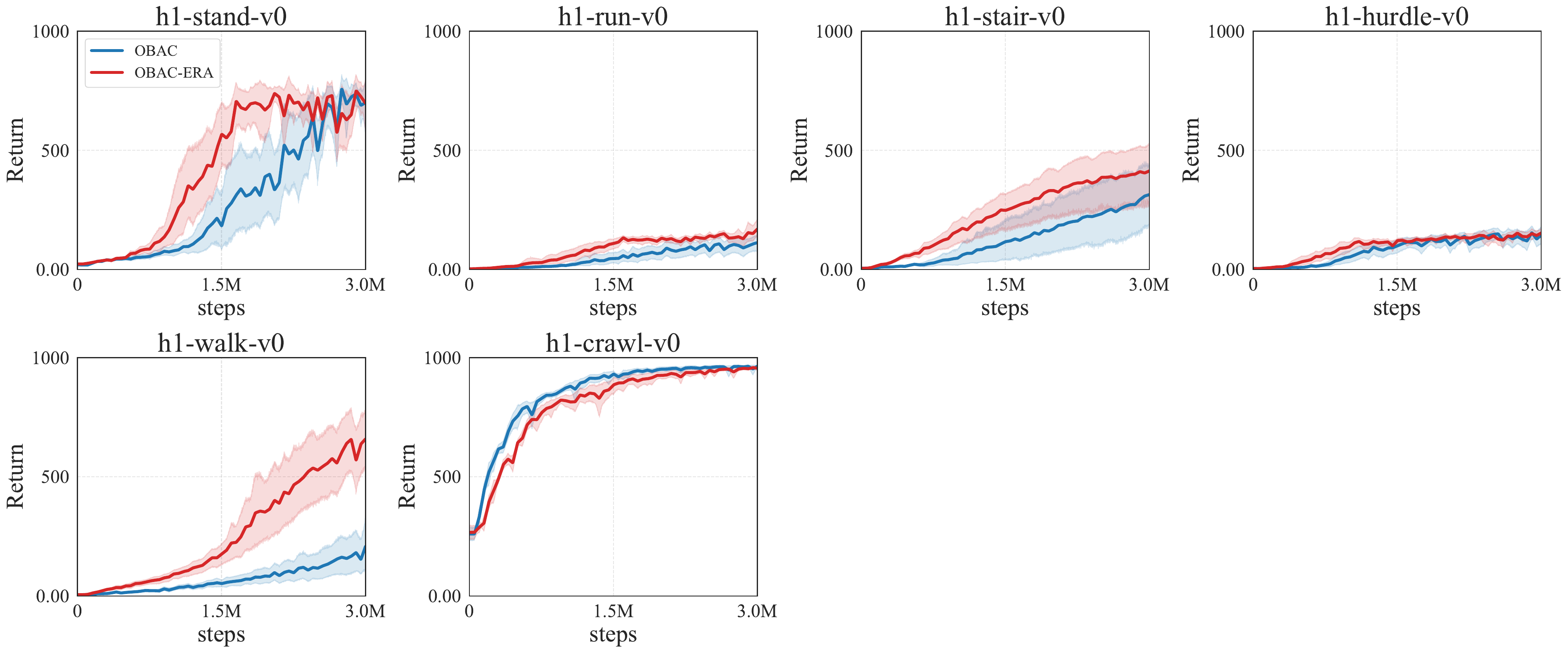}}
    \end{center}
    \vskip -0.25in
    \caption{Training curves of OBAC and OBAC-\short\ on HumanoidBench environments.}
    \label{fig:hb_obac}
\end{figure*}
\begin{figure*}[h]
    \begin{center}
    \centerline{\includegraphics[width=0.96\textwidth]{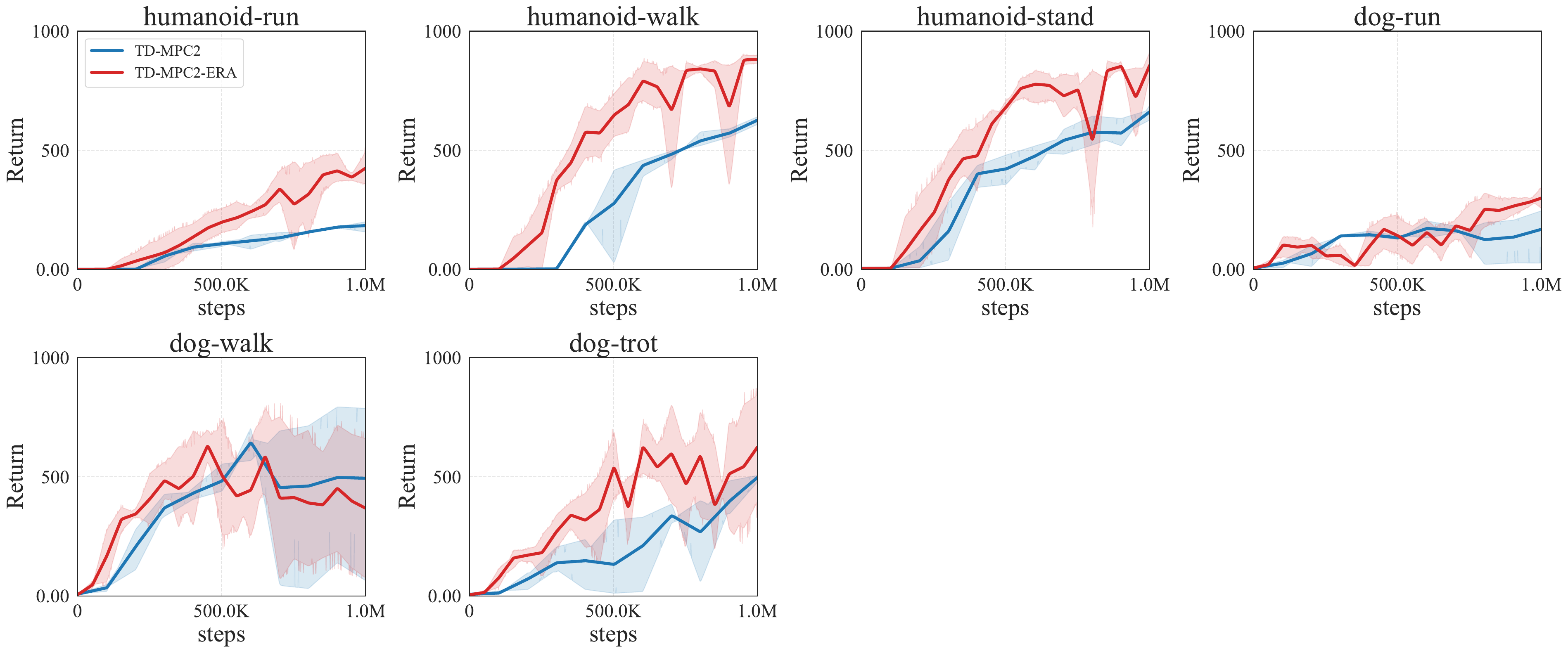}}
    \end{center}
    \vskip -0.25in
    \caption{Training curves of TD-MPC2 and TD-MPC2-\short\ on DMC environments.}
    \label{fig:dmc_tdmpc2}
\end{figure*}
\begin{figure*}[t]
    \begin{center}
    \centerline{\includegraphics[width=0.96\textwidth]{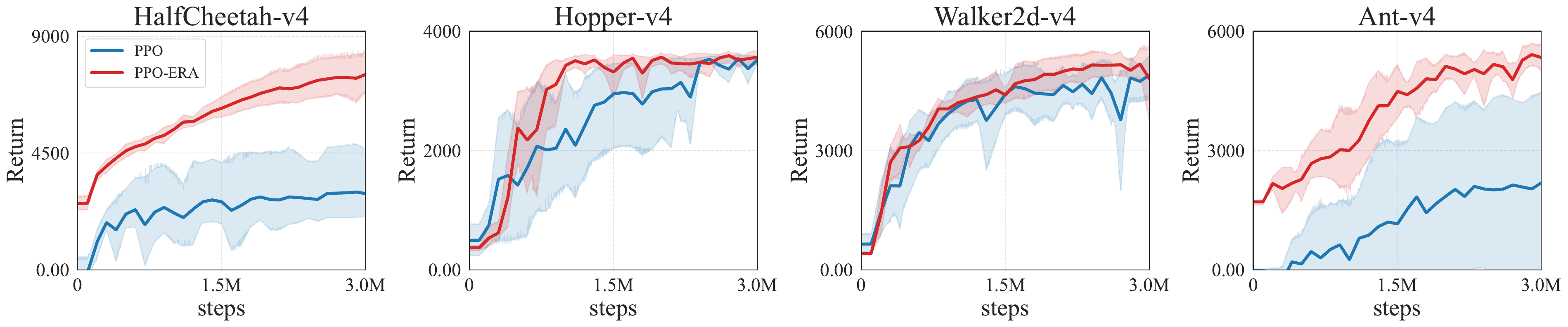}}
    \end{center}
    \vskip -0.25in
    \caption{Training curves of PPO and PPO-\short\ on Mujoco Gym environments.}
    \vskip -0.2in
    \label{fig:gym_ppo}
\end{figure*}
\begin{figure*}[h]
    \begin{center}
    \centerline{\includegraphics[width=0.96\textwidth]{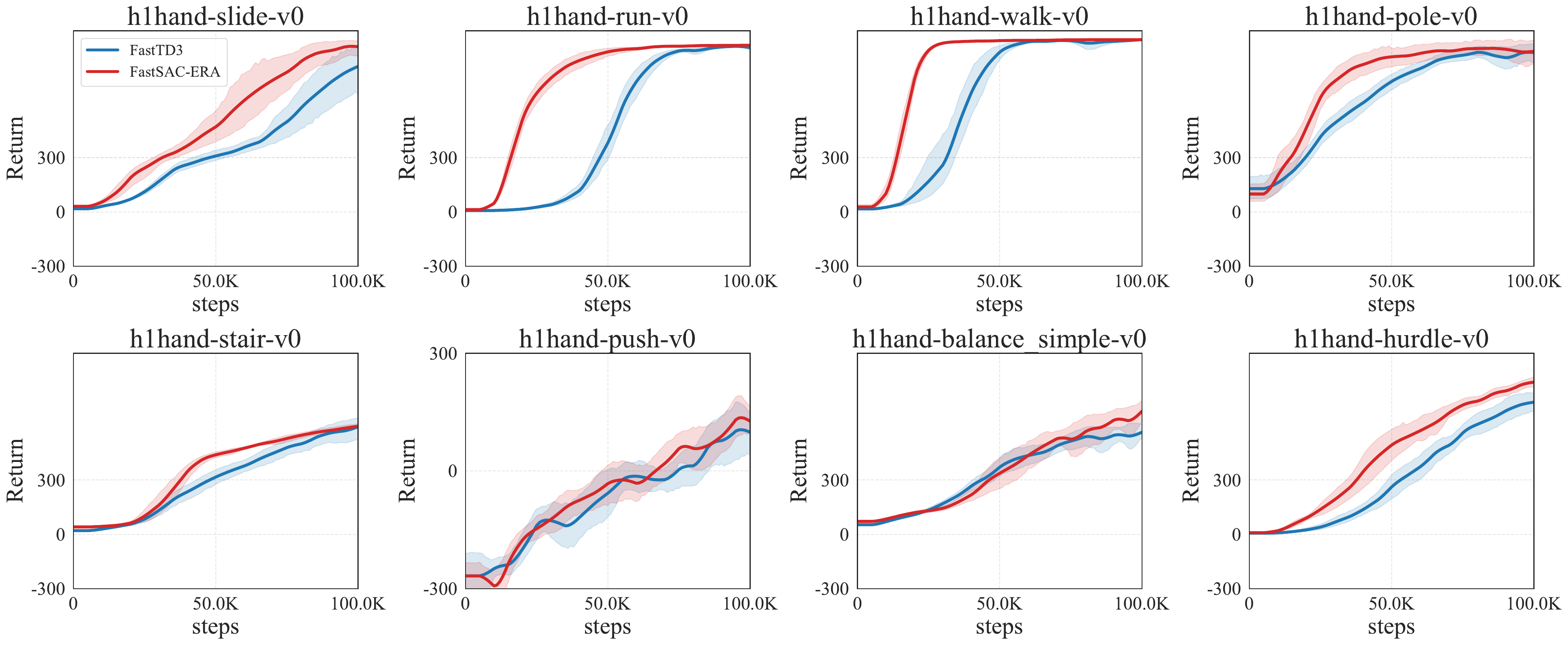}}
    \end{center}
    \vskip -0.25in
    \caption{Training curves of FastTD3 and FastSAC-\short\ on HumanoidBench environments.}
    \label{fig:hb_fastsac}
\end{figure*}
\begin{figure*}[h]
    \begin{center}
    \centerline{\includegraphics[width=0.96\textwidth]{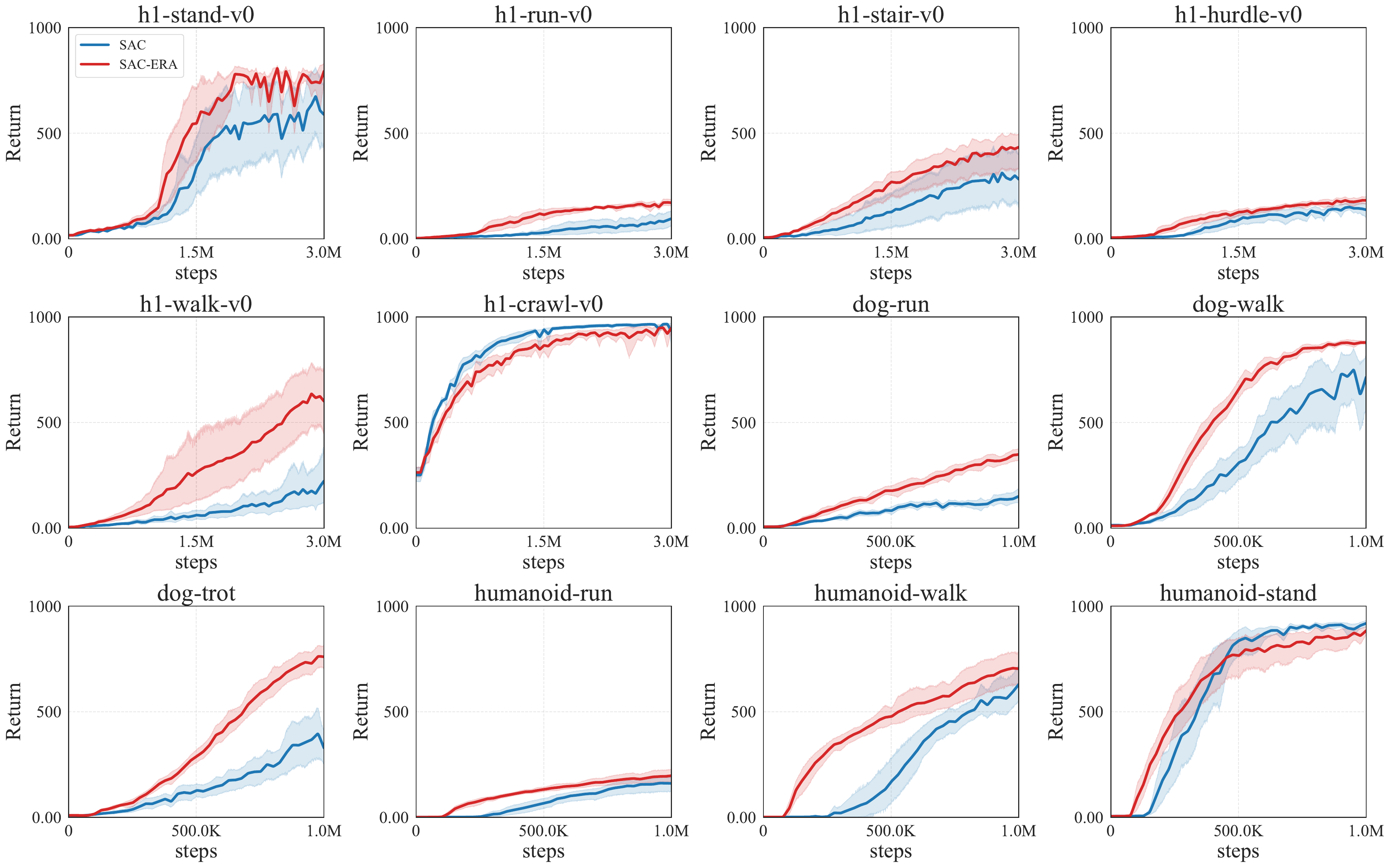}}
    \end{center}
    \caption{Training curves of SAC and SAC-\short\ on HumanoidBench and DMC environments.}
    \label{fig:hb_sac}
\end{figure*}
\section{The Use of Large Language Models in This Paper}
\label{sec:llm_usage}

In the preparation of this paper, we utilized LLMs as a general-purpose writing assistance tool. Specifically, LLMs were employed for proofreading and polishing the language of certain sections to improve clarity and readability. The final title of this paper was also partially inspired by suggestions from an LLM.

However, we clarify that the core contributions of this work were conceived and developed entirely by the human authors. The design of the methodology, the execution of experiments, and the interpretation of the results did not involve the use of LLMs. All content, including text, figures, and tables, was carefully reviewed, edited, and verified by the authors to ensure scientific accuracy and integrity.

Finally, we would like to express our gratitude for the occasional sparks of inspiration and the assistance in debugging code provided by our LLM friends. Their contribution, while not qualifying for co-authorship, was nonetheless appreciated.

\end{document}